\documentclass[a4paper,11pt]{article}

\usepackage{amsfonts}
\usepackage{graphicx}
\usepackage{epstopdf}
\usepackage{algorithmic}
\usepackage{relsize}
\usepackage{tikz}
\usepackage{mathtools}
\usepackage{setspace}
\usepackage{tabularx}
\usepackage{booktabs}
\usepackage{arydshln}

\usetikzlibrary{positioning}
\ifpdf
  \DeclareGraphicsExtensions{.eps,.pdf,.png,.jpg}
\else
  \DeclareGraphicsExtensions{.eps}
\fi

\usepackage{enumitem}
\usepackage{tkz-euclide}
\setlist[itemize]{leftmargin=.5in}

\usepackage{amsmath,amsthm}
\usepackage{amssymb}
\usepackage{amsopn}
\usepackage[hidelinks]{hyperref}
\usepackage{algorithm}
\usepackage[margin=2.5cm]{geometry}
\usepackage{esint}
\usepackage{array}
\usepackage[linewidth=2pt]{mdframed}

\providecommand{\keywords}[1]{\textbf{\textit{Keywords: }} #1}
\providecommand{\MSCcodes}[1]{\textbf{\textit{MSC Codes: }} #1}

\usepackage{xcolor}
\usepackage[most]{tcolorbox}
\usepackage[nottoc]{tocbibind}
\newcommand{\hookdoubleheadrightarrow}{%
	\hookrightarrow\mathrel{\mspace{-15mu}}\rightarrow
}

\newcommand{\supp}{\text{supp}}

\newcommand*\dx{\mathop{}\!\mathrm{d}}

\DeclareMathOperator*{\argmin}{arg\,min}

\newtheorem{theorem}{Theorem}
\newtheorem{proposition}[theorem]{Proposition}
\newtheorem{remark}[theorem]{Remark}
\newtheorem{definition}[theorem]{Definition}
\newtheorem{lemma}[theorem]{Lemma}
\newtheorem{corollary}[theorem]{Corollary}
\newtheorem{example}[theorem]{Example}
\newtheorem{assumption}[theorem]{Assumption}

\title{Physically Consistent Model Learning\\ for Reaction-Diffusion Systems}
\author{Martin Holler \thanks{IDea\_Lab - The Interdisciplinary Digital Lab at the University of Graz, University of Graz, Austria.
		\{\href{mailto:martin.holler@uni-graz.at}{martin.holler@uni-graz.at}, \href{mailto:erion.morina@uni-graz.at}{erion.morina@uni-graz.at}\}
	} \and Erion Morina \footnotemark[1]}

\ifpdf
\hypersetup{
	pdftitle={Physically Consistent Model Learning for Reaction-Diffusion Systems}, pdfauthor={M. Holler and E. Morina} }
\fi

\begin{document}
	\maketitle
	\begin{abstract}
	Learning physical models, or components thereof, from data is an active field of research. When the reaction of a reaction-diffusion (RD) system is unknown and is replaced by a parameterized function, learned from data, it is not guaranteed that the resulting model is a well-posed system, nor that it is physically meaningful, e.g., that non-negativity of the state is preserved. Penalizing violations of the corresponding conditions during the learning process enforces them only approximately, whereas the properties they encode are not approximate ones. In this paper, we introduce a framework for adapting a given class of parameterized reaction terms, such as neural networks, in a way that each of its elements satisfies quasipositivity and mass control, so that the learned RD systems preserve non-negativity and are well-posed under additional regularity and growth conditions. In particular, we show that the adapted class preserves the approximation capacity of the original one, and that within a regularization-based learning framework, solutions of the learning problem converge to a unique, regularization-minimizing solution of a limit system even when quasipositivity and mass control are enforced. We test the approach numerically on five RD systems and compare it with an unconstrained and a penalty-based approach, which, in contrast to the proposed approach, do not ensure quasipositivity of the learned reaction terms.
	\end{abstract}
	\\
	\MSCcodes{35R30, 93B30, 65M32, 35K57, 41A99}

	\section{Introduction}
	\label{sec:intro}	
	Scientific machine learning shifts from manually constructed models based on partial differential equations (PDEs) towards learning models directly from data. A variety of techniques for discovering physical laws have been proposed to infer the dynamics of the underlying PDEs. We refer to \cite{blechschmidt21,Boulle2024,kutz23, deryck24,Kovachki2024, tanyu23} for a comprehensive overview. A key feature of these methods is that, despite their data-driven flexibility, they use domain knowledge by incorporating physics-based PDEs into the proposed models. With this, interpretable results are provided and reliance on large datasets is reduced. %

	One important class of PDEs is that of reaction-diffusion (RD) systems
	\begin{align}
		\label{pdemodel}
		\partial_t u_n(t,x) = d_n \Delta u_n(t,x)+f_n(u(t,x)), \quad (t,x)\in \left]0,T\right[\times\Omega\tag{\text{RD}}
	\end{align}
	for $n=1,\dots,N$, on a domain $\Omega\subset\mathbb{R}^d$ and $T>0$, which is the focus in this work. Here $u=(u_n)_{n=1}^N$, the state variables, describe \textit{species} which change over time due to the following two central processes dominating in \eqref{pdemodel}: i) the reaction model $f=(f_n)_{n=1}^N$ describing interactions between the species such as production, consumption and transformation, and ii) the diffusion coefficients $(d_n)_{n=1}^N$ causing spatial spreading of the species due to \textit{Fick's law}. Such systems play an important role in the natural sciences \cite{Britton1986, Murray2002, Murray2003, Perthame2015}, e.g., in modeling chemical reactions with present diffusion, the spread of infectious diseases, pattern formation in animal fur, tumor growth, and population dynamics. Examples of RD models in this context include the Fisher-KPP equation \cite{Hamel2014}, the Gray-Scott model \cite{Pearson1993}, the Turing model \cite{Turing1990}, the SIR model \cite{Brauer2008}, the Lotka-Volterra equations \cite{Cantrell2004}, the Allen-Cahn equation \cite{Allen1979}, and many more.

	These models, each corresponding to a fixed reaction term $f$, contain parameters that are generally unknown, and modeling with them often faces the difficulty that even the form of the reaction term $f$ itself is uncertain, particularly when the underlying processes are complex or not well understood. This issue is addressed by works on learning RD systems from data, where the unknown reaction term is replaced by a user-defined parameterized class $\left\{f_\theta~|~\theta\in\Theta\right\}$, such as a class of neural networks, and $\theta$ is determined from measurements of the state. A crucial property of any such approach is the extent to which the learned reaction term is uniquely determined, since identifiability is rarely guaranteed if the class of candidate reaction models is too broad. The starting point of this work is \cite{morina_holler/online}, which, building on \cite{AHN23}, considers solutions to \eqref{pdemodel} minimizing a regularization functional, which can be interpreted as an instance to incorporate prior information on possible reaction terms to resolve unique identifiability. In \cite{morina_holler/online} it is shown that these solutions can be recovered by practically implementable learning frameworks in the limit of full, noiseless measurements, under: i) suitable regularization functionals together with the right choice of the regularization parameters depending on the noise level of the measured data, and ii) an \textit{approximation capacity condition} on the parameterized class, which is a universal approximation type property.

	Such results, however, are statements about $f_\theta$ as an element of a function space, and not about the system that $f_\theta$ defines. Even in the presence of a convergence result of this kind, it is not guaranteed that the learned model constitutes a mathematically well-posed system, in the sense that a global solution exists, nor that it is physically meaningful in the same way as the limit problem. In the case of RD systems, an important property is the preservation of non-negativity of the state (e.g., concentrations, densities or population sizes) and lose their meaning as physical quantities once they become negative. Closely related is the absence of finite-time blow-up of the total mass of the reacting species. While the RD system underlying the data is well-posed and preserves non-negativity by assumption, a generic parameterized class does not share these properties, and nothing in the learning problem enforces them automatically. This is relevant in particular when the learned model is used for simulation beyond the data from which it was inferred.

	In order to resolve this issue, one could think of penalizing violations of the corresponding conditions during the learning process, as is done for conservation laws in \cite{jagtap20, Li2024, Wu2022} (see Section~\ref{sec:related_work}). Here, two difficulties arise. First, it is not clear how such a penalization should be realized, since the conditions in question are pointwise conditions on $f$ on an unbounded region in the case of non-negativity preservation rather than conditions on the data residual, so that the penalty has to be assembled on a set for which no measurements are available. Second, even if the penalization is successful, the conditions are satisfied only approximately, whereas the properties they encode are not approximate ones. An arbitrarily small violation of quasipositivity is still a violation, and suffices for the solution to become negative.

	For this reason in this work we develop a framework to include the relevant conditions directly into the class of parameterized functions used to model the unknown reaction term, and to achieve \textit{physical consistency} at the level of the model itself. Every parameterized candidate reaction term during learning, and not only the limit, then defines a well-posed RD system whose solutions preserve non-negativity, by construction and independently of the data, the optimizer and the outcome of the training. The focus of this work is, on the one hand, the construction of such classes, in the sense of a systematic modification of a given, user-defined class into one whose elements satisfy the conditions in question while preserving the approximation capacity of the original class, and, on the other hand, the identifiability of the reaction terms learned within the modified classes, in the form of a convergence result of the type in \cite{morina_holler/online}.\\

	\noindent\textbf{Contributions.}
	\label{subsec:contributions}
	Following the two directions identified above, the construction of physically consistent classes and the identifiability of the reaction terms learned within them, our main contributions can be summarized as follows.

	\begin{enumerate}[label=\arabic*),leftmargin=0.5cm]
		\item \textit{Modification technique with structural guarantees} (Lemma~\ref{lem:auxiliary_result}, Corollaries~\ref{cor:wellposedness} and \ref{cor:wellposedness2}). Given any user-defined class $\mathcal{F}$ of Lipschitz continuous parameterized reaction terms, we construct a modified class $\overline{\mathcal{F}}$ each of whose elements satisfies local Lipschitz continuity \eqref{cond:lipschitz}, quasipositivity \eqref{cond:quasipos}, mass control \eqref{cond:mass} and the growth condition \eqref{cond:growth}, which together are sufficient for global well-posedness of \eqref{pdemodel} and for preservation of non-negativity \cite{Fellner2020, Pierre2010, Suzuki2017}. The construction is explicit and leaves the parameterization otherwise untouched. It is directly applicable to standard approximation classes such as neural networks.
		\item \textit{Preservation of approximation capacity} (Lemma~\ref{lem:abstract_approximation}, Proposition~\ref{cor:approximation_inherit}). Modifying a class can destroy its ability to approximate. We show that this is not the case here and that the modified class inherits the approximation capacity of the original one, in the sense that the uniform approximation rate of the reaction term is preserved and the Lipschitz constants of the approximations remain asymptotically bounded by the one of the reaction term.
		
		\item \textit{Well-posedness of the learning problem for the modified class} (Propositions~\ref{prop:f_mod_nemyt}--\ref{prop:reg_mod_class}, Theorem~\ref{thm:wellposedness} and Appendix~\ref{app:assum_reactions}). We verify that the modified classes satisfy the assumptions of the learning framework introduced in our previous work \cite{morina_holler/online}. This is not a formality, since the modification raises the regularity requirements and requires a corresponding analysis of the induced Nemytskii operators in function space.%
		\item \textit{Convergence of the learning problem under physical constraints} (Theorem~\ref{thm:main_result}). We show that the convergence result in \cite{morina_holler/online} remains valid for RD systems with physically consistent reaction terms, so that identifiability is not lost through the restriction of the class. The convergence statement itself is an application of \cite[Theorem 31]{morina_holler/online}, which requires that the modified classes satisfy its hypotheses. This is not trivial, since the properties of the original class do not carry over to the modified class directly. Besides the well-posedness in 3), this requires that the modified classes satisfy the approximation capacity condition of the framework, which we obtain from 2) under a strengthened quasipositivity of the reaction term and a choice of the transition functions depending on the approximation rate.
		\item \textit{Approximation of quasipositive functions} (Appendix~\ref{sec:approx_quaspos}, in particular Theorems~\ref{mult_inf_est} and \ref{mult_p_est}). We establish approximation results for quasipositive functions. These results are used for the modification above, but are of independent interest beyond RD systems.
		\item \textit{Numerical validation} (Section~\ref{sec:numerics} and Appendix~\ref{app:numerics}). We apply the framework to five RD systems with one to three species. The modified parameterization recovers both the state and the reaction term, while its quasipositivity violation vanishes exactly rather than approximately, by construction. An identically trained unconstrained baseline attains comparable accuracy but violates the constraint on every benchmark, and solutions of the learned system become negative when it is simulated from initial conditions outside the training set. A penalty-method baseline, which mitigates the violation instead of excluding it, does not remove it reliably. Enforcing physical consistency through the parameterization is thus also essential in practice.
	\end{enumerate}

	\noindent\textbf{Organization.} Section~\ref{sec:related_work} discusses related work. Section~\ref{sec:conformal_classes} addresses the conditions on the reaction terms that are necessary for well-posedness of RD systems, and the incorporation of these conditions into user-defined, learnable classes of functions. The proposed model learning approach for RD systems is formulated in Section~\ref{sec:model_learning}. The concrete framework is presented in Subsection~\ref{subsec:wellposedness} and the convergence result in model learning is discussed in Subsection~\ref{subsec:consistency}. Numerical experiments are presented in Section~\ref{sec:numerics}. In Appendix~\ref{app:assum_reactions} we verify the assumptions for the proposed modified reaction terms necessary for the convergence result to apply, and Appendix~\ref{app:numerics} provides extended numerical experiments. In Appendix~\ref{sec:approx_quaspos} approximation results for quasipositive functions are investigated. Existence results for RD systems from literature are recalled in Appendix~\ref{app:ex_rdsys} and relevant results from literature on operators in Sobolev spaces are addressed in Appendix~\ref{app:op_sob}.

	\section{Related work}
	\label{sec:related_work}
	In this section we review related work on inverse problems for RD systems, on identifiability of learned models, and on conservation law guided model learning\\
	
	\noindent\textbf{Inverse problems for RD systems.} For RD systems of a prescribed structure, uniqueness and identifiability of coefficients in one spatial dimension is established in \cite{Cristofol2012, Friedman1992}. Specific application cases include heat conduction laws, where unique identifiability is studied from overspecified boundary data \cite{Cannon1980} and from a single additional boundary measurement \cite{Egger15}, with corresponding stability estimates in \cite{Egger15, Roesch96}, the recovery of parameters in systems developing Turing patterns \cite{CampilloFunollet2018, Garvie2010,Schnrr2023}, and tumor-growth applications, where parameter identification is used to estimate growth dynamics \cite{Gholami2015,Hogea2007}. Beyond a prescribed structure, and typically from indirectly measured data of the state $u=(u_n)_{n=1}^N$ only, a variety of related inverse problems have been studied. This includes the recovery of state-dependent source terms from overposed data \cite{DuChateau1985, Kaltenbacher2020}, extended to the joint identification of a state-dependent source term and a multiplicative spatially dependent coefficient in \cite{Kaltenbacher2025}, the simultaneous reconstruction of conductivity and a nonlinear reaction term \cite{Kaltenbacher2020b}, the identification of a nonlinear diffusion term \cite{Kaltenbacher2021b}, and inverse problems for semilinear reaction-diffusion equations under full \cite{Isakov1993, Kian2023} and, more recently, partial \cite{Feizmohammadi2024} boundary measurements. Several of these references also present classical reconstruction methods for the unknown quantities, see in addition \cite{Christof2024, Hinze2019, Kian2023b, NgocNguyen2025}. In contrast to this classical perspective, the overview articles cited at the beginning of Section~\ref{sec:intro} take a machine-learning view on the discovery of physical laws, methods specifically tailored to RD systems in this spirit are developed in \cite{Li2020,Rao23}.\\

	\noindent\textbf{Identifiability of learned reaction models.} The question of unique identification within a class of candidate reaction terms is addressed in \cite{kutyniok23,scholl_icassp}, with an extension to the noisy setting in \cite{Kut24}, which study the symbolic recovery of differential equations. These works focus on specific classes of models, such as linear and algebraic ones, and propose a robust classification approach based on Singular Value Decomposition to ensure identifiability. The regularization-based route of \cite{AHN23, morina_holler/online} described in Section~\ref{sec:intro} instead resolves uniqueness through prior information encoded in a regularization functional, and is the one we follow. A complementary strategy for this is to constrain the space of learnable models to those that are physically realistic, conservation laws inherent to PDE systems offer a natural framework for guiding this restriction.\\

	\noindent\textbf{Conservation law guided model learning.} Problems modeled by RD systems often feature intrinsic conservation laws and symmetries that are physically relevant. Typical conservation principles include conservation of the total mass of reacting species in chemical kinetics, and conservation of the total population in compartmental epidemic models, where individuals are only redistributed between compartments rather than created or destroyed. Relevant symmetries include, e.g., permutation invariance among interchangeable species, which reflects that physically indistinguishable agents obey the same dynamics. Together with the preservation of non-negativity, these are properties whose violation yields models that contradict the underlying physics even if they fit the observed data well, which makes their inclusion essential for learned reaction terms to be reliable and interpretable.

	 A substantial body of recent research investigates how to incorporate problem-specific conservation laws and symmetries into machine-learning frameworks for discovering underlying physical laws, providing a foundation for models that are both physically accurate and data-efficient. Soft-constraint approaches add the conservation condition as a regularization term to the loss, as in \cite{jagtap20} for flow continuity and similarly in \cite{Li2024, Wu2022}, such approaches enforce the constraint only approximately and therefore do not guarantee exact satisfaction of the conservation laws. Exact conservation is instead achieved by integrating the conservation law in its integral form and carrying out predictive updates with the full governing PDE \cite{Hansen2024}, by corrective mechanisms \cite{CardosoBihlo2025, Geng2024}, and by the adaptive correction procedure for Fourier neural operators of \cite{Schoenlieb2025}. An alternative line of research modifies the network architecture itself. The work \cite{Sturm21} imposes flux continuity as a hard constraint in the final layer, \cite{Negiar23} projects the network output onto a prescribed solution space, \cite{Mueller23} encodes symmetries via \textit{Noether's theorem}, and further architectural constructions are given in \cite{Liu23, Yu2023,Richterpowell2022}. Energy-preserving methods include Hamiltonian \cite{Greydanus19} and Lagrangian \cite{Cranmer2020} neural networks, and \cite{PhysRevLett.126.098302, Beucler2019b} constrain either the loss or the architecture with applications in climate modeling. %

	\section{Physically consistent classes}
	\label{sec:conformal_classes}
	The main focus of this work is to address the problem of learning a reaction term $f=(f_n)_{n=1}^N$, together with the state $u=(u_n)_{n=1}^N$, initial condition $u_0=(u_{0,n})_{n=1}^N$ and diffusion coefficient $D=(d_n)_{n=1}^N$, satisfying the reaction-diffusion system
	\begin{equation}
		\label{gen:rdsystem}
		\tag{\text{RD}}
		\begin{aligned}
			\partial_tu_n & = d_n\Delta u_n+f_n(u) \\
			u_n(0) &= u_{0,n},
		\end{aligned}\qquad \text{for} ~ ~ n=1,\dots, N.
	\end{equation}
	Which conditions on $f$ have to be embedded into the learned model is motivated by the well-posedness theory of \eqref{gen:rdsystem}. In this context, the survey \cite{Pierre2010} provides a comprehensive overview for establishing \textit{well-posedness} for RD systems, i.e., the global existence of solutions. Well-posedness can be achieved under the following sufficient conditions imposed on the reaction terms $f = (f_n)_{n=1}^N$ in \eqref{pdemodel}: i) sufficient regularity, namely local Lipschitz continuity, ii) a growth condition, iii) a mass-control condition, and iv) \textit{quasipositivity}. While iii) ensures that the total mass of the system remains bounded %
	and iv) guarantees that non-negativity of initial conditions is preserved for the solution $u$ of \eqref{pdemodel}, and thus also have physical significance, conditions i) and ii) are primarily technical. Relaxing these conditions is an area of ongoing research, yet they remain pivotal for guaranteeing well-posedness. Further references on global existence include \cite{Souplet2018}, which investigates well-posedness under mass dissipation and quadratic growth conditions, \cite{Laamri2020}, focusing on systems under low regularity initial data, \cite{Fischer2015}, exploring entropy-dissipating RD systems, \cite{Fellner2023,Laamri2017}, addressing systems with nonlinear diffusion, and \cite{Goudon2010}, which provides a general regularity analysis of RD systems.

	Given the importance of the conditions i)-iv) for ensuring well-posedness of the underlying RD system, we propose a framework that directly incorporates them within the parameterized reaction term $f_\theta$. Here, it is important to ensure \textit{physical consistency} of the learning approach in the sense that for a learned reaction term $f$, the resulting system in \eqref{gen:rdsystem} is well-posed, i.e., attains a solution $u$ of suitable regularity. To address this, we introduce a framework for learning the reaction term while preserving the conditions required for well-posedness. This framework is built around a user-defined, learnable class of functions $\mathcal{F}\subset\left\{f:\mathbb{R}^N\to\mathbb{R}^N\right\}$ and provides a modified, learnable class $\overline{\mathcal{F}}$, derived from $\mathcal{F}$, whose instances satisfy the critical conditions addressed above. In view of well-posedness results for RD systems we refer to \cite[Theorem 1.1]{Fellner2020} for the existence of classical solutions under Neumann boundary conditions and \cite[Theorem 1]{Suzuki2017} for the existence of weak solutions under smooth Dirichlet boundary conditions, summarized in Appendix~\ref{app:ex_rdsys} for the sake of completeness. These works show existence of a unique solution to \eqref{gen:rdsystem} for bounded and non-negative initial data $u_{0,n}$ (which additionally needs to be integrable in case of the classical result in \cite{Fellner2020}), under the following conditions on $\mathcal{F}$ with $\Vert\cdot\Vert$ denoting the Euclidean norm in $\mathbb{R}^N$:\\
	
	\noindent\textbf{Local Lipschitz continuity.} The class $\mathcal{F}$ satisfies condition \eqref{cond:lipschitz} if for each $f\in \mathcal{F}$ and any $M>0$ there exists $L_M>0$ such that
	\begin{equation}
		\label{cond:lipschitz}
		\Vert f(u)-f(v)\Vert \leq L_M\Vert u-v\Vert \qquad \text{for all } ~ u,v\in \mathbb{R}^N ~ \text{with} ~ \Vert u\Vert, \Vert v\Vert <M.\tag{$\mathcal{L}$}
	\end{equation}
	This condition is required for deriving local existence of solutions to \eqref{gen:rdsystem} for bounded initial data (see results in \cite[Part I]{Rothe1984}, \cite[(2.1) Theorem]{Amann1985} or the general work \cite{Pazy1983}).\\
	
	\noindent\textbf{Quasipositivity.} The class $\mathcal{F}$ satisfies condition \eqref{cond:quasipos} if for each $f\in\mathcal{F}$ it holds true%
	\begin{equation}
		\label{cond:quasipos}
		f_n(u_1,\dots, u_{n-1}, 0, u_{n+1}, \dots, u_N)\geq 0  \qquad \text{for all } ~ u_i\geq 0\quad \text{and} ~ 1\leq n\leq N. \tag{$\mathcal{Q}$}
	\end{equation}
	This condition ensures that non-negativity of the initial data of \eqref{gen:rdsystem} is preserved for a solution as long as it exists. We refer to Definition~\ref{def:quasipositive} in Appendix~\ref{sec:approx_quaspos} for a more general formulation of quasipositivity.\\
	
	\noindent\textbf{Mass control.} The class $\mathcal{F}$  meets condition \eqref{cond:mass} if there exist $K_0\geq 0$, $K_1\in\mathbb{R}$ and $(c_n)_{n=1}^N\subset~]0,\infty[$ such that for each $f\in\mathcal{F}$ it holds true that
	\begin{align}
		\label{cond:mass}\tag{$\mathcal{M}$}
		\sum_{n=1}^N c_nf_n(u)\leq K_0+K_1\sum_{n=1}^N u_n\qquad \text{for all} ~ u_n\geq 0.
	\end{align}
	Together with the growth condition \eqref{cond:growth} discussed next, condition \eqref{cond:mass} rules out finite-time blow-up of $\sum_{n=1}^N c_n u_n$. In fact, by integrating \eqref{pdemodel} and applying Grönwall, one obtains an at most exponential-in-time bound (see \cite{Pierre2010}). We distinguish three regimes of \eqref{cond:mass} used throughout the paper as follows.
	\begin{itemize}[leftmargin=0.5cm]
	\item \textit{Control} ($K_0\ge 0$, $K_1\in\mathbb{R}$ arbitrary): The generic case stated in \eqref{cond:mass}, under which $\sum c_n u_n$ may grow exponentially in time but remains finite on $[0,T]$.
	\item \textit{Dissipation} ($c_n=1$, $K_0=K_1=0$): $\sum u_n$ is non-increasing.
	\item \textit{Conservation} ($c_n=1$, $K_0=K_1=0$ with equality in \eqref{cond:mass}): $\sum u_n$ is preserved.
	\end{itemize}
	The stronger variants of dissipation/conservation are physically motivated and will be invoked explicitly when needed (as in Corollary~\ref{cor:wellposedness2} and Theorem~\ref{th:suzuki}).\\
	
	\noindent\textbf{Growth condition.} The class $\mathcal{F}$ satisfies the (quadratic growth) condition \eqref{cond:growth} if for each $f\in \mathcal{F}$ there exists $K>0$ such that
	\begin{align}
		\label{cond:growth}\tag{$\mathcal{G}$}
		\Vert f(u)\Vert \leq K(1+\Vert u\Vert^2) \qquad \text{for all} ~ u\in\mathbb{R}^N.
	\end{align}
	We emphasize that \eqref{cond:lipschitz}, \eqref{cond:quasipos}, \eqref{cond:mass}, \eqref{cond:growth}, as adopted here, are sufficient for the existence and non-negativity results recalled above, but not, in general, necessary. Quasipositivity \eqref{cond:quasipos} is necessary for non-negativity preservation of solutions to \eqref{pdemodel}, and local Lipschitz continuity \eqref{cond:lipschitz} is required for the local existence theory used here. By contrast, the quadratic growth condition \eqref{cond:growth} is known to be relaxable, super-quadratic reactions can still yield global solutions under additional structural assumptions such as entropy dissipation (see e.g.\ \cite{Fischer2015, Souplet2018}), and the mass-control condition \eqref{cond:mass} can be weakened in specific settings, for instance in the presence of nonlinear diffusion \cite{Fellner2023, Laamri2017}.

	In practice, a standard class $\mathcal{F}$ of parameterized functions, such as neural networks, typically does not satisfy these conditions. To construct a physically consistent class $\overline{\mathcal{F}}$ that meets the conditions \eqref{cond:lipschitz}, \eqref{cond:quasipos}, \eqref{cond:mass}, and \eqref{cond:growth}, we first introduce the concept of a smooth transition function.
	\begin{definition}[Transition function]
		\label{def:transition}
		We call $\chi\in\mathcal{C}^\infty(\mathbb{R},\mathbb{R})$ a transition function if there exist $0<\delta<\epsilon$ such that $\chi(x)=1$ for $x\leq\epsilon-\delta$, $\chi(x)=0$ for $x\geq\epsilon+\delta$ and $\frac{\dx}{\dx x}\chi(x)< 0$ for $\epsilon-\delta<x<\epsilon+\delta$.
	\end{definition}
	A concrete example of a transition function, the mollified Heaviside function $\tilde{h}_\epsilon$, is constructed before Lemma~\ref{lemma:chi_derivative} below. The following general result outlines a method for constructing a physically consistent class $\overline{\mathcal{F}}$ from a given class $\mathcal{F}$, assuming that the elements of $\mathcal{F}$ are Lipschitz continuous. Throughout, we write $P_+:\mathbb{R}\to\mathbb{R}$ for the positive-part function $P_+(x)=\max(x,0)$, and $P_-(x)=\min(x,0)$ for the negative-part function.
	\begin{lemma}
		\label{lem:auxiliary_result}
		Let $f:\mathbb{R}^N\to \mathbb{R}^N$ be Lipschitz continuous and $\chi$ a transition function. Denote $\chi_n(u)=\chi(u_n)$ for $u\in\mathbb{R}^N$. Then the function $\bar{f} = (\bar f_1,\ldots,\bar f_N)$ defined by
		\begin{align}
			\label{modification_f}
			\bar{f}_n = (P_+\circ f_n-f_n)\cdot \chi_n+f_n\quad \text{for} ~ n=1,\ldots, N,
		\end{align}
		fulfills the conditions \eqref{cond:lipschitz}, \eqref{cond:quasipos}, \eqref{cond:mass} and \eqref{cond:growth} with suitable parameters.
	\end{lemma}
	\begin{proof}
		There exists $L>0$ such that $\vert f_n(u)-f_n(v)\vert \leq L\Vert u-v\Vert$ for $1\leq n\leq N$ for $u,v\in\mathbb{R}^N$. Since $f_n=P_+\circ f_n+P_-\circ f_n$, and using $\vert (P_-\circ f_n)(u)\vert\leq \vert f_n(u)\vert$ for $u\in\mathbb{R}^N$, we derive
		\begin{align*}
			\vert \bar{f}_n(u)-\bar{f}_n(v)\vert&\leq \vert (P_-\circ f_n)(u)\chi_n(u)-(P_-\circ f_n)(v)\chi_n(v)\vert +\vert f_n(u)-f_n(v)\vert\\
			&\leq \vert (P_-\circ f_n)(u)\vert \vert \chi_n(u)-\chi_n(v)\vert +(\vert\chi_n(v)\vert+1) \vert f_n(u)-f_n(v)\vert\\
			&\leq \vert f_n(u)\vert \Vert\chi'\Vert_\infty\Vert u-v\Vert +2L\Vert u-v\Vert\\
			&\leq (\vert f_n(u)-f_n(0)\vert+\vert f_n(0)\vert) \Vert\chi'\Vert_\infty\Vert u-v\Vert +2L\Vert u-v\Vert\\
			&\leq [(L\Vert u\Vert+\vert f_n(0)\vert) \Vert\chi'\Vert_\infty+2L]\Vert u-v\Vert,
		\end{align*}
		proving local Lipschitz continuity in \eqref{cond:lipschitz}. Quasipositivity as in \eqref{cond:quasipos} follows since for fixed $n=1,\ldots, N$ and any $u\in \mathbb{R}^N$ with $u_n=0$ it holds true that
		\[
			\bar{f}_n(u) = (P_+(f_n(u))-f_n(u))\cdot \chi_n(u)+f_n(u)= P_+(f_n(u))\geq 0
		\]
		due to $\chi_n(u)=\chi(u_n)=\chi(0)=1$. To see \eqref{cond:mass} note that 
		\begin{multline}
			\label{estimation:f_mod}
			\bar{f}_n(u)=(P_+\circ f_n)(u)+(1-\chi_n(u))(P_-\circ f_n)(u)\leq (P_+\circ f_n)(u)\\
			\leq\vert (P_+\circ f_n)(u)-(P_+\circ f_n)(0)\vert + (P_+\circ f_n)(0)\leq L\Vert u\Vert +\beta_n
		\end{multline}
		with $\beta_n:= (P_+\circ f_n)(0)$. Thus, for any $(c_n)_{n=1}^N\subset~ ]0,\infty[$ it holds for $u\in [0,\infty[^N$
		\[
		\sum_{n=1}^N c_n\bar{f}_n(u) \leq \sum_{n=1}^N c_n\beta_n+L\sum_{n=1}^{N}c_n\Vert u\Vert\leq K_0+L\sqrt{N}\sum_{n=1}^{N}c_n \sum_{n=1}^{N}u_n = K_0+K_1\sum_{n=1}^N u_n
		\]
		 with $K_0= \sum_{n=1}^N c_n\beta_n$ and $K_1=L\sqrt{N}\sum_{n=1}^{N}c_n$, using $\Vert u\Vert \leq \sqrt{N}\sum_{n=1}^{N}\vert u_n\vert$ for $u\in \mathbb{R}^N$, which proves \eqref{cond:mass}. Finally, \eqref{cond:growth} follows with $K:=4\max(L, \vert f_n(0)\vert)$ from
		 \begin{multline*}
			\vert \bar{f}_n(u)\vert = \vert (P_+\circ f_n)(u)+(1-\chi_n(u))(P_-\circ f_n)(u)\vert\\ \leq \vert (P_+\circ f_n)(u)\vert+\vert (P_-\circ f_n)(u)\vert\leq 2\vert f_n(u)\vert \leq 2(L\Vert u\Vert+\vert f_n(0)\vert).
		\qedhere
		\end{multline*}
	\end{proof}
	\begin{remark}
		\label{rem:local_lipschitz}
		Lemma~\ref{lem:auxiliary_result} can be extended to locally Lipschitz continuous $f$, which is relevant, e.g., for neural networks with ReQU activation (see Subsection~\ref{subsec:consistency}). Conditions \eqref{cond:lipschitz} and \eqref{cond:quasipos} follow for $\bar f$ by the same arguments as in the proof, and since $\vert\bar f_n\vert\leq 2\vert f_n\vert$, also \eqref{cond:growth} follows for $\bar f$ whenever $f$ satisfies a growth condition of the same type. Only the mass control \eqref{cond:mass} requires an additional assumption, since its proof uses the global bound $\vert f_n(u)\vert\leq L\Vert u\Vert+\vert f_n(0)\vert$. Due to $\bar f_n\leq P_+\circ f_n$ it suffices that \eqref{cond:mass} holds for $(P_+\circ f_n)_{n=1}^N$, i.e., that the production rates are controlled by the current mass, which is physically reasonable.
	\end{remark}
	Given a class $\mathcal{F}$ of Lipschitz continuous functions, the result in Lemma~\ref{lem:auxiliary_result} suggests defining the physically consistent class $\overline{\mathcal{F}}$ by $\overline{\mathcal{F}}=\left\{\bar{f}: f\in \mathcal{F}\right\}$, where $\bar{f}=(\bar{f}_n)_{n=1}^N$ is constructed according to \eqref{modification_f}. When using functions from $\overline{\mathcal{F}}$ to approximate an unknown reaction term within a model learning framework, it is essential that the modified class $\overline{\mathcal{F}}$ preserves the required approximation properties of $\mathcal{F}$, e.g.\ the ability to approximate continuous functions on compact domains. The next lemma addresses this crucial aspect under abstract and general prerequisites, after which we will present an interpretable and explicit scenario that fulfills these conditions.
	\begin{lemma}
		\label{lem:abstract_approximation}
		Let $f:\mathbb{R}^N\to \mathbb{R}^N$ be Lipschitz continuous and $U\subseteq \mathbb{R}^N$ a compact subset. Furthermore, let $(f^m)_{m\in\mathbb{N}}$ be a sequence of Lipschitz continuous functions $f^m:\mathbb{R}^N\to \mathbb{R}^N$ for $m\in\mathbb{N}$ such that
		\[
			\lim_{m\to\infty}\Vert f-f^m\Vert_{L^\infty(U)}=0\quad \text{and} \quad \limsup_{m\to\infty}\Vert \nabla f^m\Vert_{L^\infty(U)}\leq \Vert \nabla f\Vert_{L^\infty(U)}.
		\]
		Choose for two monotone zero sequences $(\epsilon_m)_m$, $(\delta_m)_m$, with $0<\delta_m<\epsilon_m$ for $m\in\mathbb{N}$, transition functions $(\chi^m)_{m\in\mathbb{N}}$ satisfying Definition~\ref{def:transition} with $\epsilon_m$ and $\delta_m$, for $m\in\mathbb{N}$, respectively. Suppose that for $\Gamma_\epsilon^n:=\left\{u\in U: \vert u_n\vert\leq \epsilon\right\}$ it holds true that
		\begin{align}
			\label{abstract_condition}
			\Vert P_-\circ f_n\Vert_{L^\infty(\Gamma_{\epsilon_m+\delta_m}^n)}+\Vert f-f^m\Vert_{L^\infty(U)}=o(\Vert \frac{\dx}{\dx x}\chi^m\Vert_{\mathcal{C}(\mathbb{R})}^{-1})\quad \text{as} ~ m\to\infty.
		\end{align}
		Then for $\bar{f}^m = (\bar f_1^m,\ldots,\bar f_N^m)$ being defined by
		\begin{align}
			\label{def:fnm}
			\bar{f}_n^m = (P_+\circ f_n^m-f_n^m)\cdot \chi_n^m+f_n^m
		\end{align}
		with $\chi_n^m(u)=\chi^m(u_n)$ for $u\in\mathbb{R}^N$, $n=1,\ldots, N$ and $m\in\mathbb{N}$, it holds true that
		\[
		\lim_{m\to\infty}\Vert f-\bar{f}^{m}\Vert_{L^\infty(U)}=0\quad \text{and} \quad \limsup_{m\to\infty}\Vert \nabla \bar{f}^{m}\Vert_{L^\infty(U)}\leq \Vert \nabla f\Vert_{L^\infty(U)}.
		\]
	\end{lemma}
	\begin{proof}
		We start by showing that $(\bar{f}^m)_m$ converges uniformly to $f$ in $U$. For that, using the definition of $\bar{f}^m_n$ in \eqref{def:fnm}, estimate for $n=1,\dots, N$ and $m\in \mathbb{N}$,
		\[
			\Vert f_n-\bar{f}^m_n\Vert_{L^\infty(U)}\leq \Vert (P_+\circ f_n^m-f_n^m)\cdot \chi_n^m\Vert_{L^\infty(U)}+\Vert f_n-f_n^m\Vert_{L^\infty(U)}.
		\]
		The second term converges to zero as $m\to \infty$ by assumption. As $0\leq\chi_n^m\leq 1$ is supported in $\Gamma^n_{\epsilon_m+\delta_m}$ by Definition~\ref{def:transition} and $P_+(z)+P_-(z)=z$ for $z\in\mathbb{R}$, the first term is bounded by
		\[
			\Vert P_-\circ f_n^m\Vert_{L^\infty(\Gamma^n_{\epsilon_m+\delta_m})}
		\]
		from above, and can be further estimated by
		\[
			\Vert P_-\circ f_n\Vert_{L^\infty(\Gamma^n_{\epsilon_m+\delta_m})}+\Vert P_-\circ (f_n^m-f_n)\Vert_{L^\infty(\Gamma^n_{\epsilon_m+\delta_m})}.
		\]
		The second term converges to zero as $m\to \infty$ again by assumption, since $P_-$ is Lipschitz continuous with constant one and $\Gamma^n_{\epsilon_m+\delta_m}\subseteq U$. To see that the first term converges to zero, it suffices by \eqref{abstract_condition} to show that $(\Vert \frac{\dx}{\dx x}\chi^m\Vert_{\mathcal{C}(\mathbb{R})}^{-1})_m$ is a zero sequence. This is a consequence of the mean value theorem which in fact implies that $\Vert \frac{\dx}{\dx x}\chi^m\Vert_{\mathcal{C}(\mathbb{R})}\geq (2\delta_m)^{-1}$ for $m\in \mathbb{N}$ by Definition~\ref{def:transition}. It remains to prove that $\limsup_{m\to\infty}\Vert \nabla \bar{f}^{m}\Vert_{L^\infty(U)}\leq \Vert \nabla f\Vert_{L^\infty(U)}$. Due to \eqref{def:fnm} it holds true with $\mathcal{N}_{n,m}=\left\{z\in\mathbb{R}^N: f_n^m(z)<0\right\}$ and $e_n$ the $n$-th unit vector in $\mathbb{R}^N$ that
		\begin{align*}
			\nabla \bar{f}_n^m(u) = (1-1_{\mathcal{N}_{n,m}}(u))\chi^m(u_n)\nabla f_n^m(u)
			-(P_-\circ f_n^m)(u) (\chi^m)'(u_n)e_n.
		\end{align*}
		Here $1_A$ denotes the indicator on the set $A\subseteq\mathbb{R}^N$, i.e., with $1_A(z)=1$ if $z\in A$ and $1_A(z)=0$ otherwise. By similar estimations as before, the second term can be uniformly bounded by
		\begin{align*}
		\Vert (P_-\circ f_n^m)(u) (\chi^m)'(u_n)\Vert_{L^\infty(U)}
			\leq (\Vert P_-\circ f_n\Vert_{L^\infty(\Gamma_{\epsilon_m+\delta_m}^n)}+\Vert f_n^m-f_n\Vert_{L^\infty(U)})\Vert (\chi^m)'\Vert_{\mathcal{C}(\mathbb{R})}
		\end{align*}
		which converges to zero by \eqref{abstract_condition}. With this, we conclude by $\vert (1-1_{\mathcal{N}_{n,m}}(u))\chi^m(u_n)\vert \leq 1$ that
		\[\limsup_{m\to \infty}\Vert \nabla \bar{f}^m\Vert_{L^\infty(U)}\leq \limsup_{m\to \infty}\Vert \nabla f^m\Vert_{L^\infty(U)} \leq \Vert \nabla f\Vert_{L^\infty(U)}.\qedhere\]
	\end{proof}
	Next, we provide an interpretable setup that satisfies the abstract condition specified in \eqref{abstract_condition}. To achieve this, we choose a concrete class of transition functions $(\chi^m)_m$ (recall Definition~\ref{def:transition}). Specifically, consider a fixed, non-negative convolution kernel $\eta\in\mathcal{C}^\infty(\mathbb{R},\mathbb{R})$ with $\eta(x)=0$ for $x\notin [-1,1]$, $x\eta'(x)\leq 0$ for $x\in \mathbb{R}$ and $\int_\mathbb{R}\eta(x)\dx x=1$. Define for $\epsilon>0$ the $\epsilon$-width kernels $\eta_\epsilon$ with $\eta_\epsilon(x)=2/\epsilon\cdot\eta(2x/\epsilon)$ for $x\in \mathbb{R}$. With this and $h_\epsilon$ denoting the Heaviside step function which attains the value $1$ on the interval $]-\infty,\epsilon[$ and $0$ otherwise, it can be easily shown that the convolution $\tilde{h}_\epsilon:=h_\epsilon*\eta_{\epsilon}$ defines a transition function (with $\delta=\epsilon/2$ in Definition~\ref{def:transition}). We will choose the functions $(\chi^m)_m$ above as $(\tilde{h}_{\epsilon_m})_m$ for a suitable sequence $(\epsilon_m)_m$. First, we require an asymptotic estimation of the $\mathcal{C}^1$-norm of $\tilde{h}_\epsilon$ as $\epsilon\to 0^+$. 
	\begin{lemma}
		\label{lemma:chi_derivative}
		It holds true that $\Vert \tilde{h}_\epsilon\Vert_{\mathcal{C}^1(\mathbb{R})} = O(\epsilon^{-1})$ for sufficiently small $\epsilon>0$.
	\end{lemma}
	\begin{proof}
		Note that since $\eta_\epsilon(x)=2/\epsilon\cdot \eta(2x/\epsilon)$ for $x\in \mathbb{R}$ it holds true that $\eta'_\epsilon(x)=4/\epsilon^2\eta'(2x/\epsilon)$ for $x\in \mathbb{R}$ and as a consequence $\Vert \eta'_\epsilon\Vert_{\mathcal{C}(\mathbb{R})}=4c/\epsilon^2$ with $c=\Vert\eta'\Vert_{\mathcal{C}(\mathbb{R})}<\infty$. By employing Young's inequality and using that $\eta_\epsilon'$ is supported in the interval $[-\epsilon/2,\epsilon/2]$ it follows that
		\[
		\Vert \tilde{h}'_\epsilon\Vert_{L^\infty(\mathbb{R})}=\Vert h_\epsilon*\eta_\epsilon'\Vert_{L^\infty(\mathbb{R})} \leq \Vert h_\epsilon\Vert_{L^\infty(\mathbb{R})}\Vert \eta_\epsilon'\Vert_{L^1(\mathbb{R})}\leq \Vert \eta_\epsilon'\Vert_{L^1(\mathbb{R})}\leq \epsilon\cdot 4c/\epsilon^2 = 4c/\epsilon.
		\] 
		Hence, we conclude that $\Vert \tilde{h}_\epsilon\Vert_{\mathcal{C}^1(\mathbb{R})}\leq 4c/\epsilon$ for sufficiently small $\epsilon>0$. 
	\end{proof}
	This result essentially provides a characterization of the decay in \eqref{abstract_condition}. %
	Additionally, we require the following notion of strict quasipositivity with rate, which implies quasipositivity in \eqref{cond:quasipos} and imposes sufficient decay of $f$ in a suitable sense.
	\begin{assumption}
		\label{ass:strict_quasipositivity}
		Let $U$ be a bounded Lipschitz domain and $f=(f_n)_{n=1}^N:\mathbb{R}^N\to \mathbb{R}^N$. Denote $\Gamma_\epsilon^n:=\left\{u\in U: \vert u_n\vert\leq \epsilon\right\}$ for $\epsilon>0$ and $1\leq n\leq N$. Suppose that $f$ is strictly quasipositive with rate $\alpha>1$, i.e., there exists $c >0$ with
		\begin{align}
			\label{eq:strict_quasi}
			\Vert P_-\circ f_n\Vert_{L^\infty(\Gamma_{\epsilon}^n)}\leq c\epsilon^{\alpha}\quad \text{as} ~ \epsilon\to0^+ \quad \text{for all} ~ 1\leq n\leq N.
		\end{align}
		
	\end{assumption}
	We now turn to a central result of this section. The modification preserves both the uniform approximation rate $m^{-\beta}$ and the Lipschitz constant of the target, with an explicit dependence on the strict-quasipositivity exponent $\alpha$.
	\begin{proposition}
		\label{cor:approximation_inherit}
		Let $f:\mathbb{R}^N\to \mathbb{R}^N$ be Lipschitz continuous and $U\subseteq \mathbb{R}^N$ a compact subset. Furthermore, let $(f^m)_{m\in\mathbb{N}}$ be a sequence of Lipschitz continuous functions $f^m:\mathbb{R}^N\to \mathbb{R}^N$ for $m\in\mathbb{N}$ such that there exist $c,\beta>0$ with
		\[
		\Vert f-f^m\Vert_{L^\infty(U)}\leq cm^{-\beta} ~ \text{for} ~ m\in\mathbb{N}\quad \text{and} \quad \limsup_{m\to\infty}\Vert \nabla f^m\Vert_{L^\infty(U)}\leq \Vert \nabla f\Vert_{L^\infty(U)}.
		\]
		Assume that $f$ satisfies Assumption~\ref{ass:strict_quasipositivity} with rate $\alpha>1$ and set $\chi^m =\tilde{h}_{\epsilon_m}$ for $m\in\mathbb{N}$ with $(\epsilon_m)_m = (m^{-\beta/\alpha})_m$. Then for $(\bar{f}^m)_m$ defined as in Lemma~\ref{lem:abstract_approximation}, it holds
		\[
		\Vert f-\bar{f}^{m}\Vert_{L^\infty(U)}\leq cm^{-\beta} ~ \text{for} ~ m\in\mathbb{N}\quad \text{and} \quad \limsup_{m\to\infty}\Vert \nabla \bar{f}^{m}\Vert_{L^\infty(U)}\leq \Vert \nabla f\Vert_{L^\infty(U)}.
		\]
	\end{proposition}
	\begin{proof}
		Due to Assumption~\ref{ass:strict_quasipositivity} and Lemma~\ref{lemma:chi_derivative} it holds some $c>0$ that
		\begin{align}
			\label{P-estimation0}
			\Vert P_-\circ f_n\Vert_{L^\infty(\Gamma_{3\epsilon_m/2}^n)}\leq c\epsilon_m^{\alpha} \leq c\epsilon_m\leq \Vert (\chi^m)'\Vert_{\mathcal{C}(\mathbb{R})}^{-1}
		\end{align}
		for sufficiently large $m\in\mathbb{N}$ and $1\leq n\leq N$. By assumption there exist $c,\beta >0$ with $\Vert f-f^m\Vert_{L^\infty(U)}\leq cm^{-\beta}$ and $\limsup_{m\to\infty}\Vert \nabla f^m\Vert_{L^\infty(U)}\leq \Vert \nabla f\Vert_{L^\infty(U)}$. Due to Lemma~\ref{lemma:chi_derivative} we have $\Vert (\chi^m)'\Vert_{\mathcal{C}(\mathbb{R})}^{-1}\geq c m^{-\beta/\alpha}$ for large $m\in\mathbb{N}$ implying
		\begin{align*}
			\Vert f-f^m\Vert_{L^\infty(U)}=o(\Vert (\chi^m)'\Vert_{\mathcal{C}(\mathbb{R})}^{-1})
		\end{align*}
		as $m\to \infty$ by $\alpha>1$. Certainly, by the representation in \eqref{def:fnm} we can estimate
		\begin{align}
			\label{first_bar_estimation0}
			\Vert f- \bar f^m\Vert_{L^\infty(U)} \leq \Vert f - f^m\Vert_{L^\infty(U)} +\max_{1\leq n\leq N}\Vert (P_-\circ f_n^m)(u)\chi^m(u_n)\Vert_{L^\infty(U)}.
		\end{align}
		In view of the second term the triangle inequality yields
		\begin{multline*}
			\Vert (P_-\circ f_n^m)(u)\chi^m(u_n)\Vert_{L^\infty(U)}\leq \Vert (P_-\circ f_n)(u)\chi^m(u_n)\Vert_{L^\infty(U)}\\
			+\Vert (P_-\circ f_n^m-P_-\circ f_n))(u)\chi^m(u_n)\Vert_{L^\infty(U)}
		\end{multline*}
		for $1\leq n\leq N$, which by using that $\chi^m(u_n)=0$ for $u_n\geq 3\epsilon_m/2$, $\Vert \chi^m\Vert_{\mathcal{C}(\mathbb{R})}\leq 1$ and \eqref{P-estimation0} together with $\epsilon_m=m^{-\beta/\alpha}$ for $m\in\mathbb{N}$ can be further estimated by
		\[
		\Vert P_-\circ f_n\Vert_{L^\infty(\Gamma_{3\epsilon_m/2}^n)}+\Vert f_n^m-f_n\Vert_{L^\infty(U)}\leq c m^{-\beta}
		\]
		since $P_-$ is Lipschitz continuous with constant one. Note that the rate follows by assumption and \eqref{P-estimation0}. Thus, we can estimate the term in \eqref{first_bar_estimation0} for $m\in\mathbb{N}$ by
		\[
		\Vert f- \bar f^m\Vert_{L^\infty(U)} \leq c m^{-\beta}.
		\]
		The convergence $\limsup_{m\to\infty}\Vert \nabla \bar{f}^m\Vert_{L^\infty(U)}\leq \Vert \nabla f\Vert_{L^\infty(U)}$ follows as in Lemma~\ref{lem:abstract_approximation}.
	\end{proof}
	\begin{remark}
		\label{rem:chi_family}
		The concrete construction $\tilde{h}_\epsilon$ enters Proposition~\ref{cor:approximation_inherit} only through the bound $\Vert (\chi^m)'\Vert_{\mathcal{C}(\mathbb{R})}=O(\epsilon_m^{-1})$ of Lemma~\ref{lemma:chi_derivative}, the property $\supp (\chi^m)'\subseteq[\epsilon_m-\delta_m,\epsilon_m+\delta_m]$ with $\delta_m=\epsilon_m/2$, and $0\leq\chi^m\leq 1$. Both results therefore hold verbatim for any sequence of transition functions with these properties. %
		In particular the quintic smoothstep used in Subsection~\ref{subsec:num_setup} is admissible, being $\mathcal{C}^2$ with $\delta=\epsilon/2$ and $\Vert\chi'\Vert_{\mathcal{C}(\mathbb{R})}=15/(8\epsilon)$.
	\end{remark}
	To conclude this section, we note a key implication of Lemma~\ref{lem:auxiliary_result}, namely that RD systems resulting from reaction terms $\bar{f}\in\overline{\mathcal{F}}$ are well-posed.
\begin{corollary}[Classical solutions]
	\label{cor:wellposedness} Let $\Omega\subseteq \mathbb{R}^d$ be a bounded domain with smooth boundary such that $\Omega$ lies locally on one side of $\partial\Omega$.
	Suppose that $u_0\in L^1(\Omega)\cap L^\infty(\Omega)$ is non-negative and that $D = (d_n)_{n=1}^N\subset~ ]0,\infty[$. Let further $\bar{f}=(\bar{f}_n)_{n=1}^N$ be given as in \eqref{modification_f} with $f:\mathbb{R}^N \rightarrow \mathbb{R}^N$ Lipschitz continuous. Then the RD system
	\begin{equation}
		\label{rdsystem}
		\begin{aligned}
			\partial_tu- D\Delta u&=\bar{f}(u),  &&(t,x)\in ]0,T[\times\Omega,\\
			\nabla_x u\cdot \nu& = 0, && (t,x)\in ]0,T[\times\partial\Omega,\\
			u(0) &= u_{0}, && x\in\Omega,
		\end{aligned}
	\end{equation}
	attains for all $p>N$ a unique global classical solution
	\begin{align}
		\label{u_regularity}
	u=(u_n)_{n=1}^N\subseteq \mathcal{C}(0,T;L^p(\Omega)\cap L^\infty(\Omega))\cap\mathcal{C}^{1,2}(]0,T[\times\overline{\Omega}).
	\end{align}
\end{corollary}
	\begin{proof}
		The statement follows by Lemma~\ref{lem:auxiliary_result} and Theorem~\ref{th:fellner}.
	\end{proof}
	\begin{corollary}[Weak solutions]
		\label{cor:wellposedness2}
		Let $\Omega\subseteq \mathbb{R}^d$ be a bounded domain with smooth boundary. Suppose that $u_0\in L^1(\Omega)\cap L^\infty(\Omega)$ is non-negative and that the boundary data $g=(g_n)_{n=1}^N$ admits non-negative extensions $G=(G_n)_{n=1}^N\subseteq \mathcal{C}^1([0,T];\mathcal{C}(\overline{\Omega}))\cap \mathcal{C}([0,T];\mathcal{C}^2(\overline{\Omega}))$ satisfying $\partial_t G_n-d_n\Delta G_n=0$ in $]0,T[\times\Omega$ and $G_n=g_n$ on $]0,T[\times\partial\Omega$ for $1\leq n\leq N$. Let $\bar{f}=(\bar{f}_n)_{n=1}^N$ be given as in \eqref{modification_f} and assume that the underlying $f:\mathbb{R}^N \rightarrow \mathbb{R}^N$ is Lipschitz continuous. Suppose further that $\bar{f}$ fulfills mass dissipation in \eqref{cond:mass} and that $D = (d_n)_{n=1}^N\subset ~ ]0,\infty[$. Then the RD system
		\begin{equation}
			\label{rdsystem2}
		\begin{aligned}
				\partial_tu-D\Delta u &= f(u), &&(t,x)\in ]0,T[\times\Omega,\\
				u&=g, && (t,x)\in ]0,T[\times\partial\Omega,\\
				u(0)&=u_{0}, && x\in \Omega,
		\end{aligned}
	\end{equation}
	 admits a global weak solution
	 \begin{align}
	 	\label{u_regularity2}
	 	u =(u_n)_{n=1}^N\subseteq \mathcal{C}(0,T;L^1(\Omega))\cap L^2(]0,T[\times\overline{\Omega}).
	\end{align}
	\end{corollary}
	\begin{proof}
		The statement follows by Lemma~\ref{lem:auxiliary_result} and Theorem~\ref{th:suzuki}.
	\end{proof}
	\section{Physically consistent model learning}
	\label{sec:model_learning}
	In this section, we develop a framework for learning reaction terms in RD systems from data while ensuring physical consistency, i.e., that the conditions outlined in Section~\ref{sec:conformal_classes} are satisfied (yielding well-posedness of the resulting RD systems). In this context, a desirable property is that solutions to the learning problem converge to a unique, regularization-minimizing solution in the limit of full, noiseless measurements. The results presented here build upon the work in \cite{morina_holler/online}, which discusses this property in a more general learning setup. Our contributions here relate to \cite{morina_holler/online} as follows. The convergence statement itself, Theorem~\ref{thm:main_result}, is an application of \cite[Theorem 31]{morina_holler/online}. What is new is that the physically consistent classes $\overline{\mathcal{F}}_n^m$ of Section~\ref{sec:conformal_classes} \emph{admit} such an application at all. The definition of $\overline{\mathcal{F}}_n^m$ makes it necessary to show that these classes induce well-defined Nemytskii operators, are weak-strong continuous and are $W^{1,\infty}_{loc}$-regular, which is neither formal nor automatic (Propositions~\ref{prop:f_mod_nemyt}--\ref{prop:reg_mod_class}, Appendix~\ref{app:assum_reactions}). %
	
	We start by introducing the learning framework and discussing its well-posedness.
	\subsection{Well-posedness of learning framework}
	\label{subsec:wellposedness}
	The basis of our considerations is, for a sufficiently large domain $U$ and spaces $\mathcal{V}, H$ to be specified later, the reconstruction and uniqueness of solutions $(D^\dagger, u^\dagger, u_0^\dagger, f^\dagger)$ to
\begin{align}
\tag{$\mathcal{P}^\dagger$} \label{eq:limit_problem}
\min_{\substack{D\in [0,\infty[^{N\times L},u\in\mathcal{V}^{N\times L},\\ 
u_0\in H^{N\times L}, f\in W^{1,\infty}(U)^N}}
& \mathcal{R}_0(D,u,u_0)+\Vert f\Vert_{L^2(U)}^2+\Vert\nabla f\Vert_{L^{\infty}(U)}
\\[0.3em]
\text{s.t. } \quad
&
\left\{
\begin{alignedat}{2}
\partial_t u^l - D^l \Delta u^l - f(u^l) \;&=\, 0, \\
u^{l}(0) \;&=\, u_{0}^l, \\
K^\dagger u^l \;&=\, y^l,
\end{alignedat}
\right.
\tag{RD+M}\label{rdsystem_meas}
\end{align}	
	with $\mathcal{R}_0(D,u,u_0)=\Vert D\Vert^2+\Vert u\Vert_{\mathcal{V}}^p+\Vert u_0\Vert_H^2$, $K^\dagger$ a full measurement operator and $(y^l)_{l=1}^L$ corresponding data for $L$-many data points. Here we write notationwise $u^l = (u_n^l)_{n=1}^N$, $u_0^l = (u_{0,n}^l)_{n=1}^N$, $D^l = (d_n^l)_{n=1}^N$, $D = (D^l)_{l=1}^L$ and $\Delta u^l = (\Delta u_n^l)_{n=1}^N$. The main result of \cite{morina_holler/online} is a convergence result showing that $(D^\dagger, u^\dagger, u_0^\dagger, f^\dagger)$ is recovered in the limit $m \rightarrow \infty $ by parameterized solutions to certain all-at-once problems at level $m \in \mathbb{N}$. We introduce adjusted all-at-once problems with modified reaction terms according to Section~\ref{sec:conformal_classes} as follows. Given user-defined classes of parameterized reaction terms $\mathcal{F}^m = (\mathcal{F}_1^m,\ldots,\mathcal{F}_N^m)$  at level $m \in \mathbb{N}$ with
	\begin{align}
		\label{param_classes}
		\mathcal{F}_n^m = \left\{f_{\theta_n,n}: \mathbb{R}^N\to \mathbb{R}~|~\theta_n\in \Theta_n^m\right\}
	\end{align}
	and parameters $\theta_n \in \Theta_n^m$, we consider modifications of $f_{\theta_n,n}$, denoted by $\bar{f}_{\theta_n,n}$ for $\theta_n\in\Theta_n^m$ similar to \eqref{modification_f} and to be clarified subsequently in detail, and a modified class of parameterized reaction terms given as
	\begin{align}
		\label{eq:mod_class}
		\overline{\mathcal{F}}^m := \otimes_{n=1}^N\overline{\mathcal{F}}_n^m \quad \text{with} \quad \overline{\mathcal{F}}_n^m:=\left\{\bar{f}_{\theta_n,n}~|~\theta_n\in\Theta_n^m\right\}.
	\end{align}
	More concretely, for a sequence of transition functions $(\chi^m)_m$, we modify for $u\in \mathbb{R}^N$ the parameterized approximations $f_{\theta_n,n}\in\mathcal{F}_n^m$ by
	\begin{align}
		\label{modification:param}
		\bar{f}_{\theta_n, n}(u) = ((P_+\circ f_{\theta_n,n})(u)-f_{\theta_n,n}(u))\chi^{m}(u_n)+f_{\theta_n,n}(u)
	\end{align}
	for $1\leq n\leq N$. Furthermore, note that we write $\bar{f}_{\theta}=(\bar{f}_{\theta_n,n})_{n=1}^N$. With this, we consider the modified all-at-once problems at level $m \in \mathbb{N}$ as
	\begin{align}
		\label{eq:all_at_once_mod}
		&\min_{\substack{D\in [0,\infty[^{N\times L},u\in\mathcal{V}^{N\times L}\\ u_0\in H^{N\times L},\theta\in\otimes_n\Theta^m_n}}\mathcal{R}_0(D,u,u_0)+ \nu^m \|\theta\| + \Vert \bar{f}_{\theta}\Vert_{L^2(U)}^2+\Vert\nabla \bar{f}_{\theta}\Vert_{L^{\infty}(U)}\tag{$\overline{\mathcal{P}}^{m}$}\\
		&+\sum_{1\leq l\leq L}\bigg[\lambda^m \bigg(\Vert \partial_tu^l-D^l\Delta u^l-\bar{f}_{\theta}(u^l)\Vert_{\mathcal{W}}^q +\Vert u^{l}(0)-u_{0}^l\Vert_{H} ^2\bigg) + \mu^m\Vert K^mu^l-y^{m,l}\Vert_\mathcal{Y}^r\bigg]	\notag	
	\end{align}
	with suitable regularization parameters $(\lambda^m,\mu^m, \nu^m)_m$ and measured data $(y^{m,l})_l$ fulfilling, for reduced measurement operators $(K^m)_m$, the noise estimate
	\begin{align}
		\label{noise_estimation}
		\Vert y^{m,l}-K^m u^{\dagger, l}\Vert \leq\delta(m)
	\end{align}
	for some zero sequence $(\delta(m))_m$. In \eqref{eq:all_at_once_mod} and where used hereinafter the norm $\Vert\theta\Vert$ denotes the Euclidean norm on the respective (finite-dimensional) parameter space.
	
	Note that at this point the choice of the sequence of transition functions $(\chi^m)_m$ introduced above is general. Later we will provide a concrete guideline to choose the transition functions in dependence of the user-defined classes $(\mathcal{F}^m)_m$. Furthermore, note that additional regularization is possible in \eqref{eq:limit_problem} and \eqref{eq:all_at_once_mod}, but not necessary for the convergence result to hold true in Subsection~\ref{subsec:consistency}. Another important observation is that, in the problems above, the reconstructed diffusion coefficients may attain the value zero. The particular choice of domain is essential to guarantee its closedness. In practical applications, however, one is interested in strictly positive diffusion coefficients. This can be enforced by constraining the diffusion coefficients to lie above a prescribed small positive threshold in the above problems. While we avoid introducing this modification to maintain readability, the validity of our results is unaffected.
	
	The first necessary step towards proving that the unique solution of \eqref{eq:limit_problem} can be approximated by solutions of \eqref{eq:all_at_once_mod} (which is the focus of Subsection~\ref{subsec:consistency}), is to prove well-posedness of the limit problem \eqref{eq:limit_problem} and the learning problem \eqref{eq:all_at_once_mod}. To achieve this, a series of assumptions is required that we will summarize next. Here, for fixed time horizon $T>0$ and spatial domain $\Omega\subset\mathbb{R}^d$ we denote by $\mathcal{V}$ the extended state space of states $u_n^l:~]0,T[\to V$ with the static state space $V$ of functions $v:\Omega\to \mathbb{R}$. The space $H$ denotes the static initial trace space, $\mathcal{W}$ the dynamic extension of the image space $W$ and $\mathcal{Y}$ of the observation space $Y$. 
	\begin{remark}[Framework]
		We emphasize that the focus of this work is not to present the most general space setup possible (which is done in \cite{morina_holler/online}) but to consider modifications of parameterized functions to achieve the goals formulated in the introduction. For that reason, we fix a possibly restrictive space setup satisfying \cite[Assumption 2]{morina_holler/online}.%
	\end{remark}
	The assumptions on the space setup are given as follows:
	\begin{assumption}[Space setup]
		\label{ass_init_set}
		Suppose that $\Omega \subset \mathbb{R}^d $ with $d \in \mathbb{N}$ is a bounded Lipschitz domain with smooth boundary lying on one side of its boundary and set
		\[
		V=\tilde{V}=H=W^{\hat{m},\hat{p}}(\Omega)\quad \text{with} ~  ~ \hat{m}\hat{p}>d  ~  ~ \text{and}  ~  ~ \hat{m}\geq 2.
		\]
		We further set $W = L^{\hat{q}}(\Omega)$ with $1<\hat{q}\leq \hat{p}<\infty$ and $Y$ a separable, reflexive Banach space such that $V\hookrightarrow Y$. Moreover, let for $1<p,q,r<\infty$ with $q\leq p$ the extended spaces be defined as (Sobolev-)Bochner spaces (see \cite[Chapter 7]{Roubíček2013}) by
		\begin{gather*}
			\mathcal{V} = L^p(0,T;V)\cap W^{1,p,p}(0,T;\tilde{V}), \quad\mathcal{W}=L^q(0,T;W), \quad \mathcal{Y} = L^r(0,T;Y).
		\end{gather*}
		Assume that the coefficients $\hat{m}, \hat{p}$ fulfill for some $\tilde{m}\in \mathbb{N}_0$, $1\leq \tilde{p}\leq \infty$ the conditions:
		\begin{enumerate}[label=\roman*)]
			\item $\tilde{m}\leq \hat{m}$ and $\hat{m}-d/\hat{p}\geq \tilde{m}-d/\tilde{p}$ (additionally with $\hat{p}\leq \tilde{p}$ if equality holds)
			\item $1\leq\hat{p}\leq \frac{d\tilde{p}}{d-\tilde{m}\tilde{p}}$ if $\tilde{m}\tilde{p}< d$ and $1\leq \hat{p}<\infty$ if $\tilde{m}\tilde{p}= d$
			\item $1\leq\tilde{p}\leq \frac{d\hat{p}}{d-(\hat{m}-1)\hat{p}}$ if $(\hat{m}-1)\hat{p}< d$ and $1\leq \tilde{p}<\infty$ if $(\hat{m}-1)\hat{p}= d$
			\item $\tilde{m}\geq 2$ and $\tilde{p}/2\leq \hat{q}\leq \tilde{p}$ with $\frac{1}{\hat{q}}>\frac{2}{\tilde{p}}-\frac{1}{d}$
		\end{enumerate}
		We deduce under the condition in i) that $V\hookrightarrow W^{\tilde{m},\tilde{p}}(\Omega)$ by \cite[Theorem 4.12]{Adams2003}. The compact embeddings $W^{\tilde{m},\tilde{p}}(\Omega)\hookdoubleheadrightarrow L^{\hat{p}}(\Omega)$ and $V\hookdoubleheadrightarrow W^{1,\tilde{p}}(\Omega)$ follow due to \cite[Theorem 6.3]{Adams2003} by ii) and iii), respectively, which implies in particular that $\tilde{V}\hookrightarrow W^{1,\tilde{p}}(\Omega)$. The condition in iv) is required later for regularity properties of \eqref{modification:param}.
	\end{assumption}
	A possible choice of space parameters fulfilling the conditions of Assumption~\ref{ass_init_set} is $d=3, \hat{p}=\tilde{p}=\hat{m}=\tilde{m}=2$ and $1\leq \hat{q}<3/2$. Note that since $V\hookrightarrow\tilde{V}$ it holds true that $\mathcal{V}\hookrightarrow\mathcal{C}(0,T;\tilde{V})$ by \cite[Lemma 7.1]{Roubíček2013}, which together with the choice $\tilde V = W^{\hat{m}, \hat{p}}(\Omega)$ with $\hat{m}\hat{p}>d$ implies for some constant $c_{\mathcal{V}}>0$ the uniform state space embedding 
	\begin{align}
		\label{uniform_state_embedding}
		\Vert v\Vert_{L^\infty(]0,T[\times\Omega)}\leq c_{\mathcal{V}}\Vert v\Vert_{\mathcal{V}} ~ ~ \text{for all} ~ v\in \mathcal{V}.
	\end{align}
	Next we set the framework for the measurements.\vspace*{-0.2cm}
	\begin{assumption}[Measurements]
		\label{ass:measurements}
		Assume that for $m\in \mathbb{N}$ the operator $K^m:\mathcal{V}^{N}\to\mathcal{Y}$ is weak-weak continuous and that $K^\dagger:\mathcal{V}^N\to \mathcal{Y}$ is weak-strong continuous and injective. Suppose that for any weakly convergent sequence $(u^m)_m\subset \mathcal{V}^N$\vspace*{-0.1cm}
		\begin{equation}
			\label{operator_conv_weak}
			(K^m-K^\dagger) u^m\to 0 \quad \text{in} ~ \mathcal{Y}\quad \text{as} ~ m\to \infty.\vspace*{-0.1cm}
		\end{equation}
	\end{assumption}
	The convergence notion in \eqref{operator_conv_weak} holds, e.g., for linear operators $(K^m)_m$ converging to $K^\dagger$ in the operator norm and for nonlinear operators $(K^m)_m$ converging uniformly to $K^\dagger$ on bounded subsets of $\mathcal{V}$. Nonetheless, the general condition in \eqref{operator_conv_weak} remains important as specific examples, such as certain sampling operators based on truncated Fourier measurements, satisfy \eqref{operator_conv_weak}, while the stricter conditions above fail to hold.

	The reaction terms $f_n$ in \eqref{eq:limit_problem} and \eqref{rdsystem_meas} are understood as Nemytskii operators of $f_n: V^N\to W$ which is an extension of $f_n: \mathbb{R}^N\to \mathbb{R}$. In view of the classes $\mathcal{F}_n^m$ of parameterized reaction terms \eqref{param_classes} we pose the following regularity assumptions.
	\begin{assumption}[Parameterized reaction terms]
		\label{ass:param_reaction}
		Let the parameter sets $\Theta_n^m$ for $1\leq n\leq N$ and $m\in \mathbb{N}$ be closed, and each contained in a finite dimensional space. Suppose that $f_{\theta_n,n}\in \mathcal{F}_n^m$ induces a well-defined Nemytskii operator $f_{\theta_n,n}:\mathcal{V}^N\to \mathcal{W}$ with $[f_{\theta_n,n}(v)](t)(x)=f_{\theta_n,n}(v(t,x))$. Assume further continuity of the map
		\begin{align}
			\label{assump_w_s_cont}
			\Theta_n^m\times (L^p(0,T;L^{\hat{p}}(\Omega)))^N\ni (\theta_n,v)\mapsto f_{\theta_n,n}(v)\in L^q(0,T;L^{\hat{q}}(\Omega)).
		\end{align}
		Moreover, suppose that $\mathcal{F}_n^m\subseteq W_{loc}^{1,\infty}(\mathbb{R}^{N})$ for $1\leq n\leq N$ and $m\in\mathbb{N}$. 
	\end{assumption}
	The requirements in Assumption~\ref{ass:param_reaction} are e.g.\ fulfilled in case the classes $\mathcal{F}_n^m$ are chosen as feed-forward neural networks with Lipschitz continuous activation function (see \cite[Propositions 19 and 20]{morina_holler/online}).
	
	While the setup so far is adapted from \cite{morina_holler/online}, the following results are new and form the main technical contribution of this section. We show that the properties of Assumption~\ref{ass:param_reaction} are inherited by the physically consistent classes $\overline{\mathcal{F}}_n^m$ in \eqref{eq:mod_class}, with $\bar{f}_{\theta_n, n}$ defined as in \eqref{modification:param}. This does not follow directly from the corresponding properties of $\mathcal{F}_n^m$, since the modification involves the positive part $P_+$ and the multiplication by $\chi^m(u_n)$, which raises the regularity requirements on the state space (see condition iv) in Assumption~\ref{ass_init_set}). The proofs are given in Appendix~\ref{app:assum_reactions}. We start with the extension result for $\bar{f}_{\theta_n,n}$:
	\begin{proposition}
		\label{prop:f_mod_nemyt}
		Under Assumption~\ref{ass_init_set} and \ref{ass:param_reaction}, let the elements of $\mathcal{F}_n^m\hspace*{-0.1cm}$ be Lipschitz continuous. Then $\bar{f}_{\theta_n, n}$ defines a well-defined Nemytskii operator $\bar{f}_{\theta_n, n}:\mathcal{V}^N\to \mathcal{W}$.
	\end{proposition}
	\begin{proof}
		See Appendix~\ref{appendix:extension}.
	\end{proof}
	
	Similarly we formulate weak-strong continuity of $\bar{f}_{\theta_n,n}$:
	\begin{proposition}
		\label{prop:ws_continuity_mod}
		Let Assumptions~\ref{ass_init_set} and \ref{ass:param_reaction} hold, and let the elements $f_{\theta_n,n}$ of $\mathcal{F}_n^m$ be Lipschitz continuous. Then the following map is weakly-strongly continuous: \[
		\Theta_n^m\times \mathcal{V}^N\ni (\theta_n,v)\mapsto \bar{f}_{\theta_n,n}(v)\in \mathcal{W}.
		\] 
	\end{proposition}
	\begin{proof}
		See Appendix~\ref{appendix:continuity}.
	\end{proof}
	\noindent Finally, we formulate the regularity property for $\bar{f}_{\theta_n,n}$:\vspace*{-0.1cm}
	\begin{proposition}
		\label{prop:reg_mod_class}
		Let $\mathcal{F}_n^m\subseteq W^{1, \infty}_{loc}(\mathbb{R}^{N})$. Then $\overline{\mathcal{F}}_n^m\subseteq W^{1, \infty}_{loc}(\mathbb{R}^{N})$.
	\end{proposition}
	\begin{proof}
		See Appendix~\ref{appendix:regularity}.
	\end{proof}
	Another key requirement in \cite{morina_holler/online} is the existence of an admissible solution.
	\begin{assumption}[Admissible solution]
		\label{ass:admissible}
		Suppose that the full measurement data $y\in \mathcal{Y}^L$ is such that there exist admissible functions $\hat{f}\in W^{1,\infty}(\mathbb{R}^{N})^N$, $\hat u \in \mathcal{V}^{N \times L}$, $\hat D \in [0,\infty[^{N\times L}$ and $\hat u_0\in H^{N \times L}$ solving \eqref{rdsystem_meas}. Setting $\epsilon=c_{\mathcal{V}}\hat{C}^{1/p}$ and $\hat{C}\geq \Vert \hat D\Vert^2+\Vert \hat u\Vert_{\mathcal{V}}^p+\Vert \hat u_0\Vert_H^2+\Vert \hat{f}\Vert_{L^2(\mathbb{R}^{N})}^2+\Vert \nabla\hat{f}\Vert_{L^\infty(\mathbb{R}^{N})}+1$, let $U$ be a bounded Lipschitz domain large enough to contain $\{z \in \mathbb {R}^{N} \,:\, \|z\|\leq \epsilon\}$.
	\end{assumption}
	The last assumption requires the given measurement data $y$ to be feasible in the sense that there exist at least some $(\hat{D},\hat u, \hat u_0,\hat{f})$ satisfying the constraint \eqref{rdsystem_meas}. %
	
	On basis of the previously introduced assumptions and the results in Propositions~\ref{prop:f_mod_nemyt}-\ref{prop:reg_mod_class} we can prove well-posedness of \eqref{eq:limit_problem} and \eqref{eq:all_at_once_mod} as claimed:
	\begin{theorem}
		\label{thm:wellposedness}
		Under Assumptions~\ref{ass_init_set}-\ref{ass:param_reaction} and \ref{ass:admissible}, problem \vspace*{-0.1cm}
\begin{itemize}
\item \eqref{eq:limit_problem} admits a unique solution $(D^\dagger, u^\dagger, u_0^\dagger, f^\dagger)$.\vspace*{-0.1cm}
\end{itemize}		
		 Let further the instances of $\mathcal{F}_n^m$ in Assumption~\ref{ass:param_reaction} be Lipschitz continuous.  Then\vspace*{-0.1cm}
\begin{itemize}
\item \eqref{eq:all_at_once_mod} admits a solution $(D^{m}, u^{m}, u_0^{m}, \theta^{m})$ for $m\in\mathbb{N}$.\vspace*{-0.1cm}
\end{itemize}  Furthermore, for $m\in\mathbb{N}$, the functions $\bar{f}_{\theta^{m}}$ fulfill the conditions \eqref{cond:lipschitz}, \eqref{cond:quasipos}, \eqref{cond:mass} and \eqref{cond:growth} such that, whenever $u_0^{m}\in L^1(\Omega)\cap L^\infty(\Omega)$ is non-negative and $D^m\in]0,\infty[^{N\times L}$, \vspace*{-0.1cm}
\begin{itemize}
\item the system in \eqref{rdsystem} with $u_0 = u_0^{m,l}$, $D = D^{m,l}$, $\bar{f} = \bar{f}_{\theta^{m}}$ admits a unique global classical solution $u$ attaining the regularity in \eqref{u_regularity} for $1\leq l\leq L$. \vspace*{-0.1cm}
\end{itemize}  
If in addition $\bar{f}_{\theta^m}$ is mass dissipating, then for any boundary data $g$ admitting an extension $G$ as in Corollary~\ref{cor:wellposedness2}\vspace*{-0.1cm}
\begin{itemize}
	\item the system in \eqref{rdsystem2} with $u_0 = u_0^{m,l}$, $D = D^{m,l}$, $\bar{f} = \bar{f}_{\theta^{m}}$ admits a global weak solution $u$ attaining the regularity in \eqref{u_regularity2} for $1\leq l\leq L$. \vspace*{-0.1cm}
\end{itemize} 
	\end{theorem}
	\begin{proof}
		Note that the physical term $\mathcal{V}\times[0,\infty[~\ni (u,d)\mapsto d\Delta u$ induces a well-defined Nemytskii operator and is weak-weak continuous due to \cite[Proposition 24]{morina_holler/online} with $s_\beta=\infty$ (where $\hat{p}=\hat{q}$). Thus, well-posedness of the problems \eqref{eq:limit_problem} and \eqref{eq:all_at_once_mod} follows by \cite[Appendix C]{morina_holler/online} once we verify Assumption~\ref{ass:param_reaction} for the classes of modified parameterized reaction terms $\bar{f}_\theta$, introduced as $\overline{\mathcal{F}}_n^m$ above. Extendability to a well-defined Nemytskii operator follows by Proposition~\ref{prop:f_mod_nemyt}. Weak-strong continuity in Assumption~\ref{ass:param_reaction} is a consequence of Proposition~\ref{prop:ws_continuity_mod}. The $W^{1,\infty}_{loc}$-regularity of the modified class of parameterizations is proven in Proposition~\ref{prop:reg_mod_class}. The remaining assertions follow directly from Lemma~\ref{lem:auxiliary_result} together with Corollaries~\ref{cor:wellposedness} and \ref{cor:wellposedness2}.
	\end{proof}
	\subsection{Convergence of learned models}
	\label{subsec:consistency}
	Building on Subsection~\ref{subsec:wellposedness}, we now show that the unique solution of problem \eqref{eq:limit_problem} can be approximated by solutions of \eqref{eq:all_at_once_mod}. The convergence mechanism is the one of \cite[Theorem 31]{morina_holler/online}. The work of this subsection lies in the hypotheses under which it becomes available for the physically consistent classes. For that, we impose an approximation capacity condition on the original classes $\mathcal{F}_n^m$, in a form weaker than in \cite{morina_holler/online}.%
	 \begin{assumption}[Approximation capacity condition]
	 	\label{ass:approx_cap_cond}
	 	Based on Assumption~\ref{ass:param_reaction} we pose the following approximation capacity condition for qualified $f\in W_{loc}^{1,\infty}(\mathbb{R}^{N})^N$ and $U$ as in Assumption~\ref{ass:admissible}: There exist $c,\beta >0$ and a monotonically increasing $\psi:\mathbb{N}\to \mathbb{R}$ such that there exist $\theta^m\in \Theta^m$ with $\Vert \theta^m\Vert \leq \psi(m)$, $\Vert f-f_{\theta^m}\Vert_{L^\infty(U)}\leq cm^{-\beta}$ for $m\in\mathbb{N}$ and
	 	\begin{align}
	 		\label{properties:approx}
	 	 \limsup_{m\to\infty}\Vert \nabla f_{\theta^m}\Vert_{L^\infty(U)}\leq \Vert \nabla f\Vert_{L^\infty(U)}.
	 	\end{align}
	 \end{assumption}
	 \begin{remark}
	 	\label{rem:uppersemi}
	 	The formulation of the approximation capacity condition in \cite[Assumption 5, iv)]{morina_holler/online} is the same as the one here in Assumption~\ref{ass:approx_cap_cond} with the difference that instead of \eqref{properties:approx} it imposes the stronger requirement $\lim_{m\to\infty}\Vert \nabla f_{\theta^m}\Vert_{L^\infty(U)}= \Vert \nabla f\Vert_{L^\infty(U)}$. Thus, in addition to \eqref{properties:approx} it requires $\liminf_{m\to\infty}\Vert \nabla f_{\theta^m}\Vert_{L^\infty(U)}\geq \Vert \nabla f\Vert_{L^\infty(U)}$. This additional requirement is not necessary, as can be seen by inspecting the proof of the main result in \cite[Theorem 31]{morina_holler/online}. There, convergence of the supremum norm of the gradient is used only to show that the right-hand side of \cite[(3.5)]{morina_holler/online}, an upper bound for the objective functional, converges to the objective functional of the limit problem. For the remainder of the proof it suffices that its limit superior is bounded by the latter, for which the condition in \eqref{properties:approx} is sufficient.
	 \end{remark}
	 The approximation capacity condition in Assumption~\ref{ass:approx_cap_cond} imposes two requirements that go beyond standard universal approximation, namely simultaneous approximation of $f$ in $L^\infty(U)$ \emph{and} its Lipschitz constant $\Vert \nabla f\Vert_{L^\infty(U)}$ as in \eqref{properties:approx}, and a controlled growth of the parameters $\Vert \theta^m\Vert \leq \psi(m)$. Concrete neural network architectures that satisfy both requirements under suitable smoothness assumptions on $f$ are the following. The first is the family of sparsely connected feed-forward networks with ReQU activation $\sigma(x)=\max(x,0)^2$ considered in \cite{Belomestny2023}, whose width scales with the parameter $m$ and whose weights are bounded uniformly in $[-1,1]$. For $f\in C^q(U)$ with $q\geq 3$ these networks approximate $f$ together with its first order derivatives, so that Assumption~\ref{ass:approx_cap_cond} holds with $\beta=q/N$ and constant $\psi$ \cite[Proposition 22]{morina_holler/online}. The second is the class of ReQU networks in \cite[Proposition 4.8]{raslan21}, which approximate in the $W^{1,\infty}$-norm with bounded weights and satisfy Assumption~\ref{ass:approx_cap_cond} already for $f\in W^{n,\infty}(U)$ with $n\geq 2$, with $\beta=(n-1)/N$ and polynomially growing $\psi$ \cite[Proposition 23]{morina_holler/online}. Note that ReQU networks are only locally Lipschitz continuous, and that beyond a single hidden layer they exceed the quadratic growth in \eqref{cond:growth}. While this is not important for Assumption~\ref{ass:approx_cap_cond}, combining them with Lemma~\ref{lem:auxiliary_result}, and Theorems~\ref{thm:wellposedness} and \ref{thm:main_result} requires the locally Lipschitz variant of Remark~\ref{rem:local_lipschitz}. For the deep ReLU construction based on the local Taylor-type approximation scheme of \cite{lu_main}, modified in \cite{Morina2024}, the available result provides the uniform estimate only. Results that also control the gradient, such as \cite{guehring20}, are available, but not together with a bound on the parameters. Alternatively, for $f\in W^{2,\infty}(U)$ or $f\in C^2(U)$ and star-shaped $U$, a lifting technique \cite[Appendix A.3]{morina_holler/online} can be applied to obtain Assumption~\ref{ass:approx_cap_cond} from the uniform estimate. We refer to \cite[Section 2]{Morina2024} for an overview of parameter bounds for feed-forward networks.

    The final necessary condition for ensuring reconstructibility of the solution to the limit problem involves strict quasipositivity at a sufficiently large rate (see Assumption~\ref{ass:strict_quasipositivity}). This condition is linked to the specific choice of transition functions $(\chi^m)_m$, as established in Proposition~\ref{cor:approximation_inherit}, which are required to define the physically consistent classes $\overline{\mathcal{F}}_n^m$ in \eqref{eq:mod_class}. We now state our main convergence result: 
	\begin{theorem}
		\label{thm:main_result}
		Let Assumptions~\ref{ass_init_set}-\ref{ass:param_reaction} and \ref{ass:admissible} apply with $f^\dagger$ being qualified for the approximation capacity condition in Assumption~\ref{ass:approx_cap_cond} with rate $\beta>0$ where $(D^\dagger, u^\dagger, u_0^\dagger, f^\dagger)$ is the unique solution of \eqref{eq:limit_problem}. Suppose that $f^\dagger$ fulfills Assumption~\ref{ass:strict_quasipositivity} with $U$ as in Assumption~\ref{ass:admissible}. Let further the instances of $\mathcal{F}_n^m$ in Assumption~\ref{ass:param_reaction} be Lipschitz continuous, and the transition functions $(\chi^m)_m$ in the physically consistent classes $\overline{\mathcal{F}}_n^m$ in \eqref{eq:mod_class} be given by $\chi^m = \tilde{h}_{\epsilon_m}$ for $m\in\mathbb{N}$ and $(\epsilon_m)_m=(m^{-\gamma})_m$ for fixed $0<\gamma<\beta$. Then, with a parameter choice $\lambda^m,\mu^m,\nu^m>0$ such that\vspace*{-0.2cm}
\[\lambda^m\to\infty, \quad\mu^m\to \infty,\quad \nu^m\to 0\quad \text{and}\]\vspace*{-0.8cm}
\[\lambda^mm^{-\min(\alpha\gamma,\beta) q}=o(1),\quad\, \mu^m\delta(m)^r=o(1),\quad \nu^m\psi(m)=o(1)\]
as $m\to \infty$, and for $(D^m, u^m, u_0^m, \theta^m)$ a solution to \eqref{eq:all_at_once_mod}, it holds true that $D^m\to D^\dagger$ in $[0,\infty[^{N\times L}$, $u^m\rightharpoonup u^\dagger$ in $\mathcal{V}^{N\times L}$, $u_0^m\rightharpoonup u_0^\dagger$ in $H^{N\times L}$ and $\bar{f}_{\theta^m}\to f^\dagger$ in $\mathcal{C}(U)^N$.
	\end{theorem}
	\begin{proof}
		Well-posedness of \eqref{eq:limit_problem} and \eqref{eq:all_at_once_mod} follows by Theorem~\ref{thm:wellposedness}. The recovery of the unique solution to \eqref{eq:limit_problem} by solutions of \eqref{eq:all_at_once_mod} follows by \cite[Theorem 31]{morina_holler/online} if we can verify Assumption~\ref{ass:param_reaction} and Assumption~\ref{ass:approx_cap_cond} for the classes of modified parameterized reaction terms $\bar{f}_\theta$, introduced as $\overline{\mathcal{F}}_n^m$ above. Extendability to a well-defined Nemytskii operator follows by Proposition~\ref{prop:f_mod_nemyt}. Weak-strong continuity in Assumption~\ref{ass:param_reaction} follows by Proposition~\ref{prop:ws_continuity_mod} and $W^{1,\infty}_{loc}$-regularity of the modified class of parameterizations is proven in Proposition~\ref{prop:reg_mod_class}. Assumption~\ref{ass:approx_cap_cond} is a consequence of the arguments in Proposition~\ref{cor:approximation_inherit} (with $\gamma$ instead of $\beta/\alpha$).
	\end{proof}
	\begin{remark}
	Using suitable classes $(\mathcal{F}^m)_m$, for example certain neural network architectures discussed in \cite[Propositions 22 and 23]{morina_holler/online}, Assumption~\ref{ass:approx_cap_cond} is satisfied for sufficiently regular functions $f$. Since $f^\dagger$ solves \eqref{eq:limit_problem}, it is Lipschitz continuous, and one can potentially expect even higher regularity as it is a reaction model for the RD system \eqref{rdsystem_meas}. Consequently, $f^\dagger$ can reasonably be expected to satisfy Assumption~\ref{ass:approx_cap_cond}. RD systems arising in practical applications, such as chemistry, naturally preserve non-negativity of the state, so reaction terms $f$ solving \eqref{rdsystem_meas} are expected to be quasipositive. However, Assumption~\ref{ass:strict_quasipositivity} is stronger than mere quasipositivity. It is imposed because deriving \eqref{properties:approx} for the physically consistent classes $(\overline{\mathcal{F}}^m)_m$ requires sufficient decay as formulated in \eqref{eq:strict_quasi} (see Proposition~\ref{cor:approximation_inherit}).
	
	The freedom in choosing $0<\gamma<\beta$ in Theorem~\ref{thm:main_result} reflects the following trade-off. The parameter-selection condition $\lambda^m m^{-\min(\alpha\gamma,\beta)q}=o(1)$ depends on $\gamma$ only through $\min(\alpha\gamma,\beta)$, so every $\gamma\in[\beta/\alpha,\beta[$ delivers rate $\beta$, with $\gamma=\beta/\alpha$ as picked in Proposition~\ref{cor:approximation_inherit}, being the smallest such value. Choosing $\gamma<\beta/\alpha$ yields a smaller rate of growth for $(\lambda^m)_m$, but a wider transition region $\epsilon_m=m^{-\gamma}$ and hence a smaller $|(\chi^m)'|_\infty\sim\epsilon_m^{-1}$, which can be preferable for the numerical conditioning.
	\end{remark}
	
	\section{Numerical experiments}
	\label{sec:numerics}
	
	We complement the theory of Section~\ref{sec:model_learning} with numerical experiments aiming at:
	\begin{enumerate}[label=\roman*),leftmargin=0.6cm] 
		\item Verifying that modification~\eqref{modification:param} and problem \eqref{eq:all_at_once_mod} are practically realizable and recover the regularization-minimizing reaction $f^\dagger$ on representative RD systems.
		\item Comparing the proposed parameterization against two baselines that do not enforce quasipositivity \eqref{cond:quasipos} structurally:
		
		\begin{enumerate}[label=\alph*),leftmargin=0.6cm] 
		\item the corresponding \emph{unconstrained} model, in which the modification is disabled,
		
		\item a \emph{penalty-method} model, in which the modification is likewise disabled but the violation of \eqref{cond:quasipos} is penalized in the loss.
			\end{enumerate}
		
		\noindent This tests whether enforcing \eqref{cond:quasipos} through the class, rather than through the objective, is necessary in practice.
		\item Studying the sensitivity of the approach to the transition function $\chi^m$.
	\end{enumerate}
	We restrict ourselves to \eqref{cond:quasipos} in ii) since a globally Lipschitz continuous $f_\theta$, such as the ReLU networks used here, satisfies \eqref{cond:lipschitz}, \eqref{cond:growth} and \eqref{cond:mass} in the control regime (see the discussion of regimes following condition \eqref{cond:mass}) already without the modification, by the estimates in the proof of Lemma~\ref{lem:auxiliary_result}. In addition to i)-iii) above we present extended experiments in Appendix~\ref{app:numerics}. These include the consideration of non-trivial observation operators and different regularization schedules. All experiments were run on a single laptop in a \texttt{PyTorch} implementation, and can be reproduced with the source code \cite{code_rdlearning}.
	
	\subsection{Setup}
	\label{subsec:num_setup}
	In the following we provide the details of the implementation framework.\\[-2pt]
	
	\noindent\textbf{Benchmarks.} To evaluate our approach, we use five reaction-diffusion systems, presented in Table~\ref{tab:benchmarks}, which are of the form
	\begin{equation}
		\label{rd:generic}
		\partial_t u_n = d_n^\dagger \Delta u_n + f_n^\dagger(u), \quad t\in\left]0,T\right[, ~ x\in \Omega:=[0,L_\Omega], \quad
		n=1,\ldots,N
	\end{equation}
	for fixed $T, L_\Omega>0$, diffusion coefficients $(d^\dagger_n)_n$ and reaction term $(f^\dagger_n)_n$. Note that the reaction term $f^\dagger$ generating the data does not need to coincide with the regularization-minimizing solution of \eqref{eq:limit_problem}. Since the diffusion coefficients are learned as well, the data determine only the total dynamics along the visited states, and the reaction term is identified there if the diffusion coefficients are identified (see \cite[Section 3]{morina_holler/online} for a discussion). Away from the visited states, the learned reaction term is determined by the regularization. The considered benchmarks range from a scalar front ($N=1$) to mass-conservative chemistry ($N=3$). For reference, the typical reaction-rate magnitudes on the data are $\|f^\dagger\|_{L^\infty(\mathrm{data})}\approx 0.25$ (B1), $1.30$ (B2), $1.65$ (B3), $0.42$ (B4), and $\lesssim 10^{-5}$ (B5). The quasipositivity violation $\mathcal{Q}$ reported in Tables~\ref{tab:comparison} and \ref{tab:transition} is in the same units and is to be read against these scales.\\
	
	\begin{table}[t]
		\centering
		\small
		\setlength{\extrarowheight}{2pt}
		\setlength{\tabcolsep}{4pt}
		\begin{tabular}{@{}l c l l @{}}
			\toprule
			benchmark & $N$ & $f^\dagger(u)$ & parameters\\
			\midrule
			B1 Fisher-KPP      & $1$ &
			$r\,u(1-u)$ &
			$\begin{array}{@{}l@{}} r=1 \\ d^\dagger = 10^{-2} \\ L_\Omega=50,\ T=20 \end{array}$\\[4pt]
			\hdashline
			B2 Lotka-Volterra  & $2$ &
			$\begin{array}{@{}l@{}} a u_1 - b u_1 u_2 \\ -c u_2 + e u_1 u_2 \end{array}$ &
			$\begin{array}{@{}l@{}} a = b = c = e = 1 \\ d^\dagger = (10^{-2},10^{-2}) \\ L_\Omega=20,\ T=10 \end{array}$\\[4pt]
			\hdashline
			B3 Schnakenberg     & $2$ &
			$\begin{array}{@{}l@{}} a - u_1 + u_1^2 u_2 \\ b - u_1^2 u_2 \end{array}$ &
			$\begin{array}{@{}l@{}} a = 0.1,\ b = 0.9 \\ d^\dagger = (5{\cdot}10^{-3},10^{-1}) \\ L_\Omega=20,\ T=30 \end{array}$\\[4pt]
			\hdashline
			B4 SIR              & $3$ &
			$\begin{array}{@{}l@{}} -\beta S I \\ \beta S I - \gamma I \\ \gamma I \end{array}$ &
			$\begin{array}{@{}l@{}} \beta = 2,\ \gamma = 0.5 \\ d^\dagger = (5{\cdot}10^{-3},10^{-2},5{\cdot}10^{-3}) \\ L_\Omega=20,\ T=10 \end{array}$\\[4pt]
			\hdashline
			B5 mass-action      & $3$ &
			$\begin{array}{@{}l@{}} -k_1 A B + k_2 C \\ -k_1 A B + k_2 C \\ \phantom{-}k_1 A B - k_2 C \end{array}$ &
			$\begin{array}{@{}l@{}} k_1 = 1,\ k_2 = 0.5 \\ d^\dagger = (10^{-2},10^{-2},10^{-2}) \\ L_\Omega=20,\ T=15 \end{array}$\\
			\bottomrule
		\end{tabular}
		\caption{Benchmark systems. Each row lists the ground-truth reaction term $f^\dagger$ (one species per line), the numerical parameters of the reference simulation (rate constants, diffusion coefficients $d^\dagger$, spatial extent $L_\Omega$ of $\Omega=[0,L_\Omega]$ and time horizon $T$). B1-B3 satisfy \eqref{cond:mass} in the control regime, B4 and B5 with equality.}
		\label{tab:benchmarks}
	\end{table}
	\noindent\textbf{Reference data.} For each benchmark we integrate the ground-truth system in \eqref{rd:generic} on $\Omega$ with homogeneous Neumann boundary conditions, using a second-order central-difference Laplacian and the implicit BDF integrator of \texttt{SciPy}. We generate $L=4$ trajectories from benchmark-specific perturbed initial conditions. Concretely, we use smoothed Heavisides of varying support for the propagating front B1, and small Gaussian perturbations of a reference initial condition for the other benchmarks B2-B5, outlined in Table~\ref{tab:benchmarks}. We sample them at $81$ discrete time points on a uniform spatial grid of $129$ cells, and consider measurements of the form $y^{m,l}=K^m u^{\dagger,l}+\xi^l$ for normally distributed $\xi^l$ with zero mean and standard deviation $\sigma=10^{-3}$. The measurement operator $K^m$ is chosen as the identity operator throughout Section~\ref{sec:numerics}, while different, non-trivial choices of $K^m$ are treated in Appendix~\ref{app:numerics}.\\
	
	\noindent\textbf{Parameterization and modification.} For the parameterized class $\mathcal{F}_n^m$ in~\eqref{param_classes} we use feed-forward neural networks $f_{\theta_n,n}\colon\mathbb{R}^N\to\mathbb{R}$ of width $W=64$ and depth $3$ with ReLU activation. The modification~\eqref{modification:param} is realized componentwise as\vspace*{-0.1cm}
	\begin{equation}
		\label{eq:mod_disc}
		\bar f_{\theta_n,n}(u)
		= \big(P_+(f_{\theta_n,n}(u)) - f_{\theta_n,n}(u)\big)\,\chi^{m}(u_n)
		+ f_{\theta_n,n}(u),\vspace*{-0.1cm}
	\end{equation}
	with $\chi^m$ the $C^2$-quintic smoothstep (i.e., degree-$5$ polynomial transition) of half-width $\delta^m=\varepsilon^m/2$ and cut-off level $\varepsilon^m=0.05$. By Remark~\ref{rem:chi_family} this choice satisfies the hypotheses of Theorem~\ref{thm:main_result}. The influence of the choice of $\chi^m$ is the subject of Subsection~\ref{subsec:num_transition}.\\
	
	\noindent\textbf{Discretized problem and optimization.} We realize \eqref{eq:all_at_once_mod} at the discrete level by replacing the differential operators by finite differences on the grid of the reference data and the norms over $U$ by sample-based approximations, and consider, with additionally introduced regularization parameters $\eta_0,\eta_f,\eta_g$, independent of $m$ and serving practical reasons, the loss
	\begin{align}
		\label{eq:loss_disc}
		\mathcal{L}^m(D,u,u_0,\theta)
		={}& \eta_0\, \mathcal{R}_0^{\mathrm{disc}}(D,u,u_0)
		+ \nu^m \|\theta\|
		+ \eta_f\, \|\bar f_\theta\|_{L^2(U)}^2
		+ \eta_g\, \|\nabla \bar f_\theta\|_{L^\infty(U)}
		\notag\\
		&+ \lambda^m \sum_{l=1}^{L} \Big[
		\|\partial_t u^l - D^l\,\Delta_x u^l - \bar f_\theta(u^l)\|_{\ell^2}^2
		+ \|u^l(0)-u_0^l\|_{\ell^2}^2 \Big]
		\notag\\
		&+ \mu^m \sum_{l=1}^{L} \|K^m u^l - y^{m,l}\|_{\ell^2}^2,
	\end{align}
	where $\mathcal{R}_0^{\mathrm{disc}}(D,u,u_0) := \|D\|_{\ell^2}^2 + \|u\|_{\ell^2}^2 + \|u_0\|_{\ell^2}^2$ is the discrete analog of the $\mathcal{R}_0$ regularizer in~\eqref{eq:all_at_once_mod}, the domain $U$ of Subsection~\ref{subsec:wellposedness} is realized as the bounding box $U=\prod_{n=1}^N[\min_{l,t,x} u_n^{\dagger,l},\max_{l,t,x} u_n^{\dagger,l}]$ of the reference data (padded by $5\%$ of the per-species range), the diffusion coefficients are parameterized via the softplus function as $D^l=\log(1+\exp(\tilde D^l))> 0$ and unconstrained learnable parameters $\tilde D^l\in\mathbb{R}$ (so that positivity of $D^l$ is enforced without inequality constraints in the optimizer), $\partial_t$ is the second-order central-difference time operator, with second-order one-sided stencils at $t=0$ and $t=T$, $\Delta_x$ the second-order central-difference Neumann Laplacian on the spatial grid of the reference data, and the $\ell^2$ inner product is the discrete spatial (or space-time) $L^2$ inner product on the uniform grid. We minimize~\eqref{eq:loss_disc} with Adam \cite{kingma2017} in two stages. First, we apply a warm-up of $E_w=250$ epochs in which $(u,u_0)$ are held fixed and only $(D,\theta)$ are updated, which escapes the trivial minimum $f_\theta\equiv 0$, followed by $E=1500$ all-at-once epochs updating $(D,u,u_0,\theta)$ jointly. During the warm-up the state stays at the value it is initialized with, and that value is built from the measurements alone. Since $K^m=\mathrm{Id}$ throughout this section, it is the measured data itself, $u=y^m$ on the $81\times 129$ nodes of the reference data and $u_0=y^m(0)$. For a general $K^m$ we initialize instead at a least-squares fitted state, i.e., a minimizer of $u\mapsto\sum_{l=1}^L\|K^m u^l-y^{m,l}\|_{\ell^2}^2$. If $K^m$ admits a right inverse the error is zero. For non-injective $K^m$ the minimizer is not unique, and the state depends on the choice of the right inverse (see the three operators in Appendix~\ref{app:numerics}). The concrete numerical choices of $\eta_0,\eta_f,\eta_g$ and $\lambda^m,\mu^m,\nu^m$, and the optimizer settings can be found in Table~\ref{tab:recipe}.
	
	The implementation differs from the continuous problem \eqref{eq:all_at_once_mod} in two places: i) the (Bochner-)Sobolev norms of the state space $V$ (with dynamic extension $\mathcal{V}$) and the initial trace space $H$ in the continuous problem are replaced by $\ell^2$ norms on the grid for simplicity. %
	Secondly, ii) the parameterized class $\mathcal{F}_n^m$ is realized as a generic feed-forward MLP rather than the architectures of \cite{Belomestny2023,raslan21}.\\
	
	\noindent\textbf{Metrics.} We report the relative $L^2$ error of the recovered reaction term over a \emph{generalization tube} $U_{\mathrm{tube}}(\alpha)=\bigcup_{(l,t,x)}\big(u^\dagger(l,t,x) +\textstyle\prod_{n=1}^{N}[-\alpha r_n,\alpha r_n]\big)\cap\mathbb{R}_{\geq0}^N$ around the data manifold (with $r_n$ the empirical per-species range and $\alpha=0.1$) normalized by $\|f^\dagger\|_{L^2(U_{\mathrm{tube}})}$. Further, we consider the relative state error, the relative diffusion error $|D_n^m-d_n^\dagger|/d_n^\dagger$, and two indicators of physical consistency. The first, the \emph{quasipositivity violation} of the learned reaction term,\vspace*{-0.2cm}
	\begin{equation}
		\label{eq:qpos_viol_emp}
		\mathcal{Q}(\bar f_\theta) = \max_{1\leq n\leq N}\
		\max_{u\in U_{\mathrm{qp}},\,u_n=0}\,
		\max\big(0,-\bar f_{\theta,n}(u)\big),\vspace*{-0.2cm}
	\end{equation}
	is considered on the non-negative box $U_{\mathrm{qp}}=[0,u_{\max}^\dagger+0.1]^N$ with $u_{\max}^\dagger=\max_{n,l,t,x}u_n^\dagger(l,t,x)$, so that for each $n$ the inner maximum is taken over the face $\left\{u\in U_{\mathrm{qp}}: u_n=0\right\}$ on which \eqref{cond:quasipos} is imposed. In practice, we estimate this maximum by Monte-Carlo sampling, drawing for each $1\leq n\leq N$ a sample of $4096$ points uniformly from $U_{\mathrm{qp}}$ and setting their $n$-th coordinate to zero. The sample is generated from a fixed seed and is thus the same for all models compared on a given benchmark. Note that $\mathcal{Q}$ is a lower bound for the exact maximum, so that $\mathcal{Q}=0$ certifies \eqref{cond:quasipos} up to the resolution of the sample only, whereas for the constrained model it holds true by construction of \eqref{eq:mod_disc}. The second indicator instead evaluates the consequence of a possible quasipositivity violation on the induced solutions. To this end, let $v^{0}=(v^{0}_i)_{i=1}^N$ denote the reference initial profile of the benchmark on $\Omega$, of which the $L$ training initial conditions $u_0^{\dagger,l}$ are the perturbations described above. For each species index $1\leq n\leq N$ we define a \emph{stress initial condition} $u^{0,n}=(u^{0,n}_i)_{i=1}^N$ by scaling down the $n$-th species only, i.e.,\vspace*{-0.1cm}
	\begin{equation*}
		u^{0,n}_n = 0.05\, v^{0}_n \quad \text{and} \quad
		u^{0,n}_i = v^{0}_i \quad \text{for} ~ i\neq n.\vspace*{-0.1cm}
	\end{equation*}
	In this way $u^{0,n}$ places species $n$ close to the non-negativity boundary $\left\{u_n=0\right\}$, on which \eqref{cond:quasipos} acts, while the remaining species are left at their reference values. Since none of the $u_0^{\dagger,l}$ is of this form, $u^{0,n}$ checks the learned model outside the training regime. We then integrate the learned RD system\vspace*{-0.1cm}
	\begin{equation*}
		\partial_t u = D\,\Delta u + \bar f_\theta(u) \quad \text{on} ~
		\Omega\times[0,T], \qquad u(0)=u^{0,n},\vspace*{-0.1cm}
	\end{equation*}
	with $D$ the mean of the learned $(D^l)_{l=1}^L$ and with the same solver and grid as for the reference data, and denote the resulting roll-out by $u^{\mathrm{stress},n}=(u^{\mathrm{stress},n}_i)_{i=1}^N$. Its most negative value over all species $i$, all grid points $x$ and all reporting times $t$, maximized over the $N$ stress initial conditions, yields the \emph{stress-rollout negativity}\vspace*{-0.1cm}
	\begin{equation*}
		\mathcal{N}_{\mathrm{stress}}
		= \max\Big(0,\,-\min_{n,i,t,x} u^{\mathrm{stress},n}_i(t,x)\Big).\vspace*{-0.1cm}
	\end{equation*}
	It holds true that $\mathcal{N}_{\mathrm{stress}}=0$ if and only if every stress roll-out remains non-negative, and $\mathcal{N}_{\mathrm{stress}}$ equals the worst observed undershoot below zero otherwise. Being an absolute quantity, $\mathcal{N}_{\mathrm{stress}}$ is to be read against the amplitude $0.05\max_x v_n^0$ of the species that was scaled down, which lies between $0.04$ and $0.07$ on B1-B5, except for $I$ in B4 with $0.005$. For the two species with $v_n^0\equiv 0$, namely $R$ in B4 and $C$ in B5, the scaling has no effect and $u^{0,n}$ coincides with $v^0$. Note that $\mathcal{Q}$ checks the reaction term directly, while $\mathcal{N}_{\mathrm{stress}}$ measures the consequence of a violation on the induced solution.
	\subsection{Recovery on the benchmarks}
	\label{subsec:num_recovery}
	
	We first verify that the pipeline recovers $f^\dagger$ on each benchmark and that the modification~\eqref{eq:mod_disc} does not reduce the recovery quality. For the latter we run the pipeline on each of the five benchmarks three times with otherwise identical settings: i) with the modification enabled (``constrained''), ii) disabled (``unconstrained''), and iii) disabled but with the violation of \eqref{cond:quasipos} penalized in the loss (``penalty''). In the latter two cases the modification in~\eqref{eq:mod_disc} drops out, so $\bar f_\theta\equiv f_\theta$. In iii) the discretized loss \eqref{eq:loss_disc} includes the additional term $\eta_{\mathrm{qp}}\mathcal{P}(\theta)$ with\vspace*{-0.2cm}
	\begin{equation}
		\label{eq:qpos_penalty}
		\mathcal{P}(\theta) := \frac{1}{N}\sum_{n=1}^{N}\frac{1}{J}
		\sum_{j=1}^{J}\big(\max\{-f_{\theta,n}(u^{n,j}),0\}\big)^2,\vspace*{-0.2cm}
	\end{equation}
	where the points $u^{n,1},\ldots,u^{n,J}$ are drawn uniformly from the face $\{u_n=0\}$ of a non-negative box covering the data range, the analogue of the $U_{\mathrm{qp}}$ of \eqref{eq:qpos_viol_emp} with the per-species maxima as upper corner, with $J=256$ per species and redrawn every epoch. We use $\eta_{\mathrm{qp}}=10$, selected as reported in Appendix~\ref{app:numerics}. Table~\ref{tab:recovery} reports the median over $n=10$ independent training seeds.
	
	\begin{table}[t]
		\centering
		\small
		\setlength{\tabcolsep}{2.5pt}
		\begin{tabular}{llccc}
			\toprule
			& & $\dfrac{\|f^\dagger-\bar f_\theta\|_{L^2(U_{\mathrm{tube}})}}{\|f^\dagger\|_{L^2(U_{\mathrm{tube}})}}$ &
			state err.\ & $|D_n^m-d_n^\dagger|/d_n^\dagger$\\
			\midrule
			B1 (Fisher-KPP)
			& con.   & $16\%$ $[14,18]$ & $0.21\%$ & $3\%\,(\pm 6)$\\
			& uncon. & $16\%$ $[13,19]$ & $0.18\%$ & $4\%\,(\pm 4)$\\
			& pen.   & $17\%$ $[14,20]$ & $0.21\%$ & $4\%\,(\pm 4)$\\
			\midrule
			B2 (Lotka-Volterra)
			& con.   & $4\%$ $[4,5]$ & $0.19\%$ & $30\%\,(\pm16)\,\big|\,50\%\,(\pm21)$\\
			& uncon. & $4\%$ $[4,5]$ & $0.19\%$ & $30\%\,(\pm16)\,\big|\,50\%\,(\pm21)$\\
			& pen.   & $4\%$ $[4,4]$ & $0.19\%$ & $37\%\,(\pm20)\,\big|\,58\%\,(\pm22)$\\
			\midrule
			B3 (Schnakenberg)
			& con.   & $25\%$ $[22,28]$ & $0.56\%$ & $60\%\,(\pm11)\,\big|\,39\%\,(\pm8)$\\
			& uncon. & $25\%$ $[22,28]$ & $0.56\%$ & $60\%\,(\pm11)\,\big|\,39\%\,(\pm8)$\\
			& pen.   & $19\%$ $[17,25]$ & $0.33\%$ & $42\%\,(\pm15)\,\big|\,23\%\,(\pm10)$\\
			\midrule
			B4 (SIR)
			& con.   & $15\%$ $[14,16]$ & $0.18\%$ & $18\%\,(\pm8)\,\big|\,13\%\,(\pm5)\,\big|\,52\%\,(\pm5)$\\
			& uncon. & $16\%$ $[16,17]$ & $0.13\%$ & $10\%\,(\pm9)\,\big|\,24\%\,(\pm5)\,\big|\,65\%\,(\pm6)$\\
			& pen.   & $15\%$ $[13,15]$ & $0.13\%$ & $26\%\,(\pm7)\,\big|\,5\%\,(\pm4)\,\big|\,51\%\,(\pm9)$\\
			\midrule
			B5 (mass-action)
			& con.   & $-$ & $0.20\%$ & $10\%\,(\pm2)\,\big|\,12\%\,(\pm2)\,\big|\,98\%\,(\pm0)$\\
			& uncon. & $-$ & $0.19\%$ & $10\%\,(\pm2)\,\big|\,11\%\,(\pm1)\,\big|\,98\%\,(\pm0)$\\
			& pen.   & $-$ & $0.20\%$ & $9\%\,(\pm1)\,\big|\,11\%\,(\pm2)\,\big|\,98\%\,(\pm0)$\\
			\bottomrule
		\end{tabular}
		\caption{Recovery quality of the constrained model ($\bar f_\theta$ with~\eqref{eq:mod_disc} enabled) and of the unconstrained and penalty-method baselines, medians over $n=10$ seeds. The diffusion column reports $|D_n^m - d_n^\dagger|/d_n^\dagger$ with $D_n^m := \tfrac{1}{L}\sum_{l=1}^{L} D_n^{m,l}$ the species-mean over the $L$ per-IC learned diffusions (standard deviation over the seeds in parentheses). The reaction-term error is taken over the generalization tube $U_{\mathrm{tube}}(\alpha{=}0.1)$ and is reported as the median with the interquartile range over the seeds in brackets. The reported medians are stable under the choice of seeds. Resampling the ten per-seed errors with replacement gives a $90\%$ interval for the median of width $3$ percentage points on B1 and $7$ on B3, against per-seed ranges of $11$--$25\%$ and $19$--$32\%$. For B5 the data occupies a vanishing fraction of the state space and $\|f^\dagger\|_{L^2(\mathrm{data})}\approx3\cdot10^{-6}\|f^\dagger\|_{L^2(U)}$ is tiny, so the reaction-term ratio is uninformative there and is not considered in the table (see Appendix~\ref{app:numerics}).}
		\label{tab:recovery}
	\end{table}
	
	The constrained model recovers the state to within $0.6\%$ relative error on every benchmark (the largest, $0.56\%$, on B3), and the empirical quasipositivity violation $\mathcal{Q}(\bar f_\theta)$ is \emph{exactly zero} on all five systems, a guarantee that holds by construction of~\eqref{eq:mod_disc}, independent of the data. On the generalization tube the reaction term is recovered to $4$-$25\%$ across the four benchmarks B1, B2, B3, B4. On B5 the ratio would attain roughly $100\%$, which does not indicate a failed recovery and is hence not reported. The reason is that B5 sits at chemical equilibrium, where the forward and backward rates $k_1AB$ and $k_2C$ cancel and $f^\dagger$ vanishes. Here this happens in a degenerate way as $A$ and $B$ start on opposite halves of $\Omega$ and diffuse too slowly to meet on $[0,T]$, so that both rates are themselves almost zero, with $\|f^\dagger\|_{L^\infty(\mathrm{data})}\approx5\cdot10^{-6}$. The relative error thus divides by small normalization $\|f^\dagger\|_{L^2(U_{\mathrm{tube}})}\approx1.2\cdot10^{-2}$ and measures the smallness of $f^\dagger$ instead of the accuracy of its recovery. In absolute terms B5 attains $\|f^\dagger-\bar f_\theta\|_{L^2(U_{\mathrm{tube}})}\approx1.2\cdot10^{-2}$, comparable to B1 and B4 ($0.9\cdot10^{-2}$ and $1.1\cdot10^{-2}$), whose ratios of $16\%$ and $15\%$ only reflect a $4$ to $7$ times larger normalization (see Appendix~\ref{app:numerics}). The diffusion coefficients separate into a well-determined group, recovered to $3$-$30\%$ relative error, and an underdetermined one ranging from $39\%$ to $98\%$, the latter including the predator species of B2, both components of B3, the recovered compartment of B4 and the product species $C$ of B5. The grouping does not follow the size of $d^\dagger_n$. B2's two species share $d^\dagger_n=10^{-2}$ yet land at $30\%$ and $50\%$, and the largest diffusion of all benchmarks, $d^\dagger_2=10^{-1}$ on B3, is underdetermined. The separation between the two groups is therefore visible but not wide, and the reason is due to the problem itself rather than the training. The diffusion and the reaction enter the residual in~\eqref{eq:all_at_once_mod} additively through $D\Delta u+\bar f_\theta(u)$, so the data constrain that sum far better than either summand separately. Accuracy consequently moves between $D$ and $\bar f_\theta$ instead of both. Appendix~\ref{app:numerics} supports this, as halving the time horizon on B3 improves the reaction term while the diffusion errors increase to $98\%$ and $80\%$. The reported diffusion errors should thus be read as a statement about the identifiability of the pair, not about the optimizer. The B5 product $C$ is the extreme case as $C$ enters only via the forward reaction $A+B\to C$ due to $C(0)=0$ and moves to a nearly uniform spatial profile, making $d_C\Delta C$ contribute orders of magnitude less to the residual than the reaction.

	Compared to the two baselines, the modification does not reduce the recovery quality. On B2 and B3 the constrained and the unconstrained model coincide. Here the smallest state values in the data are $0.32$ and $0.27$, whereas $\chi^m(u_n)=0$ for $u_n\geq\varepsilon^m+\delta^m=0.075$, so that the modification does not act during training and both runs are identical for the same seed. On B1 both models attain a reaction-term error of $16\%$, with a state error of $0.21\%$ against $0.18\%$. On B4 the constrained model attains a reaction-term error of $15\%$ against $16\%$, with a larger state error ($0.18\%$ against $0.13\%$) and diffusion errors that are smaller for the infected and recovered compartments and larger for the susceptible one. On B5 the three models agree. The penalty baseline attains a comparable accuracy on B1, B2, B4 and B5 and a better one on B3, with a reaction-term error of $19\%$ against $25\%$, a state error of $0.33\%$ against $0.56\%$ and diffusion errors of $42\%$ and $23\%$ against $60\%$ and $39\%$. As shown in the next subsection, however, it does not ensure quasipositivity.
	
	\subsection{Evaluation of quasipositivity}
	\label{subsec:num_comparison}

	We now compare the three models of Subsection~\ref{subsec:num_recovery} with respect to physical consistency. Table~\ref{tab:comparison} reports the quasipositivity violation $\mathcal{Q}$ of the learned reaction term and the negativity $\mathcal{N}_{\mathrm{stress}}$ of the learned model's solution from the stress initial conditions.
	
	\begin{table}[t]
		\centering
		\small
		\begin{tabular}{llcc}
			\toprule
			& & $\mathcal{Q}(\,\cdot\,)$ & $\mathcal{N}_{\mathrm{stress}}$\\
			\midrule
			B1 (Fisher-KPP)
			& con.   & $0$ & $0$ \\
			& uncon. & $0$ ($0/3.2\!\cdot\!10^{-3}$) & $0$ ($0/406$) \\
			& pen.   & $7.1\!\cdot\!10^{-5}$ ($0/1.7\!\cdot\!10^{-3}$) & $0.051$ ($0/5.8$) \\
			\midrule
			B2 (Lotka-Volterra)
			& con.   & $0$ & $0$ \\
			& uncon. & $0.517$ ($0.443/0.620$) & $0.047$ ($0/0.169$) \\
			& pen.   & $0$ ($0/0.018$) & $0$ \\
			\midrule
			B3 (Schnakenberg)
			& con.   & $0$ & $0$ \\
			& uncon. & $0.543$ ($0.357/0.735$) & $0.152$ ($0/0.536$) \\
			& pen.   & $0$ ($0/0.053$) & $0$ ($0/\text{div.}$) \\
			\midrule
			B4 (SIR)
			& con.   & $0$ & $0$\\
			& uncon. & $0.536$ ($0.284/0.671$) & $0.067$ ($0.044/0.140$) \\
			& pen.   & $0.043$ ($0.007/0.189$) & $0.0059$ ($0/0.030$) \\
			\midrule
			B5 (mass-action)
			& con.   & $0$ & $0$ \\
			& uncon. & $4.0\!\cdot\!10^{-5}$ ($0/6.2\!\cdot\!10^{-4}$) & $6.0\!\cdot\!10^{-4}$ ($0/9\!\cdot\!10^{-3}$) \\
			& pen.   & $1.7\!\cdot\!10^{-4}$ ($0/7.8\!\cdot\!10^{-4}$) & $2.5\!\cdot\!10^{-3}$ ($0/1.2\!\cdot\!10^{-2}$) \\
			\bottomrule
		\end{tabular}
		\caption{Quasipositivity violation and stress roll-out negativity of the constrained model and of the unconstrained and penalty-method baselines on the five benchmarks, medians over $n=10$ seeds, with min/max over the seeds in parentheses wherever the median does not already determine the outcome. A bare $0$ means every seed is $0$ to numerical precision. The penalty rows use $\eta_{\mathrm{qp}}=10$ in \eqref{eq:qpos_penalty}. On three of the ten B3 penalty seeds the stress roll-out diverges, which happens on no seed of either other baseline.}
		\label{tab:comparison}
	\end{table}
	
	Since the constrained and the unconstrained model attain a comparable accuracy (Subsection~\ref{subsec:num_recovery}), the difference between the two lies in $\mathcal{Q}$ and $\mathcal{N}_{\mathrm{stress}}$. The unconstrained baseline violates quasipositivity~\eqref{cond:quasipos} consistently across seeds, with a median $L^\infty$ violation of $0.52$ on B2, $0.54$ on B3 and $0.54$ on B4, and with a minimum of $0.28$ over all seeds of these three benchmarks. Relative to the typical reaction-rate magnitude $\|f^\dagger\|_{L^\infty(\mathrm{data})}$ this amounts to $40\%$, $33\%$ and $128\%$, i.e., the violation, where it occurs, is of the order of a typical reaction-rate value. The constrained model attains $\mathcal{Q}=0$ by construction. To assess the consequence of the unconstrained model's quasipositivity violation on the solution itself, we forward-integrate each model from the stress initial conditions, using \emph{that model's own} learned pair $(\bar f_\theta,D)$. On B4, over the ten seeds the unconstrained stress-rollout negativity has median $0.067$ and reaches $0.140$. On B3, the median stress-rollout negativity is $0.152$ (Table~\ref{tab:comparison}). Relative to the largest state value in the data, $\|u^\dagger\|_{L^\infty(\mathrm{data})}$, the medians of $\mathcal{N}_{\mathrm{stress}}$ amount to $2\%$ on B2, $7\%$ on B3, $6\%$ on B4 and $0.06\%$ on B5, and on B2, B3 and B4 they are of the order of the amplitude $0.05\max_x v_n^0$ of the species driven to the boundary. The constrained model, by contrast, stays non-negative to numerical precision on every seed. For the unconstrained model on B2 the roll-out negativity varies strongly with the seed. Its median, $0.047$, is the smallest among B2, B3 and B4, it vanishes on one of the ten seeds, and it reaches $0.169$ on the most severe one. The defect itself, measured by $\mathcal{Q}$, is nevertheless present on every seed. On B5 the data already drives the unconstrained network close to quasipositive on the training support, yet only the constrained model guarantees it. On B1 the unconstrained model violates $f(0)\geq 0$ on three of the ten seeds, by at most $3.2\cdot10^{-3}$. Since $u=0$ is an unstable equilibrium of B1, this suffices for the stress roll-out to become negative on these seeds, reaching $-406$. As illustrated in Figure~\ref{fig:rollout}, on B4 the unconstrained state violates non-negativity in all three species, most strongly in the susceptible one. Thus, under out-of-training conditions the unconstrained model is not physically consistent, whereas the constrained model is, by construction.

	The penalty baseline, while comparable or better in accuracy, does not ensure quasipositivity. On B2 and B3 its median violation is zero, but not on every seed as the maxima over the ten seeds are $0.018$ and $0.053$. On B4 it does not reach zero at all, at a median $0.043$, i.e., $10\%$ of a typical reaction rate. On B1 it violates $f(0)\geq0$ on five of the ten seeds. On B5 it is counterproductive, its median violation $1.7\cdot10^{-4}$ exceeding the unconstrained model's $4.0\cdot10^{-5}$. Here the true reaction rates are of size $\|f^\dagger\|_{L^\infty(\mathrm{data})}\approx 5\cdot10^{-6}$, more than two orders of magnitude below the noise level $\sigma=10^{-3}$, so that the data carry almost no information about $f$ and both violations are far below what they could resolve. Finally, on three of the ten B3 seeds the penalty-trained stress roll-out diverges, which occurs on no seed of either other baseline. The penalty controls the sign of $f_\theta$ on the sampled faces but nothing in it bounds the growth of $f_\theta$ away from the data manifold. The constrained model, by contrast, attains $\mathcal{Q}=0$ and stays non-negative, by construction and independently of the data. A penalty can thus make physical consistency approximately hold but only the modification guarantees it.
	
	\begin{figure}[t]
		\centering
		\includegraphics[width=.6\linewidth]{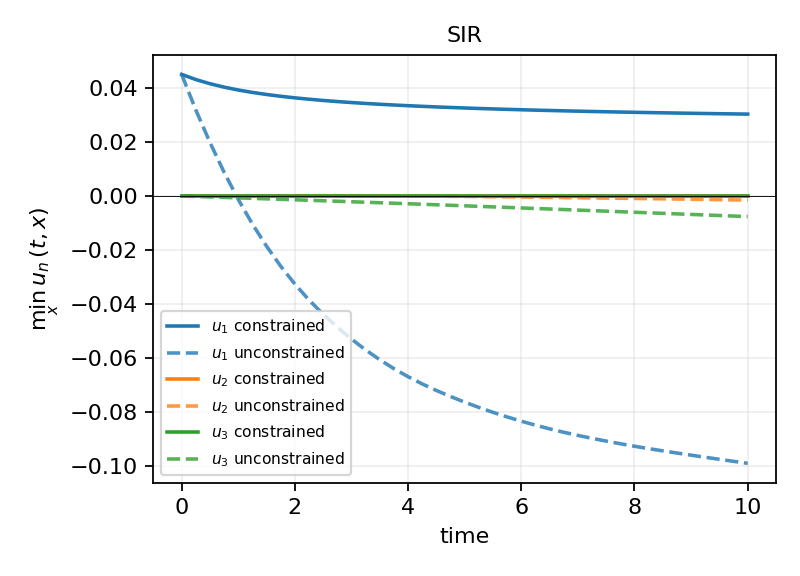}
		\caption{Minimum over the spatial domain of each species during the stress roll-outs on B4. Constrained (solid), unconstrained (dashed). All three species of the unconstrained model take negative values, the susceptible one reaching $-0.099$ and the infected and recovered ones $-0.0015$ and $-0.0076$. The constrained model stays non-negative. The pair shown is retrained for the figure, %
			and is not one of the ten seeds in Table~\ref{tab:comparison}, its minimum lies inside that table's per-seed range $0.044$-$0.140$ on B4.}
		\label{fig:rollout}
	\end{figure}
	
	\subsection{Sensitivity to the transition function}
	\label{subsec:num_transition}
	
	In this subsection we test for a possible negative effect of the construction including $\chi^m$ on gradient issues under training. We address it empirically using three realizations of $\chi^m$, mollified Heaviside (X1), $\tanh$ smoothstep (X2) and the quintic smoothstep (X3), together with the cut-off level $\varepsilon^m\in\{0.5, 0.2, 0.1, 0.05, 0.02\}$ on B4 (Table~\ref{tab:transition}). The gradient-norm panel of Figure~\ref{fig:transition_sweep} reports the median loss-gradient norm $\|\nabla_\theta\mathcal{L}^m\|$ over the final $100$ epochs. We restrict the median to the final epochs since $\lambda^m,\mu^m,\nu^m$ are ramped during training, so that a median over all iterates would average gradient norms of different objectives, whereas at the end of training the schedules rest at their end values and the objective is fixed.
	\begin{table}[t]
		\centering
		\small
		\begin{tabular}{l cc@{\hskip 1em} cc@{\hskip 1em} cc}
			\toprule
			& \multicolumn{2}{c}{X1 (mollified)} & \multicolumn{2}{c}{X2 ($\tanh$)} & \multicolumn{2}{c}{X3 (quintic)}\\
			$\varepsilon^m$ & $\|\cdot\|_{L^2}$ & $\mathcal{Q}$ & $\|\cdot\|_{L^2}$ & $\mathcal{Q}$ & $\|\cdot\|_{L^2}$ & $\mathcal{Q}$\\
			\midrule
			$0.5 $ & $0.18$ & $0$ & $0.32$ & $5.4\cdot10^{-4}$ & $0.29$ & $0$\\
			$0.2 $ & $0.11$ & $0$ & $0.18$ & $6.8\cdot10^{-5}$ & $0.10$ & $0$\\
			$0.1 $ & $0.17$ & $0$ & $0.18$ & $3.9\cdot10^{-6}$ & $0.18$ & $0$\\
			$0.05$ & $0.16$ & $0$ & $0.16$ & $5.3\cdot10^{-7}$ & $0.16$ & $0$\\
			$0.02$ & $0.16$ & $0$ & $0.17$ & $8.2\cdot10^{-7}$ & $0.15$ & $0$\\
			\bottomrule
		\end{tabular}
		\caption{Final reaction-term $L^2$ error over the bounding box $U$ and quasipositivity violation for the three $\chi^m$ variants for different cut-off levels $\varepsilon^m$, on B4 at seed $0$.}
		\label{tab:transition}
	\end{table}
	\begin{figure}[t]
		\centering
		\includegraphics[width=.95\linewidth]{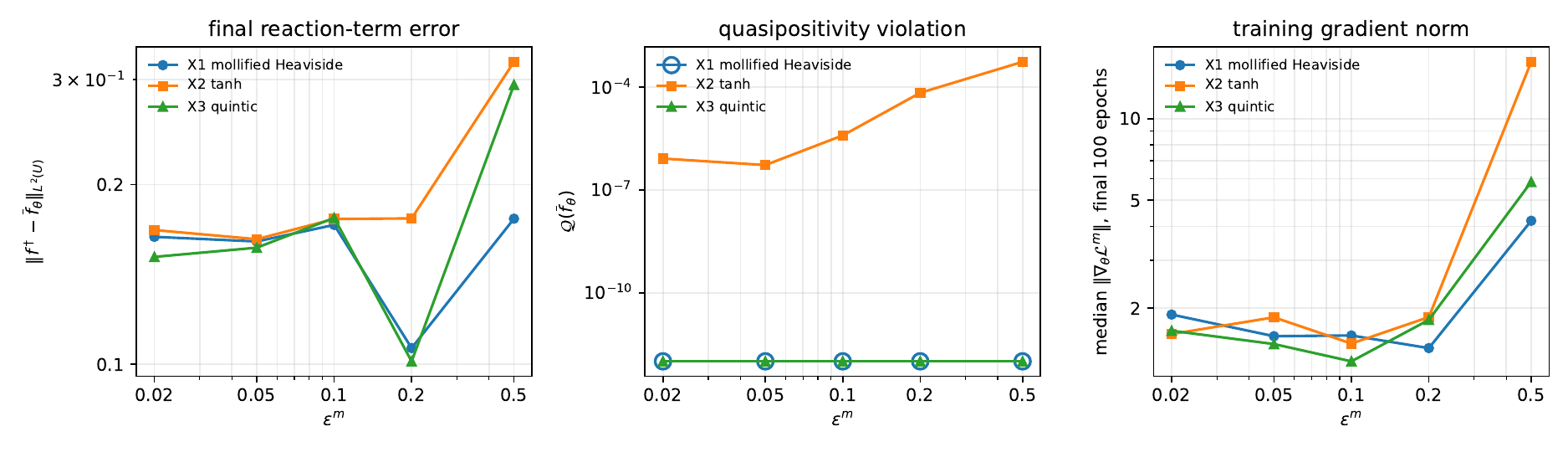}
		\caption{Final reaction-term $L^2$ error (left), quasipositivity violation $\mathcal{Q}(\bar f_\theta)$ (middle), and median gradient norm of the training loss over the final $100$ epochs (right) as functions of $\varepsilon^m$, for the three transition-function variants on B4. In the middle panel X1 (blue) and X3 (green) both lie at $\mathcal{Q}=0$ (clamped to $10^{-12}$ for log display) and coincide at the floor.}
		\label{fig:transition_sweep}
	\end{figure}
	
	We observe the following. First, in the working regime $\varepsilon^m\in[0.02,0.1]$ all three variants train to comparable accuracy, with $L^2$ errors between $0.15$ and $0.18$, and the gradient norm shows no collapse. Its tail median lies between $1.3$ and $1.9$ for every variant and every cut-off level in this regime, i.e., it is essentially independent of both. Second, X1 and X3 keep $\mathcal{Q}(\bar f_\theta)$ \emph{exactly} zero, whereas X2 leaks a violation between $5\cdot10^{-7}$ and $5\cdot10^{-4}$, growing with $\varepsilon^m$ and negligible relative to $\|f^\dagger\|_{L^\infty(\mathrm{data})}\approx 0.42$ on B4, but strictly nonzero. The modification rectifies $\bar f_n$ at $\{u_n=0\}$ to $\chi^m(0)P_+(f_n)+(1-\chi^m(0))f_n$, which equals $\max(f_n,0)\geq 0$ when $\chi^m(0)=1$. X1 and X3 satisfy this, while X2 gives $\chi^m(0)=\tfrac{1}{2}(1+\tanh(\varepsilon^m/(\delta^m/3)))\approx 1-6\!\cdot\!10^{-6}$ for every $\varepsilon^m$, and the missing $6\!\cdot\!10^{-6}$ fraction of $f_n$ leaks through. At $\varepsilon^m=0.5$ the smooth variants X2 and X3 additionally degrade in the $L^2$ error, to $0.32$ and $0.29$, and their tail gradient norm rises to $16.2$ and $5.8$, because the transition region becomes wide and distorts $\bar f_\theta$. The compactly supported X1 is least affected, with $0.18$ and a tail gradient norm of $4.2$. This effect only sets in at $\varepsilon^m=0.5$, as at $\varepsilon^m=0.2$ the transition still covers a minor part of the data range and all three variants remain at the accuracy of the working regime. %
	
	\subsection{Limitations}
	\label{subsec:num_limitations}
	
	\emph{Off-manifold behavior.} The data constrain $\bar f_\theta$ only on the generalization tube, off the data manifold the reaction term is governed by the regularizer, so for near-equilibrium systems (B3, B5) the global $L^2(U)$ error is large even though the on-tube error and the consistency are unaffected. B5 is the extreme case since its default IC sits at near-equilibrium so $\|f^\dagger\|_{L^\infty(\mathrm{data})}\!\lesssim\!10^{-5}$ on the data. The data determine the pair $(D,\bar f_\theta)$ jointly, and two modifications improve the reaction term while worsening the diffusion coefficients, on B3 by halving the horizon and on B5 by widening the species overlap. Some species, e.g., the product species of B5, are only weakly identifiable from the data, a design effect of near-equilibrium ICs.\\[-8pt]
	
	\noindent\emph{Theorem applicability.} The experiments report a fixed finite-$m$ operating point with $K^m=\mathrm{Id}$, and the parameterized class is a generic feed-forward MLP rather than the architectures of \cite{Belomestny2023,raslan21}, for which Assumption~\ref{ass:approx_cap_cond} is verified. Strict quasipositivity in the strong sense (Assumption~\ref{ass:strict_quasipositivity} with $\alpha>1$) holds trivially for B1 ($f^\dagger\!\geq\!0$ on a whole neighborhood of the boundary face, so that $P_-\circ f^\dagger$ vanishes there) and uniformly for B3 ($f^\dagger$ bounded below by $a,b>0$), but only weakly ($\alpha=1$) for B2, B4 and B5. %
	
	\section{Conclusions}
	\label{sec:conclusions}
	In this work, we addressed the problem that learning the reaction term of a reaction-diffusion (RD) system within a parameterized class of functions does not by itself guarantee that the learned model is mathematically well-posed or physically meaningful. To this end, we proposed a framework that incorporates the conditions in question, in particular mass control and quasipositivity, directly into the parameterization of the reaction terms. In this way, every candidate model during learning, and not only the limit, defines an RD system which preserves non-negativity and which is well-posed under additional regularity and growth conditions, by construction.%
	
	Within a regularization-based model learning framework, we then showed that identifiability is not lost through this restriction of the class. Specifically, we proved that solutions to the learning problem converge to a unique, regularization-minimizing solution in the limit of full, noiseless measurements, even when conservation laws and quasipositivity are enforced. %
	
	The theoretical contributions were complemented by numerical experiments, carried out on five RD systems. These experiments confirm that the modified parameterization is practically realizable and that it recovers both the state and the reaction term, while its quasipositivity violation vanishes exactly rather than approximately. The corresponding unconstrained baseline, trained under otherwise identical settings, attains comparable accuracy but violates quasipositivity, and solutions of the learned system become negative when it is simulated from initial conditions outside the training set. A penalty-method baseline promoting quasipositivity does not remove the violation on every seed either. Enforcing physical consistency at the level of the parameterization is therefore essential in practice.
	
	Our contributions combine data-driven modeling with physical consistency to obtain reliable models for RD systems. Future work could extend the approach to additional conservation laws, and address the numerical limitations discussed in Subsection~\ref{subsec:num_limitations}. In particular, the joint determination of the diffusion coefficients and the reaction term remains a limitation.
\section*{Acknowledgement}
This research was funded in whole or in part by the Austrian Science Fund (FWF) 10.55776/F100800.
\section*{Declarations}
	\textbf{Conflict of interest}\quad The authors have not disclosed any competing interests.
	\appendix
	\section{Proofs deferred from Subsection~\ref{subsec:wellposedness}}
	\label{app:assum_reactions}
	In this section, we discuss Assumption~\ref{ass:param_reaction} for the physically consistent classes $\overline{\mathcal{F}}_n^m$ of modified parameterized reaction terms in \eqref{eq:mod_class}, with $\bar{f}_{\theta_n, n}$ defined by
	\begin{align}
		\label{def_fmod}
		\bar{f}_{\theta_n, n}(u) = ((P_+\circ f_{\theta_n,n})(u)-f_{\theta_n,n}(u))\chi^m(u_n)+f_{\theta_n,n}(u)
	\end{align}
	as in \eqref{modification:param}, where $(\chi^m)_m$ is a sequence of transition functions as in Definition~\ref{def:transition}.
	
	\subsection{Extension property}
	\label{appendix:extension}
	We prove first that $\bar{f}_{\theta_n,n}$ induces a well-defined Nemytskii operator mapping from $\mathcal{V}^N$ to $\mathcal{W}$ with $[\bar{f}_{\theta_n,n}(v)](t)(x)=\bar{f}_{\theta_n,n}(v(t,x))$. In view of \eqref{def_fmod} this puzzles down to the consideration of the terms $P_+\circ f_{\theta_n, n}$, $\chi^m(u_n)$ and the product involved in \eqref{def_fmod}. This can be achieved under the space setup in Assumption~\ref{ass_init_set} if the elements of $\mathcal{F}_n^m$ are Lipschitz continuous. We will also formulate higher regularity extension results, but require a more regular space setup for that. %
	We start by verifying auxiliary lower and higher regularity results for the extension of general Lipschitz continuous functions $g:\mathbb{R}^N\to \mathbb{R}$ to well-defined Nemytskii operators.
	\begin{lemma}
		\label{lem:lower_chi}
		Let Assumption~\ref{ass_init_set} hold true and let $g:\mathbb{R}^N\to \mathbb{R}$ be Lipschitz continuous with constant $L>0$. Then $g:\mathcal{V}^N\to L^p(0,T;L^{\hat{p}}(\Omega))$ defines a well-defined Nemytskii operator and moreover, also $g:\mathcal{V}^N\to\mathcal{W}$.
	\end{lemma}
	\begin{proof}
		For $u\in \mathcal{V}^N$ it holds for a.e. $t\in ]0,T[$ that $u(t,\cdot)$ is measurable. By continuity of $g$ also $g(u(t,\cdot))$ is measurable for a.e. $t\in]0,T[$ and $$\Vert g(u(t,\cdot))\Vert_{L^{\hat{p}}(\Omega)} \leq L \Vert u(t,\cdot)\Vert_{L^{\hat{p}}(\Omega)^N}+\vert g(0)\vert\vert\Omega\vert^{1/\hat{p}}<\infty$$ for a.e. $t\in]0,T[$ by $V\hookdoubleheadrightarrow L^{\hat{p}}(\Omega)$. Due to weak measurability of $t\mapsto g(u(t,\cdot))$ and separability of $L^{\hat{p}}(\Omega)$ as $1\leq \hat{p}<\infty$ we derive by Pettis Theorem that $t\mapsto g(u(t,\cdot))$ is Bochner measurable. As for $u\in\mathcal{V}^N$ with $\mathcal{V}=L^p(0,T;V)\cap W^{1,p,p}(0,T;\tilde{V})$
		\[
		\Vert g(u)\Vert_{L^p(0,T;L^{\hat{p}}(\Omega))}\leq L\Vert u\Vert_{L^p(0,T;L^{\hat{p}}(\Omega))^N}+\vert g(0)\vert T^{1/p}\vert\Omega\vert^{1/\hat{p}}<\infty
		\]
		the remaining assertion follows as $p\geq q$ and $L^{\hat{p}}(\Omega)\hookrightarrow W$.
	\end{proof}
	\begin{lemma}
		\label{higher_reg_a}
		Let Assumption~\ref{ass_init_set} hold true and let $g:\mathbb{R}^N\to \mathbb{R}$ be Lipschitz continuous with constant $L>0$. Suppose further that $V\hookrightarrow W^{1,\tilde{p}}(\Omega)$ for some $1<\tilde{p}<\infty$. Then $g$ defines a well-defined Nemytskii operator $g:\mathcal{V}^N\to L^p(0,T;W^{1,\tilde{p}}(\Omega))$.%
	\end{lemma}
	\begin{proof}
		For $u\in \mathcal{V}^N$ it holds for a.e. $t\in ]0,T[$ that $u(t,\cdot)$ is measurable. By continuity of $g$ also $g(u(t,\cdot))$ is measurable for a.e. $t\in]0,T[$. By Theorem~\ref{th:mizel} we have that $g: V^N\to W^{1,\tilde{p}}(\Omega)$ is well-defined, bounded and continuous. In particular, we have for $u\in\mathcal{V}^N$ and a.e. $t\in]0,T[$ that $\Vert g(u(t,\cdot))\Vert_{W^{1,\tilde{p}}(\Omega)}<\infty$. Due to weak measurability of $t\mapsto g(u(t,\cdot))$ and separability of $W^{1,\tilde{p}}(\Omega)$ as $1< \tilde{p}<\infty$ we derive by Pettis Theorem that $t\mapsto g(u(t,\cdot))$ is Bochner measurable. We show that for $u\in\mathcal{V}^N$ with $\mathcal{V}=L^p(0,T;V)\cap W^{1,p,p}(0,T;\tilde{V})$ it holds
		\begin{align}
			\label{assert_florin}
			\Vert g(u)\Vert_{L^p(0,T;W^{1,\tilde{p}}(\Omega))}<\infty.
		\end{align}
		By \cite[Theorem 1.3]{Florin2022} the chain rule applied to $g(u)$ holds almost everywhere, i.e., $$\nabla_x(g(u(t,x)))=\nabla g(u(t,x))\nabla_x u(t,x) \quad \text{for a.e.} ~ (t,x)\in]0,T[\times\Omega.$$ Thus, with $\Vert \nabla g(u)\Vert_{L^\infty(]0,T[\times\Omega)}\leq L$ we obtain for some $c>0$ that
		\begin{align*}
			\Vert g(u)-&g(0)\Vert_{L^p(0,T;W^{1,\tilde{p}}(\Omega))}\\
			&= \left(\int_0^T\left(\int_\Omega \vert g(u(t,x))-g(0)\vert^{\tilde{p}}\dx x\dx t+\int_\Omega\vert\nabla_x(g(u(t,x)))\vert^{\tilde{p}}\dx x\dx t\right)^{p/\tilde{p}}\right)^{1/p}\\
			&\leq L \left(\int_0^T\left(\int_\Omega \vert u(t,x)\vert^{\tilde{p}}\dx x\dx t+\int_\Omega\vert \nabla_xu(t,x)\vert^{\tilde{p}}\dx x\dx t\right)^{p/\tilde{p}}\right)^{1/p}\\
			&\leq c\Vert u\Vert_{L^p(0,T;W^{1,\tilde{p}}(\Omega))^N}.
		\end{align*}
		Consequently, the assertion in \eqref{assert_florin} follows by the embedding $V\hookrightarrow W^{1, \tilde{p}}(\Omega)$.%
	\end{proof}
	With this, we can prove the extension result for $\bar{f}_{\theta_n,n}$ formulated in Proposition~\ref{prop:f_mod_nemyt}:
	\begin{proof}[Proof of Proposition~\ref{prop:f_mod_nemyt}]
		We consider first the extension of $\chi_n:\mathbb{R}^N\to \mathbb{R}, u\mapsto \chi^m(u_n)$ to $\chi_n:\mathcal{V}^N\to \mathcal{W}$, w.l.o.g.\ for $N=1$. As $\chi^m$ is a transition function, the $\chi_n^{(l)}$ are compactly supported in a common interval for all $l\in \mathbb{N}$, and $\sup_{x\in\mathbb{R}}\vert \chi_n^{(l)}(x)\vert <\infty$ for $l\in \mathbb{N}_0$. Thus, the $\chi^{(l)}_n$ are Lipschitz continuous for $l\in\mathbb{N}_0$ and it holds true that
		\begin{align}
			\label{chiderivatives}
			\sup_{0\leq j\leq l}\Vert \chi_n^{(j)}\Vert_{\mathcal{C}(\mathbb{R})} <\infty
		\end{align}
		for $l\in \mathbb{N}_0$. In particular the results in Lemma~\ref{lem:lower_chi} apply. Now since $f_{\theta_n, n}$ is Lipschitz continuous, also $P_+\circ f_{\theta_n,n}$ is Lipschitz continuous. By Lemma~\ref{lem:lower_chi}, $P_+\circ f_{\theta_n,n}:\mathcal{V}^N\to L^p(0,T;L^{\hat{p}}(\Omega))$ is well-defined. Given spatial regularity $L^{\hat{p}}(\Omega)$ of $P_+\circ f_{\theta_n,n}-f_{\theta_n,n}$ we derive by $\Vert \chi^m\Vert_{\mathcal{C}(\mathbb{R})} \leq1$ that for $u\in \mathcal{V}^N$ and a.e. $t\in ]0,T[$ it holds
		\begin{multline*}
			\Vert ((P_+\circ f_{\theta_n,n})(u(t,\cdot))-f_{\theta_n,n}(u(t,\cdot)))\chi_n(u(t,\cdot))\Vert_{L^{\hat{p}}(\Omega)}\\
			\leq \Vert (P_+\circ f_{\theta_n,n})(u(t,\cdot))-f_{\theta_n,n}(u(t,\cdot))\Vert_{L^{\hat{p}}(\Omega)}<\infty.
		\end{multline*}
		As in the previous results the underlying image space is separable and the operator is Bochner measurable. To obtain space-time regularity we have to assure that the image space $V$ under multiplication is contained in $L^{\hat{p}}(\Omega)$ by a suitable Sobolev embedding as before. Then
		\begin{align*}
			\Vert (P_+\circ f_{\theta_n,n}-f_{\theta_n,n})(u)\chi_n(u)\Vert_{L^{p}(0,T;L^{\hat{p}}(\Omega))} \leq \Vert (P_+\circ f_{\theta_n,n})(u)-f_{\theta_n,n}(u)\Vert_{L^{p}(0,T;L^{\hat{p}}(\Omega))},
		\end{align*}
		as $\Vert \chi^m\Vert_{\mathcal{C}(\mathbb{R})} \leq1$, which is bounded by the previous considerations and a well-defined Nemytskii operator mapping $\mathcal{V}^N$ to $\mathcal{W}$. The same applies to $\bar{f}_{\theta_n,n}$.
	\end{proof}
	Similarly a higher regularity result based on Lemma~\ref{higher_reg_a} can be recovered:
	\begin{proposition}
		Under Assumption~\ref{ass_init_set} with $V\hookrightarrow W^{1,\tilde{p}}(\Omega)$ for some $2\leq\tilde{p}<\infty$ and $p\geq 2$, let Assumption~\ref{ass:param_reaction} hold for $\mathcal{F}_n^m$ and its elements $f_{\theta_n,n}$ be Lipschitz continuous. Then $\bar{f}_{\theta_n, n}:\mathcal{V}^N\to L^{p/2}(0,T;W^{1,\tilde{p}/2}(\Omega))$ is well-defined.
	\end{proposition}
	\begin{proof}
		Denoting as in the previous proof $\chi_n:\mathbb{R}^N\to \mathbb{R}, u\mapsto \chi^m(u_n)$, well-definedness of $\chi_n, f_{\theta_n,n},P_+\circ f_{\theta_n, n}:\mathcal{V}^N\to L^p(0,T;W^{1,\tilde{p}}(\Omega))$ follows by Lemma~\ref{higher_reg_a}. Given spatial regularity $W^{1,\tilde{p}}(\Omega)$ for $\chi_n$ and $f_{\theta_n,n}$ it holds by Theorem~\ref{th:behzadan} that
		\[
		\Vert ((P_+\circ f_{\theta_n,n})(u)-f_{\theta_n,n}(u))\chi_n(u)\Vert_{W^{1,\tilde{p}/2}(\Omega)}<\infty.
		\]
		For spatio-temporal regularity note that we have for some constant $c>0$%
		\begin{multline*}
			\Vert ((P_+\circ f_{\theta_n,n})(u)-f_{\theta_n,n}(u))\chi_n(u)\Vert^{p/2}_{L^{p/2}(0,T;W^{1,\tilde{p}/2}(\Omega))}\\
			\leq c\Vert u\Vert^{p/2}_{L^{p/2}(0,T;W^{1,\tilde{p}/2}(\Omega))}+c\Vert (P_+\circ f_{\theta_n,n})(u)-f_{\theta_n,n}(u)\Vert^{p/2}_{L^{p/2}(0,T;W^{1,\tilde{p}/2}(\Omega))}\\
			+c\int_0^T\left(\int_\Omega \vert u(t,x)\nabla_x u(t,x)\vert^{\tilde{p}/2}\dx x\right)^{p/\tilde{p}}\dx t.
		\end{multline*}
		As $V \hookrightarrow W^{1,\tilde{p}}(\Omega)\hookrightarrow W^{1,\tilde{p}/2}(\Omega)$ it holds that
		\begin{align}
			\label{termsbdd}
			\Vert (P_+\circ f_{\theta_n,n})(u)-f_{\theta_n,n}(u)\Vert_{L^{p/2}(0,T;W^{1,\tilde{p}/2}(\Omega))}, \Vert u\Vert_{L^{p/2}(0,T;W^{1,\tilde{p}/2}(\Omega))}<\infty.
		\end{align}
		Now by boundedness of
		\begin{align}
			\label{uestimation}
			\notag\int_0^T\left(\int_\Omega \vert u(t,x)\nabla_x u(t,x)\vert^{\tilde{p}/2}\dx x\right)^{p/\tilde{p}}\dx t&\leq \int_0^T\Vert u(t,\cdot)\Vert_{L^{\tilde{p}}(\Omega)}^{p/2}\Vert \nabla_x u(t,\cdot)\Vert_{L^{\tilde{p}}(\Omega)}^{p/2}\dx t\\
			\notag&\leq \Vert u\Vert^{p/2}_{L^p(0,T;L^{\tilde{p}}(\Omega))}\Vert \nabla u\Vert^{p/2}_{L^p(0,T;L^{\tilde{p}}(\Omega))}\\
			&\leq\Vert u\Vert^p_{L^p(0,T;W^{1,\tilde{p}}(\Omega))},
		\end{align}
		\[
		\text{also}\quad
		\Vert ((P_+\circ f_{\theta_n,n})(u)-f_{\theta_n,n}(u))\chi_n(u)\Vert_{L^{p/2}(0,T;W^{1,\tilde{p}/2}(\Omega))}<\infty.
		\]
		By definition of $\bar{f}_{\theta_n,n}$ in \eqref{def_fmod} the claimed statement follows.
	\end{proof}
	The discussed results hold in particular under the space setup in Assumption~\ref{ass_init_set}.
	
	\begin{remark}[Maximal regularity]
		Above estimations do not exploit the maximal possible regularity. Under Assumption~\ref{ass_init_set} for $W^{1,\tilde{p}}(\Omega)-$regularity of $P_+\circ f_{\theta_n,n}-f_{\theta_n,n}$ and $\chi_n$ (by Theorem~\ref{th:mizel}) it is possible to show that $((P_+\circ f_{\theta_n,n})(u)-f_{\theta_n,n}(u))\chi_n(u)$ attains $W^{1,\beta}(\Omega)$-regularity for $u\in\mathcal{V}^N$ with maximal $\beta$ given by
		\[
		\beta=
		\begin{cases}
			\tilde{p} &\text{if} ~ \tilde{p}> d\\
			\tilde{p}-\epsilon & \text{if } ~ \tilde{p}=d\\
			\frac{d\tilde{p}}{2d-\tilde{p}}&\text{else}
		\end{cases}
		\]
		for some small $\epsilon>0$ using Theorem~\ref{th:behzadan}. Note that $W^{\tilde{m},\tilde{p}}(\Omega)\hookrightarrow W^{1, v}(\Omega)$ with $v=\frac{d\tilde{p}}{d-(\tilde{m}-1)\tilde{p}}\geq \tilde{p}\geq \beta$ if $\tilde{p}<\frac{d}{\tilde{m}-1}$ and else $W^{\tilde{m},\tilde{p}}(\Omega)\hookrightarrow W^{1, v}(\Omega)$ in particular for $\tilde{p}\leq v<\infty$, such that the terms in \eqref{termsbdd} are indeed bounded for above $\beta$ instead of $\tilde{p}/2$. The term in \eqref{uestimation} is bounded for above $\beta$ instead of $\tilde{p}/2$ due to Hölder's generalized inequality by the previous embedding together with $W^{\tilde{m},\tilde{p}}(\Omega)\hookrightarrow L^{\frac{d\tilde{p}}{d-\tilde{m}\tilde{p}}}(\Omega)$ if $\tilde{p}<\frac{d}{\tilde{m}}$ and else $W^{\tilde{m},\tilde{p}}(\Omega)\hookrightarrow L^{w}(\Omega)$ in particular with $\tilde{p}\leq w<\infty$ (and $\tilde{p}\leq w\leq \infty$ if $\tilde{p}\tilde{m}>d$). This result can be improved by exploiting maximal first order regularity of the superposition operators $(P_+\circ f_{\theta_n,n})(u)-f_{\theta_n,n}(u)$ and $\chi_n(u)$. By \cite[Theorem 1.3]{Florin2022} one can show $W^{1,\gamma}(\Omega)$-regularity for each term with
		\begin{equation}
			\label{gamma_cases}
			\gamma
			\begin{cases}
				=\frac{d\tilde{p}}{d-(\tilde{m}-1)\tilde{p}} &\text{if} ~ \tilde{p}<\frac{d}{\tilde{m}-1}\\
				\in(d,\infty)&\text{else}
			\end{cases}.
		\end{equation}
		As a consequence, we obtain the enhanced choices
		\begin{equation}
			\label{beta_cases}
			\beta 
			\begin{cases}
				= \frac{d\tilde{p}}{2d-(2\tilde{m}-1)\tilde{p}} &\text{if } ~ \tilde{p}<\frac{d}{\tilde{m}}\\
				= d-\epsilon & \text{if} ~ \tilde{p}=\frac{d}{\tilde{m}}\\
				= \frac{d\tilde{p}}{d-(\tilde{m}-1)\tilde{p}} &\text{if } ~ \frac{d}{\tilde{m}}<\tilde{p}<\frac{d}{\tilde{m}-1}\\
				\leq \min(\gamma,v) <\infty &\text{else}
			\end{cases}.
		\end{equation}
		Similarly as before boundedness of \eqref{termsbdd} and \eqref{uestimation} follows under \eqref{beta_cases} and \eqref{gamma_cases}.
	\end{remark}
	\subsection{Continuity property}
	\label{appendix:continuity}
	We proceed with verifying that \[
	\Theta_n^m\times \mathcal{V}^N\ni (\theta_n,v)\mapsto \bar{f}_{\theta_n,n}(v)\in \mathcal{W}
	\] is weakly-weakly continuous, which is the property required in the existence proof in \cite[Appendix C]{morina_holler/online} (cf. \cite[Lemma 43]{morina_holler/online}). This follows from the weak-strong continuity in Proposition~\ref{prop:ws_continuity_mod} shown below. In view of the definition of $\bar{f}_{\theta_n,n}$ in \eqref{def_fmod}, we argue first that
	$$\Theta_n^m\times \mathcal{V}^N\ni (\theta_n,v)\mapsto P_+\circ f_{\theta_n,n}\in L^p(0,T;L^{\hat{q}}(\Omega))$$ is weakly-strongly continuous. Since $\Theta_n^m\times \mathcal{V}^N\ni (\theta_n,v)\mapsto f_{\theta_n,n}\in L^p(0,T;L^{\hat{q}}(\Omega))$ is weakly-strongly continuous due to Assumption~\ref{ass:param_reaction}, this follows by Lipschitz continuity of $P_+$ with constant one, i.e., for $x,y\in \mathbb{R}$ we have $\vert P_+(x)-P_+(y) \vert \leq \vert x-y\vert$. One can also show weak-strong continuity of $\chi_n:\mathcal{V}^N\to L^p(0,T;L^{\hat{p}}(\Omega))$, $u\mapsto \chi^m(u_n)$ using %
	Lipschitz continuity of $\chi_n$ and the Aubin-Lions Lemma \cite[Lemma 7.7]{Roubíček2013}. However, it turns out that we require higher regularity due to the multiplication involved in the definition of $\bar{f}_{\theta_n, n}$. For that, we need to impose higher regularity on the state space $V$, as formulated in Assumption~\ref{ass_init_set}. In the following we write $\chi = \chi^m$ for a fixed $m\in\mathbb{N}$.
	\begin{lemma}
		\label{impr_ws_cont_chi}
		Under Assumption~\ref{ass_init_set} it holds true that $\chi:\mathcal{V}\to L^p(0,T;W^{1,\tilde{p}}(\Omega))$ is weakly-strongly continuous.
	\end{lemma}
	\begin{proof}
		First we note that due to Theorem~\ref{th:mizel} the operators $\chi,\chi':W^{1,\tilde{p}}(\Omega)\to W^{1,\tilde{p}}(\Omega)$ are well-defined. Let now $(u_k)_k\subseteq \mathcal{V}$ such that $u_k\rightharpoonup u\in \mathcal{V}$ as $k\to \infty$. We show that $\chi(u_k)\to \chi(u)$ in $L^p(0,T;W^{1,\tilde{p}}(\Omega))$ as $k\to \infty$. As a consequence of the assumptions $V\hookdoubleheadrightarrow W^{1,\tilde{p}}(\Omega)$ and $\tilde{V}\hookrightarrow W^{1,\tilde{p}}(\Omega)$, there exists a subsequence (w.l.o.g.\ the whole sequence, else apply a subsequence argument) such that $u_k\to u$ in $L^p(0,T;W^{1,\tilde{p}}(\Omega))$ by the Aubin-Lions Lemma \cite[Lemma 7.7]{Roubíček2013}. As $\lim_{k\to\infty}\Vert u_k-u\Vert_{L^{p}(0,T;W^{1,\tilde{p}}(\Omega))}=0$ and
		\begin{multline}
			\label{chi_estimate_1}
			\Vert \chi(u_k)-\chi(u)\Vert_{L^p(0,T;W^{1,\tilde{p}}(\Omega))}^p\leq 2^{p-1}\Vert u_k-u\Vert_{L^p(0,T;L^{\tilde{p}}(\Omega))}\\
			+2^{p-1}\int_0^T\Vert\nabla_x\chi(u_k(t,\cdot))-\nabla_x\chi(u(t,\cdot))\Vert_{L^{\tilde{p}}(\Omega)}^p\dx t
		\end{multline}
		it remains to show that the integral on the right hand side of \eqref{chi_estimate_1} approaches zero as $k\to \infty$. Employing the chain rule, which holds a.e., this term is bounded by
		\begin{multline}
			\label{chi_estimate_2}
			\int_0^T\Vert \chi'(u_k(t,\cdot))\nabla_x u_k(t,\cdot)-\chi'(u(t,\cdot))\nabla_x u(t,\cdot)\Vert_{L^{\tilde{p}}(\Omega)}^p\dx t\\
			\leq \Vert\chi'\Vert_{L^\infty(\mathbb{R})}^p\Vert \nabla_x u_k-\nabla_x u\Vert_{L^p(0,T;L^{\tilde{p}}(\Omega))}\\+\int_0^T\Vert[\chi'(u_k(t,\cdot))-\chi'(u(t,\cdot))]\nabla_x u(t,\cdot)\Vert_{L^{\tilde{p}}(\Omega)}^p\dx t.
		\end{multline}
		The only open point to conclude the assertion of this lemma is to show that the integral on the right hand side of \eqref{chi_estimate_2} converges to zero as $k\to \infty$. For that, we show first that the integrand, which is majorized by the integrable function $t\mapsto 2\Vert \chi'\Vert_{L^\infty(\mathbb{R})}^p\Vert\nabla_xu(t,\cdot)\Vert_{L^{\tilde{p}}(\Omega)}^p,$ approaches zero pointwise in time, finishing the proof due to the Dominated Convergence Theorem. As $u_k\to u$ in $L^p(0,T;W^{1,\tilde{p}}(\Omega))$, it holds true that $u_k(t)\to u(t)$ in $W^{1,\tilde{p}}(\Omega)$ for a.e. $t\in ]0,T[$ and hence, for fixed $t$ that there exists a subsequence (again w.l.o.g.\ the whole sequence due to a subsequence argument) such that $u_k(t)\to u(t)$ pointwise in $\Omega$. Continuity of $\chi'$ implies that $\vert [\chi'(u_k(t,x))-\chi'(u(t,x))]\nabla_x u(t,x)\vert^{\tilde{p}}$ converges to zero for a.e. $x\in \Omega$. As it is majorized by the space integrable function $x\mapsto 2\Vert \chi'\Vert_{L^\infty(\mathbb{R})}^p\vert\nabla_xu(t,x)\vert^{\tilde{p}}$ for a.e. $t\in]0,T[$, dominated convergence yields that the integrand on the right hand side of \eqref{chi_estimate_2} approaches zero as $k\to \infty$ for a.e. $t\in]0,T[$, concluding the assertion by the previous considerations.
	\end{proof}
	
	We argue next weak-strong continuity of $\bar{f}_{\theta_n,n}$ formulated in Proposition~\ref{prop:ws_continuity_mod}.
	\begin{proof}[Proof of Proposition~\ref{prop:ws_continuity_mod}]
		Define $\phi_1: A\to L^p(0,T;W^{1,\tilde{p}}(\Omega))$ and $\phi_2: A\to L^p(0,T;L^{\tilde{p}}(\Omega))$ by
		\[
		\phi_1(a) = \chi_n(v) ~ \text{and} ~ \phi_2(a) = (P_+\circ f_{\theta,n})(v)-f_{\theta,n}(v) ~ \text{for} ~ a = (\theta, v)\in A:=\Theta_n^m\times\mathcal{V}^N.
		\]
		Well-definedness of $\phi_2$ (also considered as a map into $L^p(0,T;L^{\hat{q}}(\Omega))$ together with weak-strong continuity in the latter case) follows by previous considerations right before Lemma~\ref{impr_ws_cont_chi}. Well-definedness and weak-strong continuity of the maps $\phi_1$ follow by Lemma~\ref{impr_ws_cont_chi}. In view of the definition of $\bar{f}_{\theta_n,n}$ in \eqref{def_fmod} the claimed statement follows by weak-strong continuity of $A\ni a\mapsto \phi_1(a)\cdot\phi_2(a)\in L^p(0,T;L^{\hat{q}}(\Omega))$ which is in fact well-defined by Proposition~\ref{prop:f_mod_nemyt}. For that, let $(a_k)_k\subseteq A$ with $a_k\rightharpoonup a$ as $k\to \infty$ for some $a\in A$. Since $\tilde{p}/2\leq \hat{q}\leq \tilde{p}$ with $\frac{1}{\hat{q}}>\frac{2}{\tilde{p}}-\frac{1}{d}$ by Assumption~\ref{ass_init_set}, it holds true that the multiplication operator $\cdot: W^{1,\tilde{p}}(\Omega)\times L^{\tilde{p}}(\Omega)\to L^{\hat{q}}(\Omega)$ is a well-defined continuous bilinear form due to Theorem~\ref{th:behzadan}. Thus, we can estimate with a constant $c>0$ for a.e. $t\in]0,T[$ the term $\Vert \phi_1(a_k(t))\cdot\phi_2(a_k(t))-\phi_1(a(t))\cdot\phi_2(a(t))\Vert_{L^{\hat{q}}(\Omega)}$ using that $\chi$ is a transition function with $\Vert \chi\Vert_{\mathcal{C}(\mathbb{R})}\leq 1$ by
		\begin{align*}
			\Vert\phi_2(a_k(t))-\phi_2(a(t))\Vert_{L^{\hat{q}}(\Omega)}+c\Vert \phi_2(a(t))\Vert_{L^{\tilde{p}}(\Omega)}\Vert \phi_1(a_k(t))-\phi_1(a(t))\Vert_{W^{1,\tilde{p}}(\Omega)}.
		\end{align*}
		Since $\phi_2(a(t))= (P_+\circ f_{\theta,n}-f_{\theta,n})(v(t))$ we derive by \eqref{uniform_state_embedding} that $\vert v(t)\vert\leq c_\mathcal{V}\Vert v\Vert_{\mathcal{V}^N}$ for a.e. $t\in ]0,T[$ together with Lipschitz continuity of $P_+\circ f_{\theta,n}-f_{\theta,n}$ that the term $\Vert \phi_2(a(t))\Vert_{L^{\tilde{p}}(\Omega)}$ is bounded uniformly for a.e. $t\in ]0,T[$. With this and Minkowski's inequality we derive that there exists some $c>0$ such that
		\begin{multline*}
			\Vert \phi_1(a_k)\cdot\phi_2(a_k)-\phi_1(a)\cdot\phi_2(a)\Vert_{L^p(0,T;L^{\hat{q}}(\Omega))}\\
			\leq \Vert\phi_2(a_k)-\phi_2(a)\Vert_{L^p(0,T;L^{\hat{q}}(\Omega))}+c\Vert \phi_1(a_k)-\phi_1(a)\Vert_{L^p(0,T;W^{1,\tilde{p}}(\Omega))}
		\end{multline*}
		which converges to zero as $k\to \infty$ due to weak-strong continuity of $\phi_1,\phi_2$ in the respective spaces, finally concluding the claimed statement.
	\end{proof}
	\subsection{Regularity property}
	\label{appendix:regularity}
	We address next the remaining regularity property of Assumption~\ref{ass:param_reaction} in Proposition~\ref{prop:reg_mod_class}.
	\begin{proof}[Proof of Proposition~\ref{prop:reg_mod_class}]
		For $\bar{f}_{\theta_n,n}\in\overline{\mathcal{F}}_n^m$ it holds by \eqref{def_fmod} for a.e. $u\in \mathbb{R}^{N}$ that
		\begin{equation}
			\label{eq:bar_derivative}
			\begin{aligned}
				\nabla \bar{f}_{\theta_n,n}(u) = -1_{\left\{z\in\mathbb{R}^N:~ f_{\theta_n,n}(z)<0\right\}}(u)&\chi^m(u_n)\nabla f_{\theta_n,n}(u)\\
				&-(P_-\circ f_{\theta_n,n})(u)(\chi^m)'(u_n)e_n+\nabla f_{\theta_n,n}(u)
			\end{aligned}
		\end{equation}
		with $e_n$ the $n$-th unit vector in $\mathbb{R}^N$, and that $\vert\bar{f}_{\theta_n,n}(u)\vert \leq \vert f_{\theta_n,n}(u)\vert$. This implies in particular
		\[
		\vert\nabla \bar{f}_{\theta_n,n}(u)\vert\leq 2\vert\nabla  f_{\theta_n,n}(u)\vert+\Vert (\chi^m)'\Vert_{\mathcal{C}(\mathbb{R})}\vert f_{\theta_n,n}(u)\vert.
		\]
		As a consequence, for any compact $K\subset\mathbb{R}^{N}$ there exists $c>0$ such that
		\[
		\Vert \bar{f}_{\theta_n,n}\Vert_{W^{1,\infty}(K)}\leq c\Vert f_{\theta_n,n}\Vert_{W^{1,\infty}(K)}
		\]
		proving the claimed assertion.
	\end{proof}
	
	\section{Extended numerical experiments}
	\label{app:numerics}
	
	This appendix provides extended experiments including a closer look at the off-manifold behavior of the recovered reaction term and at diffusion identifiability, the weight selection for the penalty-method baseline, different noise levels, a comparison of activation functions, the consideration of non-trivial observation operators, and the regularization-schedule comparison. Throughout we use the setup and the benchmarks of Subsection~\ref{subsec:num_setup}. Unless stated differently, all runs of Section~\ref{sec:numerics} and of this appendix use the constants of Table~\ref{tab:recipe}.\\  %
	
	\begin{table}[b]
		\centering
		\small
		\begin{tabular}{ll}
			\toprule
			component & setting \\
			\midrule
			network architecture          & width $W=64$, depth $D_{\mathrm{net}}=3$, ReLU activation \\
			transition function $\chi^m$  & quintic smoothstep, cut-off level $\varepsilon^m=0.05$ \\
			loss multipliers (constants)  & $\eta_0 = 10^{-6}$, \ $\eta_f = 10^{-6}$, \ $\eta_g = 10^{-3}$ \\
			schedule $\lambda^m,\mu^m$    & logarithmic ramp $50 \to 10^{3}$ over $E$ epochs \\
			schedule $\nu^m$              & logarithmic ramp $10^{-4} \to 10^{-7}$ \\
			optimizer                     & Adam, learning rate $5\cdot 10^{-3}$, warm-up $10^{-2}$ \\
			training stages               & warm-up $E_w = 250$, all-at-once $E = 1500$ \\
			warm-up multipliers           & $\lambda^m = 10^{3}$, \ $\nu^m = 10^{-4}$, \ $(u,u_0)$ frozen at the lift \\
			\bottomrule
		\end{tabular}
		\caption{Constants of the discretized loss $\mathcal{L}^m$ in~\eqref{eq:loss_disc} and of the optimization scheme. Applied to every benchmark unless stated differently.}
		\label{tab:recipe}
	\end{table}
	
	\noindent\textbf{Off-manifold behavior and diffusion identifiability.} The reaction-term errors of Table~\ref{tab:recovery} are taken over the generalization tube. Over the full bounding box $U=\prod_n[\min u_n^\dagger,\max u_n^\dagger]$ they are uniformly larger, in particular for the near-equilibrium benchmarks. The species of B5 start on opposite halves of $\Omega$ and diffuse too slowly to meet on $[0,T]$, so that both of its reaction rates stay almost zero at every time rather than relaxing to equilibrium over time. In either case the data manifold occupies only a small portion of $U$, with $\|f^\dagger\|_{L^2(\mathrm{data})}/\|f^\dagger\|_{L^2(U)}\approx0.07$ for B3 and $\approx3\cdot10^{-6}$ for B5 (against $0.19$-$0.69$ on B1, B2, B4). The resulting box-relative errors of $55\%$ (B3) and $100\%$ (B5) reflect the indeterminacy of $\bar f_\theta$ \emph{off} the data, where the loss carries no data-driven information and the regularizers take over and are not a failure of the recovery on the states the data visits. The gap between the tube and box metrics quantifies how far the learned reaction term may be trusted away from the data.
	
	The diffusion coefficients reflect a corresponding identifiability structure. Six of the eleven coefficients are recovered to $3$-$30\%$ relative error, whereas the remaining five (the predator species of B2, both components of B3, the recovered compartment of B4 and the product species $C$ of B5) are systematically underestimated, at $39$-$98\%$. Two comparisons confirm this is structural rather than optimization variance. First, the per-seed spread of the stuck cells is small against their distance from the well-determined group. B4's recovered compartment ranges over $49$-$63\%$ across the ten seeds of Table~\ref{tab:recovery}, against $3$-$30\%$ for the diagonally coupled species. Second, doubling the network capacity to $W{=}128$, $D_{\mathrm{net}}{=}4$ does not displace them (B3 species $0$: $60\to50\%$, B4 species $2$: $52\to67\%$ over $n=3$ seeds, with the per-seed ranges overlapping in both cases), and neither does doubling the initial-condition diversity from $L=4$ to $L=8$, reported below. For these species the contribution of $d_n\Delta u_n$ to the PDE residual is small against that of the reaction term on the states the data visits, so the residual is insensitive to $d_n$ and the data carries little information.\\

	\noindent\textbf{Determining $D$ and $\bar f_\theta$ jointly.} The errors discussed above are a property of what the data determine, not of how \eqref{eq:all_at_once_mod} is solved. Near a state where $\partial_t u\approx0$, the PDE residual constrains only the combination $D\Delta u\approx-\bar f_\theta(u)$, and scaling $D$ and $\bar f_\theta$ by a common factor leaves it unchanged. The data fix the pair only up to that factor. By halving the time horizon on B3, from $T=30$ to $T=15$, so that the growth phase rather than the settled Turing pattern dominates the samples, the reaction-term error improves, from $25\%$ to $15\%$ on the data manifold and from $21\%$ to $14\%$ on the generalization tube, although the per-seed ranges still overlap ($20$--$40\%$ against $12$--$21\%$ on the manifold). The diffusion coefficients worsen from $58\%$ and $37\%$ to $98\%$ and $80\%$ relative error.%

	The reaction-term ratio on B5 is a property of its initial condition rather than of the recovery. The default IC places $A$ and $B$ as Gaussians of width $0.08L_\Omega$ at $0.30L_\Omega$ and $0.70L_\Omega$, a separation of $0.4L_\Omega=8$ against a diffusion length $\sqrt{2d^\dagger T}\approx0.55$, so the reactants never meet on $[0,T]$ and $\|f^\dagger\|_{L^2(\mathrm{data})}\approx3\cdot10^{-6} \|f^\dagger\|_{L^2(U)}$. Widening them to $0.15L_\Omega$ and moving them to $0.40L_\Omega$ and $0.60L_\Omega$, so that both are $\approx0.51$ at the midpoint and $k_1AB\approx0.26$ from $t=0$ on, raises that ratio to $7\cdot10^{-2}$, comparable to the $0.07$ of B3, and the reaction term is recovered to $14\%$ on the data manifold. The diffusion errors move the other way, from $10\%\,|\,12\%\,|\,98\%$ to $84\%\,|\,90\%\,|\,98\%$. Lengthening the horizon is not a remedy here since reactants dilute as they mix, and $\|f^\dagger\|_{L^\infty(\mathrm{data})}$ saturates at $1.3\cdot10^{-4}$ for $T\geq200$.

	Note that the indeterminacy is not a matter of data quantity as doubling the number of initial conditions from $L=4$ to $L=8$ leaves the tube error within one percentage point on B2 and B4 and within the seed spread on B3. The extreme case is the product species of B5, whose diffusion is essentially unidentified ($98\%$ off). Its initial condition places $C(0)=0$, so $C$ enters only via the forward reaction $A+B\to C$ and balances at a nearly uniform profile. The term $d_C\Delta C$ in the PDE residual then lies orders of magnitude below $f_C(u)$, and an expressive $f_\theta$ can absorb the residual diffusion signal, leaving $d_C$ under-determined.\\
	
	\noindent\textbf{Penalty weight for the comparison of Subsection~\ref{subsec:num_comparison}.} The penalty baseline of Table~\ref{tab:comparison} requires a weight $\eta_{\mathrm{qp}}$ in \eqref{eq:qpos_penalty}, which we select over $\eta_{\mathrm{qp}}\in\{1,10,10^2,10^3\}$ on B4, the benchmark whose unconstrained violation is largest relative to its reaction-rate magnitude, with the framework otherwise unchanged (Table~\ref{tab:soft_pilot}). The state error is almost unchanged across the four weights, and $\mathcal{Q}$ is smallest at $\eta_{\mathrm{qp}}=10$ and about ten times larger for $\eta_{\mathrm{qp}}\in\{1,10^3\}$.\\

	\begin{table}[t]
		\centering
		\small
		\begin{tabular}{lcccc}
			\toprule
			& $n$ & rel.\ $f$ err.\ & state err.\ & $\mathcal{Q}(\,\cdot\,)$ \\
			\midrule
			$\eta_{\mathrm{qp}}=1$      & $2$ & $41\%$ & $0.0012$ & $0.134$ \\
			$\eta_{\mathrm{qp}}=10$     & $5$ & $38\%$ & $0.0012$ & $\mathbf{0.015}$ \\
			$\eta_{\mathrm{qp}}=10^2$   & $5$ & $36\%$ & $0.0015$ & $0.072$ \\
			$\eta_{\mathrm{qp}}=10^3$   & $2$ & $32\%$ & $0.0015$ & $0.148$ \\
			\midrule
			unconstrained               & $10$ & $39\%$ & $0.0013$ & $0.536$ \\
			constrained                 & $10$ & $39\%$ & $0.0018$ & $0$ \\
			\bottomrule
		\end{tabular}
		\caption{Penalty-weight on B4, medians over $n$ seeds, with the two baselines of Table~\ref{tab:comparison} for reference. The reaction-term error is relative to $\|f^\dagger\|_{L^2(U)}$ over the bounding box $U$. $\mathcal{Q}$ attains its minimum at the interior weight $\eta_{\mathrm{qp}}=10$, which is the value used in Table~\ref{tab:comparison}. The $\eta_{\mathrm{qp}}=10$ row is the penalty baseline of Table~\ref{tab:comparison} restricted to the first five of its ten seeds, so the two reported $\mathcal{Q}$ values, $0.015$ here and $0.043$ there, differ by the seed count alone.}
		\label{tab:soft_pilot}
	\end{table}

	\noindent\textbf{Noise sensitivity.} The diffusion recovery is sensitive to the noise level. Testing different noise levels on B2 and B4 ($n=3$) shows the reaction-term recovery worsening (B4 $L^2(U)$: $39\to42\to95\%$ at $\sigma\in\{10^{-3},10^{-2},10^{-1}\}$) but the diffusion errors rising sharply once $\sigma\gtrsim10^{-2}$ (B2 species $0$: $30\to72\to89\%$), because the finite-difference Laplacian amplifies measurement noise by $\mathcal{O}(\mathrm{d}x^{-2})$ and the small diffusive signal is the first to be lost. Noise-robust differentiation (spectral or weak-form residuals) is a natural extension within the same loss.\\
	
	\noindent\textbf{Activation functions.} We repeat the recovery experiment with three activation functions: ReLU, ReQU and $\tanh$ (Table~\ref{tab:activation}). All three attain $\mathcal{Q}=0$ and show the same off-manifold limitation on B3 and B5. In the reaction-term error ReQU is ahead of ReLU on B1, B3 and B4, most clearly on B4 ($27\%$ against $39\%$), behind it on B2 ($14\%$ against $11\%$) and tied on B5, while $\tanh$ is the poorest on every benchmark (e.g.\ $16\%$ against $8$-$9\%$ on B1) and the state error as well ($0.45\%$ against $0.21\%$ on B1). The ReQU rows are reported as a reference bound rather than as a competing activation as it is only locally Lipschitz continuous, and a network of depth $D_{\mathrm{net}}$ built from it grows like $\Vert u\Vert^{2^{D_{\mathrm{net}}}}$, so at $D_{\mathrm{net}}=3$ it violates both the quadratic growth \eqref{cond:growth} and the global Lipschitz hypothesis of Lemma~\ref{lem:auxiliary_result} and Theorems~\ref{thm:wellposedness}, \ref{thm:main_result} (cf.\ Remark~\ref{rem:local_lipschitz}).\\
	
	\begin{table}[t]
		\centering
		\small
		\begin{tabular}{lcccc}
			\toprule
			benchmark & arch.\ & $n$ &
			$\|f^\dagger-\bar f_\theta\|_{L^2(U)}/\|f^\dagger\|_{L^2(U)}$ &
			state err.\ \\
			\midrule
			B1 (Fisher-KPP)     & ReLU & $10$ & $9\%$   & $0.21\%$\\
			& ReQU & $5$ & $8\%$   & $0.17\%$\\
			& tanh & $5$ & $16\%$  & $0.45\%$\\
			\midrule
			B2 (Lotka-Volterra) & ReLU & $10$ & $11\%$  & $0.19\%$\\
			& ReQU & $5$ & $14\%$  & $0.17\%$\\
			& tanh & $5$ & $25\%$  & $0.34\%$\\
			\midrule
			B3 (Schnakenberg)    & ReLU & $10$ & $55\%$  & $0.56\%$\\
			& ReQU & $5$ & $42\%$  & $0.48\%$\\
			& tanh & $5$ & $77\%$  & $0.49\%$\\
			\midrule
			B4 (SIR)             & ReLU & $10$ & $39\%$  & $0.18\%$\\
			& ReQU & $5$ & $27\%$  & $0.18\%$\\
			& tanh & $5$ & $48\%$  & $0.28\%$\\
			\midrule
			B5 (mass-action)     & ReLU & $10$ & $100\%$ & $0.20\%$\\
			& ReQU & $5$ & $100\%$ & $0.20\%$\\
			& tanh & $5$ & $100\%$ & $0.19\%$\\
			\bottomrule
		\end{tabular}
		\caption{Recovery with ReLU, ReQU and a $\tanh$ baseline, medians over $n$ seeds, reaction-term error relative to $\|f^\dagger\|_{L^2(U)}$. The B5 rows exceed $80\%$ because its reactants are initialized too far apart to meet on $[0,T]$, the separation being $0.4L_\Omega=8$ against a diffusion length $\sqrt{2d^\dagger T}\approx0.55$, so that $\|f^\dagger\|_{L^2(\mathrm{data})}\approx0$ and the box-relative comparison is uninformative. The ReQU rows violate \eqref{cond:growth} and the global Lipschitz hypothesis of Theorem~\ref{thm:main_result} and are included as a reference bound only.}
		\label{tab:activation}
	\end{table}
	
	\noindent\textbf{Non-trivial observation operators.} To verify consistent behavior under non-trivial $K^m$, we repeat the recovery experiment on B4 with three operators (Table~\ref{tab:observation}), full state observation, stride-$2$ spatial subsampling $K_2u(t,x_i)=u(t,x_{2i})$, and the Fourier truncation $P_{16}$ keeping the first $16$ low-frequency modes. For each operator the warm-up state of Subsection~\ref{subsec:num_setup} is initialized by applying a right inverse of $K^m$ to the data, i.e., cubic spline interpolation onto the spatial grid for $K_2$ and $P_{16}$ itself for $P_{16}$, so that $K^m u_{\mathrm{init}} = y^m$ holds to machine precision. The quasipositivity violation remains zero throughout. The reaction-term error is essentially unaffected by the Fourier truncation and grows under stride-$2$ subsampling, while the state error is the more sensitive quantity, rising from $0.19\%$ under full observation to $0.62\%$ and $3.17\%$ respectively. The diffusion coefficients worsen most, $d_S$ and $d_I$ are recovered to within $26\%$ and $11\%$ under full observation but only to $44\%$ and $77\%$ under subsampling, and $d_R$ is poorly identified under all three operators, consistent with the identifiability structure discussed above. The recovered compartment of B4 is produced purely by the reaction term $\gamma I$, so that $d_R\Delta R$ contributes little to the PDE residual.\\
	
	\begin{table}[t]
		\centering
		\small
		\setlength{\tabcolsep}{4pt}
		\begin{tabular}{@{}lccccc@{}}
			\toprule
			$K^m$ &
			$\dfrac{\|f^\dagger-\bar f_\theta\|_{L^2(U_{\mathrm{tube}})}}{\|f^\dagger\|_{L^2(U_{\mathrm{tube}})}}$ &
			state err.\ &
			$|D^m_S-d^\dagger_S|/d^\dagger_S$ &
			$|D^m_I-d^\dagger_I|/d^\dagger_I$ &
			$|D^m_R-d^\dagger_R|/d^\dagger_R$ \\
			\midrule
			$\mathrm{Id}$   & $16\%$ & $0.19\%$ & $26\%$ & $11\%$ & $51\%$ \\
			$K_2$           & $21\%$ & $3.17\%$ & $44\%$ & $77\%$ & $52\%$ \\
			$P_{16}$        & $14\%$ & $0.62\%$ & $36\%$ & $18\%$ & $45\%$ \\
			\bottomrule
		\end{tabular}
		\caption{Recovery on B4 under three observation operators, medians over $n=5$ seeds. Here $K_2$ is the stride-$2$ subsampling and $P_{16}$ the Fourier truncation to $16$ modes. In each case the warm-up state is initialized by applying a right inverse of $K^m$ to the measurements. $\mathcal{Q}(\bar f_\theta)=0$ under all three by construction, the reaction-term error is relative to $\|f^\dagger\|_{L^2(U_{\mathrm{tube}})}$ over the generalization tube of Subsection~\ref{subsec:num_setup}.}
		\label{tab:observation}
	\end{table}
	
	\noindent\textbf{Regularization schedules.} Theorem~\ref{thm:main_result} requires the regularization parameters to satisfy $\lambda^m,\mu^m\to\infty$ and $\nu^m\to0$ together with the coupling conditions $\lambda^m m^{-\min(\alpha\gamma,\beta)q}=o(1)$, $\mu^m\delta(m)^r=o(1)$ and $\nu^m\psi(m)=o(1)$, where $(\delta(m))_m$ is the zero sequence bounding the data error in \eqref{noise_estimation}. The experiments of Section~\ref{sec:numerics} use constants and finite $m$. To satisfy these conditions in an experiment, we tie the noise level to the schedule on B2 (here and in the fixed-noise comparison below with $E_w=200$, $E=600$), with $\sigma(m)=10^{-3}\cdot 50/m$ and $\lambda^m=\mu^m=200\sqrt{m}$ for $m\in\{50,100,200,400\}$, so that $\mu^m\delta(m)^r\sim m^{-3/2}=o(1)$. Here $m$ enters through the width $W(m)\approx16\sqrt{m}$ and the cut-off $\varepsilon^m=0.5\,m^{-1/2}$. The training remains stable and the relative state error improves from $0.24\%$ to $0.19\%$ as the noise vanishes, while the median absolute reaction-term error $\|f^\dagger-\bar f_\theta\|_{L^2(U)}$ stays between $0.082$ and $0.098$. At a fixed noise level, constant and ramped schedules for $\lambda^m,\mu^m$, growing by up to two orders of magnitude, give the same results within the variation over the seeds at $m=50$ and $m=200$, so the constant choice of Section~\ref{sec:numerics} does not worsen the results.\\
	
	\section{Quasipositive functions}
	\label{sec:approx_quaspos}
	In the following section we develop approximation results for continuous and quasipositive functions, a property which needs to be imposed on the reaction terms for well-posedness of the underlying RD system (see condition \eqref{cond:quasipos} in Section~\ref{sec:conformal_classes}).
	\begin{definition}[Quasipositive function]
		\label{def:quasipositive}
		Let $N\in \mathbb{N}$ and $F=(F_n)_{n=1}^N\in \mathcal{C}(\Omega, \mathbb{R}^N)$ with $\Omega\subseteq[0,\infty[^N=\{(x_1,\dots,x_N)\in\mathbb{R}^N|~x_n\geq 0,~ 1\leq n\leq N\}$ a measurable set such that $F$ is continuously extendable to the closure $\overline{\Omega}$. We call $F$ quasipositive if
		\[
		F_n(x_1,\dots,x_{n-1},0,x_{n+1},\dots,x_N)\geq 0
		\]
		for $(x_1,\dots,x_{n-1},0,x_{n+1},\dots,x_N)\in\overline{\Omega}$ and $1\leq n\leq N$. We call $F$ strongly quasipositive if $F_n(x)\geq 0$ for $x\in\partial[0,\infty[^N\cap\partial\Omega$. We further call a continuous function $f\in \mathcal{C}(\overline{\Omega}, \mathbb{R})$ strongly quasipositive if $f(x)\geq 0$ for all $x\in\partial[0,\infty[^N\cap\partial\Omega$.
	\end{definition}
	We consider first approximation results for real-valued strongly quasipositive functions. The generalization to quasipositive functions in Subsection~\ref{subsec:generalization_quasipositive} is straightforward. Throughout the following considerations assume w.l.o.g.\ $\partial[0,\infty[^N\cap\partial\Omega\neq \emptyset$. Furthermore, let $\mathcal{D}:\mathbb{R}^N\times\mathbb{R}^N\to [0,\infty[$ be some norm-induced metric, where
	\[
	\mathcal{D}(x,A):=\inf_{a\in A}\mathcal{D}(x,a)
	\]
	for a set $A\subseteq \mathbb{R}^N$ and $x\in \mathbb{R}^N$ with the infimum being defined as infinity if $A=\emptyset$.
	\subsection{Approximation of strongly quasipositive functions}
	The basis of the subsequent results is the following modification technique. For that, note that by $\chi_A$ we denote the characteristic function on a measurable set $A\subseteq \mathbb{R}^N$, attaining the value $1$ on $A$ and vanishing elsewhere. We will also require continuous modifications of certain characteristic functions as introduced next.
	\begin{definition}
		\label{def:modification}
		Let $f\in \mathcal{C}(\Omega, \mathbb{R})$ for measurable $\Omega\subseteq[0,\infty[^N$ and define for $\epsilon>0$
		\[
		\Gamma_\epsilon:=\{x\in \overline{\Omega}~|~\mathcal{D}(x,\partial[0,\infty[^N\cap\partial\Omega)<\epsilon\}.
		\]
		Given $0<\delta<\epsilon$ and $\Omega_\epsilon = \overline{\Omega}\backslash \Gamma_\epsilon$ let $\chi_{\Gamma_\epsilon}^\delta\in\mathcal{C}(\overline{\Omega},\mathbb{R})$ with $0\leq \chi_{\Gamma_{\epsilon}}^{\delta}\leq 1$, 
		\[
		\chi_{\Gamma_{\epsilon}}^{\delta}\big|_{\Gamma_{\epsilon-\delta}}\equiv 1 ~ ~ ~ \text{and} ~ ~ ~  \chi_{\Gamma_{\epsilon}}^{\delta}\big|_{\Omega_{\epsilon+\delta}}\equiv 0.
		\]
		Then we define, implicitly depending on the concrete form of $\chi_{\Gamma_\epsilon}^\delta$, the function
		\[
		f_{\epsilon,\delta}=(P_+\circ f)\cdot\chi_{\Gamma_\epsilon}^\delta+f\cdot(1-\chi_{\Gamma_\epsilon}^\delta)
		\]
		where $P_+(z)=\max(z,0)$ for $z\in\mathbb{R}$.
	\end{definition}
	\begin{remark}
		\label{rem:transition_example}
		An example of a function $\chi_{\Gamma_\epsilon}^\delta$ for $0<\delta<\epsilon$ is as follows. Consider
		\begin{align*}
			h_\delta: ~\mathbb{R}^N\to \mathbb{R}, ~ 
			x\mapsto\begin{cases}
				c\exp((\Vert x\Vert_2^2-\delta^2)^{-1}), &\text{if} ~ \Vert x\Vert_2\leq \delta\\
				0, &\text{otherwise}
			\end{cases}
		\end{align*}
		where $c=c(N,\delta)>0$ such that $\int_{\mathbb{R}^N} h_\delta(x)\dx x=1$. Note that $h_\delta\in \mathcal{C}^\infty(\mathbb{R}^N)$ is compactly supported in the Euclidean ball in $\mathbb{R}^N$ with radius $\delta>0$. The convolution $\chi_{\Gamma_\epsilon}^\delta = \chi_{\Gamma_\epsilon}*h_\delta$ fulfills the requirements in Definition~\ref{def:modification}. See also Definition~\ref{def:transition} and Subsection~\ref{subsec:consistency} for one-dimensional examples.
	\end{remark}
	Based on the modifications in Definition~\ref{def:modification} we prove in the following $L^p$-convergence results, starting with the uniform case for $p=\infty$. 
	\begin{theorem}
		\label{mult_inf_est}
		Let $N\in \mathbb{N}$, $\Omega\subseteq[0,\infty[^N$ be measurable and $f\in \mathcal{C}(\Omega, \mathbb{R})$. Suppose that $f$ is uniformly continuous on $\Gamma_s$ for some $s>0$, that $f$ is strongly quasipositive and $(f_n)_n\subseteq \mathcal{C}(\overline{\Omega},\mathbb{R})$ approximates $f$ uniformly. Given two positive and monotone zero sequences $(\epsilon_n)_n, (\delta_n)_n$ with $0<\delta_n<\epsilon_n$ it holds that $(f_{n,\epsilon_n,\delta_n})_n$ is a sequence of continuous and strongly quasipositive functions approximating $f$ uniformly.
	\end{theorem}
	\begin{proof}
		Certainly, continuity of the $f_{n,\epsilon_n, \delta_n}$ follows as for continuous functions $g:\Omega\to\mathbb{R}$ the modifications $g_{\epsilon,\delta} = (P_+\circ g)\cdot\chi_{\Gamma_{\epsilon}}^\delta+g\cdot(1-\chi_{\Gamma_{\epsilon}}^\delta)$ are continuous for $0<\delta<\epsilon$ by continuity of $P_+$, $g$ and $\chi_{\Gamma_{\epsilon}}^\delta$. Due to the estimation
		\[
		g_{\epsilon,\delta}(x) = (P_+\circ g)(x)\chi_{\Gamma_{\epsilon}}^\delta(x)=P_+(g(x))\geq 0
		\]
		for $x\in \partial[0,\infty[^N\cap\partial\Omega$ we derive strong quasipositivity of the $f_{n,\epsilon_n,\delta_n}$. Finally, we verify that $f_{n,\epsilon_n,\delta_n}$ converges to $f$ uniformly in $\Omega$ as $n\to \infty$. As $f = f\chi_{\Gamma_{\epsilon_n}}^{\delta_n}+f(1-\chi_{\Gamma_{\epsilon_n}}^{\delta_n})$ the triangle inequality, $0\leq\chi_{\Gamma_{\epsilon_n}}^{\delta_n}\leq 1$ and $\supp(\chi_{\Gamma_{\epsilon_n}}^{\delta_n})\subseteq \overline{\Gamma_{\epsilon_n+\delta_n}}$ imply that
		\begin{align*}
			\Vert f_{n,\epsilon_n,\delta_n}-f\Vert_{L^\infty(\Omega)}
			&\leq \Vert (P_+\circ f_n-f)\chi_{\Gamma_{\epsilon_n}}^{\delta_n}\Vert_{L^\infty(\Omega)}+\Vert (f_n-f)(1-\chi_{\Gamma_{\epsilon_n}}^{\delta_n})\Vert_{L^\infty(\Omega)}\\
			&\leq \Vert P_+ \circ f_n -f\Vert_{L^\infty(\Gamma_{\epsilon_n+\delta_n})}+\Vert f_n-f\Vert_{L^\infty(\Omega)}.
		\end{align*}
		The second term approaches zero by assumption. The first term is bounded by
		\begin{align*}
			\Vert P_+\circ f_n -f\Vert_{L^\infty(\Gamma_{\epsilon_n+\delta_n})}&\leq \Vert P_+ \circ f_n -P_+\circ f\Vert_{L^\infty(\Gamma_{\epsilon_n+\delta_n})}+\Vert P_+\circ f -f\Vert_{L^\infty(\Gamma_{\epsilon_n+\delta_n})}\\
			&= \Vert P_+ \circ (f_n - f)\Vert_{L^\infty(\Gamma_{\epsilon_n+\delta_n})}+\Vert P_+\circ f -f\Vert_{L^\infty(\Gamma_{\epsilon_n+\delta_n})}\\
			&\leq \Vert f_n -f\Vert_{L^\infty(\Gamma_{\epsilon_n+\delta_n})}+\Vert P_+\circ f -f\Vert_{L^\infty(\Gamma_{\epsilon_n+\delta_n})}
		\end{align*}
		where we have used that $\vert P_+(z)\vert \leq \vert z\vert$ for $z\in\mathbb{R}$. Thus, it suffices to show that for all $\epsilon>0$ there exists some $m\in\mathbb{N}$ such that for $n\geq m$
		\begin{align}
			\label{remaining_uniform}
			\vert P_+ (f(x))-f(x)\vert <\epsilon
		\end{align}
		for all $x\in \Gamma_{\epsilon_n+\delta_n}$. For fixed $\epsilon>0$ we choose $m$ large enough such that for $n\geq m$ it holds true that $\vert f(x)-f(y)\vert < \epsilon$ for $x,y\in \Gamma_{\epsilon_n+\delta_n}$ with $\mathcal{D}(x,y)<\epsilon_n+\delta_n$. This is possible due to uniform continuity of $f$ on a sufficiently small boundary section of $\Omega$ by assumption. In fact, for $x\in \Gamma_{\epsilon_n+\delta_n}$ with $f(x)\geq 0$ it holds $P_+(f(x))=f(x)$ implying \eqref{remaining_uniform}. In case $f(x)<0$ choose $y\in \partial[0,\infty[^N\cap\partial\Omega$ such that $\mathcal{D}(x,y)<\epsilon_n+\delta_n$. Then since $f(y)\geq 0$ we have
		\[
		0>f(x)\geq f(x)-f(y)\geq -\vert f(x)-f(y)\vert >-\epsilon 
		\]
		implying again \eqref{remaining_uniform}. Thus, the $f_{n,\epsilon_n, \delta_n}$ converge to $f$ uniformly in $\Omega$ as $n\to\infty$.
	\end{proof}
	We acquire the following result on basis of the previous considerations where it is worth noting that $L^p$-convergence for $1\leq p<\infty$ may be derived even if $f$ is not uniformly continuous on some sufficiently small boundary section. In addition, we take the subsequent assumption for granted in view of the next result.
	\begin{assumption}
		\label{assump}
		\normalfont There exist $s>0$ and a function $\varphi:[0,s[\to\mathbb{R}$, which is right-continuous in zero with $ \varphi(0)=0$, fulfilling $\vert \Gamma_x\vert \leq \varphi(x)$ for $0\leq x<s$.
	\end{assumption}
	\begin{theorem}
		\label{mult_p_est}
		Let $N\in \mathbb{N}$, $\Omega\subseteq[0,\infty[^N$ be measurable and $f\in \mathcal{C}(\Omega, \mathbb{R})$ with $\Vert f\Vert_{L^p(\Omega)}<\infty$. Furthermore, suppose that $f$ is uniformly bounded on some $\Gamma_s$ with $\vert\Gamma_s\vert<\infty$ and let Assumption~\ref{assump} hold true. Assume that $f$ is strongly quasipositive and $(f_n)_n\subseteq \mathcal{C}(\overline{\Omega},\mathbb{R})$ approximate $f$ in $L^p(\Omega)$. Given two positive and monotone zero sequences $(\epsilon_n)_n, (\delta_n)_n$ with $0<\delta_n<\epsilon_n$ it holds that $(f_{n,\epsilon_n,\delta_n})_n$ is a sequence of continuous and strongly quasipositive functions approximating $f$ in $L^p(\Omega)$.
	\end{theorem}
	\begin{proof}
		We have similarly as in Theorem~\ref{mult_inf_est} that
		\begin{align*}
			\Vert f_{n,\epsilon_n,\delta_n}-f\Vert_{L^p(\Omega)} & \leq \Vert (P_+\circ f_n-f)\chi_{\Gamma_{\epsilon_n}}^{\delta_n}\Vert_{L^p(\Omega)}+\Vert (f_n-f)(1-\chi_{\Gamma_{\epsilon_n}}^{\delta_n})\Vert_{L^p(\Omega)}\\
			&\leq \Vert P_+\circ f_n-f\Vert_{L^p(\Gamma_{\epsilon_n+\delta_n})}+\Vert f_n-f\Vert_{L^p(\Omega)}
		\end{align*}
		The second term approaches zero by assumption. As the first term is bounded by
		\begin{align*}
			\Vert P_+\circ f_n-f\Vert_{L^p(\Gamma_{\epsilon_n+\delta_n})}&\leq \Vert P_+\circ f_n-P_+\circ f\Vert_{L^p(\Gamma_{\epsilon_n+\delta_n})}+\Vert P_+\circ f-f\Vert_{L^p(\Gamma_{\epsilon_n+\delta_n})}\\
			&\leq \Vert f_n-f\Vert_{L^p(\Gamma_{\epsilon_n+\delta_n})}+\Vert f\Vert_{L^\infty(\Gamma_{\epsilon_n+\delta_n})}\Vert\chi_{[f<0]\cap\Gamma_{\epsilon_n+\delta_n}}\Vert_{L^p(\Omega)}\\
			&\leq \Vert f_n-f\Vert_{L^p(\Omega)}+\Vert f\Vert_{L^\infty(\Gamma_{\epsilon_n+\delta_n})}\vert \Gamma_{\epsilon_n+\delta_n}\vert^{1/p}
		\end{align*}
		it converges to zero as $n\to\infty$ by the convergence of $f_n$ to $f$ in $L^p(\Omega)$, uniform boundedness of $f$ on $\Gamma_s$ and Assumption~\ref{assump}. Thus, we derive convergence of $(f_{n,\epsilon_n, \delta_n})_n$ to $f$ in $L^p(\Omega)$.
	\end{proof}
	Next we verify Assumption~\ref{assump} for $s$ given as in Theorem~\ref{mult_p_est}, together with further properties of an exemplary choice of $\varphi$.
	
	\begin{lemma}
		\label{assproof}
		Let $\Omega$ be a measurable set and assume there exists some $s>0$ such that $\vert \Gamma_s\vert<\infty$. Then the function $\varphi:[0,s[\to\mathbb{R}$ defined for $0\leq x<s$ by $\varphi(x):=\vert \Gamma_x\vert$ is increasing, left-continuous in $]0,s[$, right-continuous in $0$ and fulfills $\varphi(0)=0$.
	\end{lemma}
	\begin{proof}
		As for $0\leq x<s$ the measurable sets $\Gamma_x\subseteq \Gamma_s$ have finite measure and $\vert\Gamma_x\vert\geq 0$ it holds true that $\varphi(x)$ is well-defined for all $0\leq x<s$. Note that measurability follows directly by continuity of the distance function to the boundary $\partial[0,\infty[^N\cap\partial\Omega$ and the fact that $\Gamma_x$ is a corresponding reduced level set. In case $x=0$ we have $\varphi(0)=\vert\Gamma_0\vert=\vert\emptyset\vert=0$. We prove next monotonicity and the continuity properties. For $0\leq x<y<s$ it holds
		\[
		\varphi(x)=\vert\Gamma_x\vert\leq \vert\Gamma_y\vert=\varphi(y).
		\]
		since $\Gamma_x\subseteq \Gamma_y$ by definition. To show left-continuity it suffices to show that for all $0< x_0<s$ 
		\[
		\lim_{x\to x_0^-}\vert\Gamma_x\vert = \vert\Gamma_{x_0}\vert. 
		\]
		This follows by the continuity property of the measure as the sets are increasing
		\[
		\lim_{x\to x_0^-}\vert\Gamma_x\vert =\bigg\vert\bigcup_{x<x_0}\Gamma_x\bigg\vert =\vert \Gamma_{x_0}\vert.
		\]
		The argument for the last equality is that for $z\in \Gamma_{x_0}$ it holds $\mathcal{D}(z,\partial[0,\infty[^N\cap\partial\Omega)<x_0$ and hence, also $\mathcal{D}(z,\partial[0,\infty[^N\cap\partial\Omega)<x$ for some $x<x_0$. Right-continuity in zero follows by $\vert \Gamma_{0}\vert=0$ and the chain of equalities
		\[
		\lim_{x\to 0^+}\vert\Gamma_x\vert =\big\vert\bigcap_{0<x<s}\Gamma_x\big\vert=\vert \{z\in\overline{\Omega}: \mathcal{D}(z,\partial[0,\infty[^N\cap\partial\Omega)\leq 0\}\vert=\vert\partial[0,\infty[^N\cap\partial\Omega\vert=0.
		\]
		Note that the second identity follows since for $z$ with $\mathcal{D}(z,\partial[0,\infty[^N\cap\partial\Omega)<x$ for all $x>0$ it necessarily holds true that $\mathcal{D}(z,\partial[0,\infty[^N\cap\partial\Omega)\leq 0$. Further note that for right-continuity we required the finite measure property of $\vert\Gamma_s\vert$.
	\end{proof}
	\begin{example}
		In the setup of the previous result consider $\Omega = [0,1]^N$. %
		Then
		\[
		\varphi(x)= \vert \Gamma_x\vert = 1-(1-\min(x,1))^N\quad \text{for} ~ x\in [0,\infty[.
		\]
	\end{example}
	\subsection{Extension to general domains}
	\label{subsec:ext_gen_dom}
	Motivated by the study of general domains that may attain infinite measure, such as \(\Omega = [0,\infty[^N\), we aim to extend the approximation results discussed above to such settings. Concerning uniform convergence, recall that no restriction on the domain was required, consequently, this result remains valid for domains of infinite measure. For \(L^p\)-convergence with \(1 \leq p < \infty\), however, we assumed the existence of some \(s > 0\) such that the boundary layer \(\Gamma_s\) has finite measure, a condition weaker than requiring the entire domain to have finite measure. Difficulties arise for domains such as \(\Omega = [0,\infty[^N\), where the sets \(\Gamma_\epsilon\) defined in Theorem~\ref{mult_p_est} are constructed in a way that is essentially uniform with respect to the underlying distance topology and therefore fail to have finite measure. To overcome this issue, we redefine the sets \(\Gamma_\epsilon\) so that the resulting boundary layers approximate \(\partial [0,\infty[^N \cap \partial \Omega\) more accurately, with diminishing error for points increasingly far from the origin.
	
	In view of Lemma~\ref{assproof}, for the derived results to hold, it is necessary that the class of sectional boundaries $(\Gamma_x)_{0\leq x<s}$ for some $s>0$ fulfill the following conditions:
	\begin{tabbing}
		\hspace{0.2cm}\=\hspace{3.3cm}\=\kill \>\textit{(Monotonicity)}\>for $0\leq x<y<s$ it holds $\Gamma_x\subseteq\Gamma_y$\\
		\>\textit{(Consistency)}\> the inclusion $\Gamma_0\subseteq\partial[0,\infty[^N\cap\partial\Omega \subseteq \Gamma_x$ for $0<x<s$ is valid\\
		\>\textit{(Finiteness)}\> we have $\vert \Gamma_s\vert <\infty$\\
		\>\textit{(Continuity)}\> it holds $\vert\bigcap_{0<x<s}\Gamma_x\vert = 0$
	\end{tabbing}
	Note that these properties are essential to obtain right-continuity of the function $x\mapsto\vert\Gamma_x\vert$ in zero. The monotonicity and consistency conditions are natural requirements when characterizing an approximation of the boundary $\partial[0,\infty[^N\cap\partial\Omega$. For that, let $$f:[0,\infty[\to[0,\infty[\quad \text{with} \quad f(x)=
	\begin{cases}
		\sqrt{x}, & 0\leq x<1\\
		x^2, & 1\leq x
	\end{cases}.$$
	Define further for $\epsilon>0$ the implicit nonlinear boundary sections 
	\[
	\Gamma_\epsilon:= \{x\in\overline{\Omega} ~ \vert ~ \prod_{n=1}^N f(x_n)\leq \epsilon\}.
	\]
	Next we show that the class $(\Gamma_\epsilon)_{0\leq \epsilon<1}$ fulfills the previously discussed properties:\\
	
	\noindent\textit{(Monotonicity):} We have for $0\leq \epsilon<\delta<1$ and arbitrary $x\in \Gamma_\epsilon$ that $\prod_{n=1}^N f(x_n)\leq \epsilon<\delta$ and thus, the inclusion $x\in \Gamma_\delta$.\\
	
	\noindent\textit{(Consistency):} For $x\in\overline{\Omega}$ such that $x_m=0$ for some $1\leq m\leq N$ it immediately follows that $\prod_{n=1}^N f(x_n)=0\leq \epsilon$ for all $\epsilon>0$. Furthermore,
	\[
	\Gamma_0 = \{x\in\overline{\Omega}~\vert ~ \prod_{n=1}^N f(x_n) = 0\} = \{x\in\overline{\Omega}~\vert ~ \exists 1\leq m\leq N: x_m = 0\} = \partial[0,\infty[^N\cap\partial\Omega.
	\]
	
	\noindent\textit{(Continuity):} If $\prod_{n=1}^N f(x_n)\leq \epsilon$ for all $\epsilon>0$, it holds true that $x\in \partial[0,\infty[^N\cap\partial \Omega$.\\%
	
	\noindent\textit{(Finiteness):} We introduce the set $\{-1,1\}^N = \{z\in\mathbb{N}^N~\vert ~ z_n \in \{-1,1\}\}$ and the Hadamard-Product between vectors (entrywise product) by $\odot$. The idea is to decompose $[0,\infty[^N$ into disjoint subsets where we consider for each entry whether it is larger or smaller than one. With this and $\Gamma_1 \subseteq [0,\infty[^N$ we obtain
	\[
	\text{vol}_N(\Gamma_1) = \int_{\Gamma_1}\dx x = \sum_{B\in \{-1,1\}^N}\int_{\Gamma_1 \cap[B\odot x > B]}\dx x \leq \sum_{B\in \{-1,1\}^N}\int_{[B\odot x > B]\cap[\prod_{n=1}^N f(x_n)\leq 1]}\dx x.
	\]
	By symmetry of the summands it suffices to verify for fixed $B\in\{-1,1\}^N$ with $p:=\#\{n~\vert ~ B_n = 1\}$, $q:=N-p$, w.l.o.g.\ $x_1, \dots, x_p\geq 1$, $y_n:= x_{p+n}$ for $1\leq n\leq q$
	\[
	I :=\int_{(\prod_{i=1}^p x_i)^2(\prod_{j=1}^q y_j)^{1/2}\leq 1, x_i\geq 1, 0\leq y_j<1}\dx x\dx y<\infty.
	\]
	W.l.o.g. we assume $p,q\geq 1$ (if $p=0$ then $I= 1$ and for $q=0$ we have $I=0$). For that, we show first that
	\begin{equation}
		\label{Plogineq}
		I \lesssim \int_{\prod_{n=1}^q y_n\leq 1, 0\leq y_n<1} (\prod_{n=1}^q y_n)^{-1/4} (-\log\prod_{n=1}^q y_n)^{p-1}\dx y,
	\end{equation}
	where $"\lesssim"$ has to be understood as $\leq$ up to some multiplicative positive constant. %
	As preparation for the proof of \eqref{Plogineq}, we introduce the following compact notation for the sake of readability. We define for $0\leq k\leq p$ the coefficients \( \gamma_k^{x,y}:=\prod_{i=1}^k x_i (\prod_{j=1}^q y_j)^{1/4}. \) Note that $\gamma_0^{x,y} = (\Pi_{j=1}^q y_j)^{1/4}$. Denote
	\[
	[Y,X^k]:=\{(y,x)\in \mathbb{R}^q\times\mathbb{R}^k ~\vert ~ x_i\geq 1, 0\leq y_j<1, 1\leq i\leq k, 1\leq j\leq q, \gamma_k^{x,y}\leq 1\}
	\]
	for $0\leq k\leq p$ and $[Y]:=[Y,X^0]=\{y\in\mathbb{R}^q ~\vert ~ 0\leq y_j<1, 1\leq j\leq q, \gamma_0^{x,y}\leq 1\}$. Due to the propagation property $\gamma_{k+1}^{x,y} = x_{k+1}\gamma_k^{x,y}$ and $x_i\geq 1$ for $1\leq i\leq p$ the inclusions $[Y, X^{k+1}]\subseteq [Y, X^k]\times \mathbb{R}_{\geq 1}$ for $0\leq k\leq p$ are valid. We further point out that by $\dx y$ we denote the integration with respect to the variables $y_1, \dots, y_q$ and by $\dx x^k$ the integration with respect to $x_1, \dots, x_k$ for $1\leq k\leq p$. As a first step for proving \eqref{Plogineq} we show that under the restriction $\gamma_{k+1}^{x,y}\leq 1$ for $l\in\mathbb{N}$ we have
	\[
	\int_1^\infty x_{k+1}^{-1}(-\log(\gamma_{k+1}^{x,y}))^l \dx x_{k+1}\lesssim (-\log(\gamma_{k}^{x,y}))^{l+1}.
	\]
	As $\gamma_{k+1}^{x,y} = x_{k+1}\gamma_{k}^{x,y}$ it holds by Hölders inequality for $x_{k+1}\geq 1$ and $\gamma_{k+1}^{x,y}\leq 1$ that $$(-\log(\gamma_{k+1}^{x,y}))^l\leq 2^{l-1}(-\log(\gamma_{k}^{x,y}))^l+2^{l-1}\log(x_{k+1})^l$$ since also $\gamma_{k}^{x,y}\leq 1$. With this it follows that
	\begin{align}
		\label{lesssimestim}
		\notag\int_1^{1/\gamma_k^{x,y}} x_{k+1}^{-1}(-\log(\gamma_{k+1}^{x,y}))^l \dx x_{k+1}&\leq 2^{l-1}(-\log(\gamma_k^{x,y}))^l \int_1^{1/\gamma_k^{x,y}} x_{k+1}^{-1}\dx x_{k+1}\\
		\notag&\hspace{1cm}+2^{l-1}\int_1^{1/\gamma_k^{x,y}} x_{k+1}^{-1}\log(x_{k+1})^l\dx x_{k+1}\\
		&\leq 2^{l-1}\frac{l+2}{l+1}(-\log(\gamma_k^{x,y}))^{l+1}\lesssim(-\log(\gamma_k^{x,y}))^{l+1}.
	\end{align}
	As a consequence, we derive that for $1\leq m\leq p-1$ it holds
	\begin{multline*}
		\int_{[Y,X^{p-m}]}\frac{\log(\gamma_{p-m}^{x,y})^{m-1}}{\gamma_{p-m}^{x,y}}\dx x^{p-m}\dx y\\ = \int_{[Y,X^{p-m-1}]}\frac{1}{\gamma_{p-m-1}^{x,y}}\bigg(\int_1^{1/\gamma_{p-m-1}^{x,y}}\frac{\log(\gamma_{p-m}^{x,y})^{m-1}}{x_{p-m}}\dx x_{p-m}\bigg)\dx x^{p-m-1}\dx y
	\end{multline*}
	and thus, by the previous considerations
	\begin{equation}
		\label{powerestim2}
		\int_{[Y,X^{p-m}]}\frac{(-\log(\gamma_{p-m}^{x,y}))^{m-1}}{\gamma_{p-m}^{x,y}}\dx x^{p-m}\dx y \lesssim \int_{[Y,X^{p-m-1}]}\frac{(-\log(\gamma_{p-m-1}^{x,y}))^{m}}{\gamma_{p-m-1}^{x,y}}\dx x^{p-m-1}\dx y.
	\end{equation}
	Note that the multiplicative constant in \eqref{powerestim2} only depends on the exponent $m$ and is independent of $x,y$. Now as under $x_p \leq 1/\gamma_{p-1}^{x,y}$
	\[
	I = \int_{[Y, X^p]} \dx x \dx y \leq \int_{[Y, X^{p-1}]} \frac{1}{\gamma_{p-1}^{x,y}}\dx x^{p-1}\dx y,
	\]
	we derive by successively applying \eqref{powerestim2} the claimed estimation in \eqref{Plogineq}, i.e., that
	\begin{equation}
		\label{Iestim3}
		I \lesssim \int_{[Y]}\frac{(-\log(\gamma_{0}^{x,y}))^{p-1}}{\gamma_{0}^{x,y}}\dx y.
	\end{equation}
	Next we show that \eqref{Iestim3} is finite. By Hölders inequality and $0\leq y_j<1$ we derive that
	\[
	(-\log(\gamma_{0}^{x,y}))^{p-1}=(-\frac{1}{4}\sum_{j=1}^q\log y_j)^{p-1}\leq \frac{1}{4}\bigg(\frac{q}{4}\bigg)^{p-2} \sum_{j=1}^q (-\log y_j)^{p-1}.
	\]
	As a consequence, by symmetry it suffices to show boundedness of the term
	\[
	\int_{[Y]}\frac{(-\log(y_1))^{p-1}}{\gamma_{0}^{x,y}}\dx y.
	\]
	Indeed, as $[Y]\subseteq [0,1]^q$, the integrand of the previous integral does not change sign and %
	\[
	\int_{[Y]}\frac{(-\log(y_1))^{p-1}}{\gamma_{0}^{x,y}}\dx y \lesssim\bigg(\int_0^1 \frac{1}{z^{1/4}}\dx z\bigg)^{q-1}\int_0^1 \frac{(-\log z)^{p-1}}{z^{1/4}}\dx z <\infty.
	\]
	Finally, we conclude that the class $(\Gamma_\epsilon)_{0\leq\epsilon<1}$ is finite.
	
	\begin{remark}
		Note that above $\Gamma_\epsilon$ implicitly entail a distance notion, as for $x\in\Gamma_\epsilon$ we may consider for $\epsilon<1$ and $k = \argmin_{1\leq n\leq N} x_n$ the element $$\hat{x}=(x_1, \dots, x_{k-1},0,x_{k+1},\dots,x_N)\in\partial[0,\infty[^N,$$ fulfilling $\Vert x-\hat{x}\Vert_\infty \leq \epsilon^{2/N}$. In fact, using that $x_k\leq 1$ by the choice of $k$ and $\epsilon$ it holds
		\[
		\epsilon \geq \Pi f(x_j) \geq \sqrt{x_k}\Pi_{j\neq k}f(x_j)\geq \sqrt{x_k}\Pi_{j\neq k} \min(\sqrt{x_j},1)\geq \sqrt{x_k}\sqrt{x_k}^{N-1}=x_k^{N/2}.
		\]
	\end{remark}
	We showcase how the $L^p$-convergence result can be extended to general domains for $\Omega = [0,\infty[^N$. Let $(\chi_{\Gamma_{\epsilon_n}}^{\delta_n})_n$ for $(\epsilon_n)_n, (\delta_n)_n$ with $0<\delta_n<\epsilon_n$ be such that 
	\[
	\chi_{\Gamma_{\epsilon_n}}^{\delta_n}\big|_{\Gamma_{\epsilon_n-\delta_n}}\equiv 1, ~ ~ ~ \chi_{\Gamma_{\epsilon_n}}^{\delta_n}\big|_{\Omega_{\epsilon_n+\delta_n}}\equiv 0, ~ ~ \text{and} ~ ~ 0\leq \chi_{\Gamma_{\epsilon_n}}^{\delta_n}\leq 1.
	\]
	Each $\chi_{\Gamma_{\epsilon_n}}^{\delta_n}$ is w.l.o.g.\ continuous as for $x\in\Gamma_{\epsilon_n+\delta_n}\cap\Omega_{\epsilon_n-\delta_n}$ one can set $x_+=\text{sup}\{\lambda~|~\lambda x\in \Gamma_{\epsilon_n-\delta_n}\}x$ and $x_-=\text{inf}\{\lambda~|~\lambda x\in \Omega_{\epsilon_n+\delta_n}\}x$, write $x = \mu x_++(1-\mu)x_-$ for some $0\leq\mu\leq 1$ and set $\chi_{\Gamma_{\epsilon_n}}^{\delta_n}(x)=\mu$. By local convolution with kernel-radii depending on the position (to not change boundary values of one), this family may be even chosen to be smooth. One can show that the resulting $f_{n,\epsilon_n,\delta_n}$ are continuous, strictly quasipositive and approximate $f$ by similar techniques.
	
	\subsection{Generalization to quasipositive functions}
	\label{subsec:generalization_quasipositive}
	The previously discussed approximation results may be directly extended to quasipositive functions $F\in\mathcal{C}(\overline{\Omega},\mathbb{R}^N)$ as follows. Consider the results in Theorem~\ref{mult_inf_est}, Assumption~\ref{assump}, Theorem~\ref{mult_p_est} and Lemma~\ref{assproof} with the following modifications: Use
	\begin{itemize}
		\item $H_n\cap\partial\Omega$ instead of $\partial[0,\infty[^N\cap\partial\Omega$ with $H_n := \{x\in\partial[0,\infty[^N~|~x_n=0\}$,
		\item $\Gamma^n_\epsilon:=\{x\in\overline{\Omega}~|~\mathcal{D}(x, H_n\cap\partial\Omega)<\epsilon\}$ instead of $\Gamma_\epsilon$ with corresponding $\Omega^n_\epsilon$.
	\end{itemize}
	As a consequence, we derive that continuous functions $g\in\mathcal{C}(\Omega, \mathbb{R})$, that for some $1\leq n\leq N$ fulfill $g(\hat{x}_n)\geq 0$ for all $\hat{x}_n=(x_1, \dots, x_{n-1}, 0, x_{n+1}, \dots, x_N)\in\overline{\Omega}$, may be approximated by the same type of functions (with the same $n$) based on the notions of convergence discussed previously. Taking now a quasipositive function $F\in\mathcal{C}(\overline{\Omega},\mathbb{R}^N)$ we can approximate its components $F_n$ and hence, also $F$. The extension to general domains as discussed in Subsection~\ref{subsec:ext_gen_dom} may be achieved by
	\[
	\Gamma_\epsilon^n := \{x\in\overline{\Omega}~|~\Pi_{j=1}^Nf(x_j)\leq \epsilon, x_n\leq \epsilon\}.
	\]
	\section{Auxiliary results from the literature}
	This appendix collects results from the literature that are used throughout the paper but whose statements are not original: i) existence theory for RD systems and ii) operator theory on Sobolev spaces. We include them to standardize notation and to keep this work self-contained. No new results are proven or presented here.
	\subsection{Existence theory for RD systems}
	\label{app:ex_rdsys}
	In the following section, we discuss existence results for RD systems and refer to \cite{Pierre2010,Fellner2020,Suzuki2017,Laamri2020} for state-of-the-art developments. We present the classical existence result from \cite{Fellner2020} as well as the weak existence result from \cite{Suzuki2017}.\\
	
	\noindent\textbf{Classical existence result:}~
	\begin{theorem}{\cite[Theorem 1.1]{Fellner2020}}
		\label{th:fellner}
		Let $\Omega\subseteq \mathbb{R}^d$ be a bounded domain with smooth boundary such that $\Omega$ lies locally on one side of $\partial\Omega$. Consider the RD system
		\begin{align}
			\label{fellner_sys}
			\begin{cases}
				\frac{\partial}{\partial t}u_n-d_n\Delta u_n = f_n(u), &(t,x)\in ]0,T[\times\Omega,\\
				\nabla_x u_n\cdot \nu = 0, & (t,x)\in ]0,T[\times\partial\Omega,\\
				u_n(0,x)=u_{n,0}(x), & x\in \Omega,
			\end{cases}
		\end{align}
		for $n=1,\ldots, N$. Assume that the initial data $(u_{n,0})_{n=1}^N\subseteq L^1(\Omega)\cap L^\infty(\Omega)$ is bounded and non-negative. Furthermore, suppose that the mass is controlled by
		\[
		\sum_{n=1}^{N}f_n(u) \leq K_0+K_1\sum_{n=1}^N u_n
		\]
		for some $K_0\geq 0$ and $K_1\in \mathbb{R}$ for all $u\in [0,\infty[^N$ and the reaction term $f$ is locally Lipschitz continuous and quasipositive, i.e.,
		\[
		f_n(u_1, \dots, u_{n-1}, 0, u_{n+1}, \dots, u_N)\geq 0
		\]
		for $1\leq n\leq N$ and $u\in [0,\infty[^N$. Then there exists some $\epsilon>0$ such that for $$\vert f_n(u)\vert \leq K(1+\vert u\vert^{2+\epsilon})$$ for $1\leq n\leq N$ and $u\in \mathbb{R}^N$, system \eqref{fellner_sys} admits a unique global classical solution $$(u_n)_{n=1}^N\subseteq \mathcal{C}(0,T;L^p(\Omega)\cap L^\infty(\Omega))\cap\mathcal{C}^{1,2}(]0,T[\times\overline{\Omega})$$ for all $p>N$ satisfying \eqref{fellner_sys} for $T>0$.
	\end{theorem}
	\begin{remark}
		Theorem~\ref{th:fellner} holds also in case the Neumann boundary condition in system \eqref{fellner_sys} is replaced by a homogeneous Dirichlet boundary condition.
	\end{remark}
	
	\noindent\textbf{Weak existence result:}
	\begin{theorem}[Special case of {\cite[Theorem 1]{Suzuki2017}}]
		\label{th:suzuki}
		Let $\Omega\subseteq \mathbb{R}^d$ be a bounded domain with smooth boundary. Consider the reaction-diffusion system for $n=1,\ldots, N$:
		\begin{align}
			\label{suzuki_sys}
			\begin{cases}
				\frac{\partial}{\partial t}u_n-d_n\Delta u_n = f_n(u), &(t,x)\in ]0,T[\times\Omega,\\
				u_n(t,x)=g_n(t,x)\geq 0, & (t,x)\in ]0,T[\times\partial\Omega,\\
				u_n(0,x)=u_{n,0}(x)\geq0, & x\in \Omega.
			\end{cases}
		\end{align}
		Assume that $(u_{n,0})_{n=1}^N\subseteq L^\infty(\Omega)$ is non-negative and that the boundary data $(g_n)_{n=1}^N$ admit non-negative extensions $(G_n)_{n=1}^N\subseteq \mathcal{C}^1([0,T];\mathcal{C}(\overline{\Omega}))\cap \mathcal{C}([0,T];\mathcal{C}^2(\overline{\Omega}))$ satisfying $\partial_t G_n-d_n\Delta G_n=0$ in $]0,T[\times\Omega$ and $G_n=g_n$ on $]0,T[\times\partial\Omega$ for $1\leq n\leq N$. Furthermore, suppose mass dissipation, i.e., there exists $(c_n)_{n=1}^N\subseteq ~]0,\infty[$ with
		\[
		\sum_{n=1}^{N}c_nf_n(u) \leq 0\quad \text{for all} ~ u\in [0,\infty[^N
		\]
		and that the reaction term $f$ is locally Lipschitz continuous and quasipositive. If 		$$\vert f(u)\vert \leq K(1+\vert u\vert^2)$$ for $u\in \mathbb{R}^N$, system \eqref{suzuki_sys} admits a global weak solution, i.e., we have $$(u_n)_{n=1}^N\subseteq \mathcal{C}(0,T;L^1(\Omega))\cap L^2(]0,T[\times\overline{\Omega})$$ for all $T>0$ and  \eqref{suzuki_sys} is fulfilled in the weak sense.
	\end{theorem}
	In the assumptions of both results, it is required that the reaction term $f$ is locally Lipschitz continuous, satisfies an appropriate growth condition together with a mass estimate, and is quasipositive, as discussed in detail in Section~\ref{sec:conformal_classes}.
	
	\subsection{Operators in Sobolev spaces}
	\label{app:op_sob}
	In this section, we review existing results on superposition and multiplication operators in Sobolev spaces. In particular, the following result, established in \cite{mizel79}, concerns superposition mappings acting on $N$-tuple first-order Sobolev spaces.
	\begin{theorem}[{\cite[Theorem 1]{mizel79}}]
		\label{th:mizel}
		Let $\Omega\subseteq \mathbb{R}^d$ be a bounded domain, $g:\mathbb{R}^N\to \mathbb{R}$ a Borel function and $p,r\geq 1$ real numbers. For $\mathcal{M}(\Omega)$ the space of real measurable functions in $\Omega$ denote by $T_g:\mathcal{M}(\Omega)^N\to \mathcal{M}(\Omega)$ the superposition mapping
		\[
		T_g u = g\circ u ~ ~ \text{ for } ~ u = (u_1, \dots, u_N)\in \mathcal{M}(\Omega)^N.
		\]
		In case $1\leq r\leq p< d$, the superposition operator $T_g$ maps $W^{1,p}(\Omega)^N$ into $W^{1,r}(\Omega)$ if and only if $g$ is locally Lipschitz continuous in $\mathbb{R}^N$ and the partial derivatives fulfill the growth condition $\vert \partial_{x_n} g(x)\vert \leq c(1+\vert x\vert^{d(p-r)/(r(d-p))})$ a.e. in $\mathbb{R}^N$ for some $c>0$ and $1\leq n\leq N$. It further holds true that
		\[
		\Vert T_g u\Vert_{W^{1,r}(\Omega)}\leq c\left(1+\sum_{n=1}^N\Vert u_n\Vert_{W^{1,p}(\Omega)}^{d(p-r)/(r(d-p))+1}\right)
		\]
		for some $c>0$. In case $d<p$ (or $d=1$ and $1\leq p$) the previous statements holds without imposing the growth condition. Furthermore, for $\sum_{n=1}^N\Vert u_n\Vert_{W^{1,p}(\Omega)}\leq M$ there exists some $c(M)>0$ such that
		\[
		\Vert T_g u\Vert_{W^{1,r}(\Omega)}\leq c(M)\left(1+\sum_{n=1}^N\Vert u_n\Vert_{W^{1,p}(\Omega)}\right).
		\]
	\end{theorem}
	Additional material on superposition operators in Sobolev spaces can be found in \cite[Chapter 9]{AppZab}. For results on autonomous Nemytskii operators acting between general Sobolev spaces, together with a higher-order chain rule, we refer to \cite{Florin2022}.
	
	Next we consider the multiplication of Sobolev regular functions based on \cite{behzadan2021}.
	\begin{theorem}[{\cite[Theorem 6.1, Corollary 6.3, Theorem 7.4]{behzadan2021}}]
		\label{th:behzadan}
		Let $\Omega$ be a bounded Lipschitz domain in $\mathbb{R}^d$ and $s_1, s_2, s$, $1\leq p_1, p_2, p<\infty$ real numbers satisfying:
		\begin{enumerate}[label=\roman*)]
			\item $s_i\geq s \geq 0$
			\item $s_i-s\geq d(\frac{1}{p_i}-\frac{1}{p})$
			\item $s_1+s_2-s>d(\frac{1}{p_1}+\frac{1}{p_2}-\frac{1}{p})$
		\end{enumerate}
		If $s\in \mathbb{N}_0$ assume further that $\frac{1}{p_1}+\frac{1}{p_2}\geq \frac{1}{p}$. The strictness of the inequalities ii) and iii) is interchangeable if $s\in\mathbb{N}_0$. In case of $s\in]0,\infty[\backslash\mathbb{N}$ where $p<\max(p_1,p_2)$ suppose that $s_1+s_2-s>d/\min(p_1,p_2)$ instead of iii) together with strictness in i) and ii). Then the multiplication operator considered on $W^{s_1,p_1}(\Omega)\times W^{s_2, p_2}(\Omega)$ defines a well-defined continuous bilinear map $W^{s_1,p_1}(\Omega)\times W^{s_2, p_2}(\Omega)\to W^{s,p}(\Omega)$.
	\end{theorem}
	\bibliographystyle{plain}
	\bibliography{references}
	
\end{document}